\documentclass[twoside,11pt]{article}
\usepackage{jmlr2e_arxiv}
\usepackage{amssymb}
\usepackage{amsmath}
\usepackage{graphicx}
\usepackage{array}
\usepackage{hyperref}
\usepackage{enumitem}
\usepackage{accents}
\usepackage{tabularx}
\usepackage{floatflt}
\usepackage{tikz}
\usepackage{mathtools}
\usepackage[ruled]{algorithm}
\usepackage{algpseudocode}
\usepackage{natbib}
\usepackage{stmaryrd}
\usepackage{bbm}
\usepackage{subfigure}
\usepackage{cases}
\usepackage{xspace}

\mathtoolsset{centercolon}

\jmlrheading{1}{2000}{1-48}{4/00}{10/00}{J. Bouvrie and M. Maggioni}
\ShortHeadings{Multiscale Markov Decision Problems}{J. Bouvrie and M. Maggioni}
\firstpageno{1}

\usetikzlibrary{shapes,arrows,automata,chains,matrix,positioning,fit,scopes,decorations.pathmorphing,shadows,svg.path}

\newcommand{\bindex}{\ensuremath{{\mathsf{b}}}\xspace}		
\newcommand{\clusters}{\ensuremath{{\mathsf{C}}}\xspace}	
\newcommand{\clustertext}{cluster\xspace}		
\newcommand{\clusterstext}{clusters\xspace}		
\newcommand{\clusterindex}{\ensuremath{\mathsf{c}}\xspace}	
\newcommand{\boundary}[1]{\ensuremath{{{{\partial{#1}}}}}\xspace} 
\newcommand{\picluster}{\ensuremath{{\boldsymbol{\pi}}}\xspace}	
\newcommand{\Ppirestrtocluster}{P_\clusterindex^{\pi_\clusterindex}}
\newcommand{\Rpirestrtocluster}{R_\clusterindex^{\pi_\clusterindex}}
\newcommand{\Gpirestrtocluster}{\Gamma_\clusterindex^{\pi_\clusterindex}}

\newcommand{\problem}[1]{\ensuremath{\mathrm{MMDP}_{(#1)}}\xspace}
\newcommand{\problemone}{\ensuremath{\problem 1}\xspace}
\newcommand{\problemtwo}{\ensuremath{\problem 2}\xspace}
\newcommand{\clusterone}{\ensuremath{\clusterindex_1}\xspace}
\newcommand{\clustertwo}{\ensuremath{\clusterindex_2}\xspace}

\newcommand{\bbone}{\mathbf{1}}
\newcommand{\tr}{\top\!}
\newcommand{\bbE}{{\mathbb E}}
\newcommand{\bbR}{{\mathbb R}}
\newcommand{\bbP}{{\mathbb P}}

\newcommand{\Bpi}{{\mathcal B}^{\pi}}
\newcommand{\B}{{\mathcal B}}
\newcommand{\cO}{{\mathcal O}}
\newcommand{\cL}{{\mathcal L}}
\newcommand{\cW}{{\mathcal W}}
\newcommand{\cG}{{\mathcal G}}
\newcommand{\cP}{{\mathcal P}}
\newcommand{\cZ}{{\mathcal Z}}

\newcommand{\R}{{\mathbb R}}

\newcommand{\sgn}{\operatorname{sgn}}
\newcommand{\dom}{\operatorname{dom}}
\newcommand{\supp}{\operatorname{supp}}
\newcommand{\indic}{\mathbbm{1}}

\newcommand\restr[2]{{
  \left.\kern-\nulldelimiterspace 
  #1 
  \vphantom{\big|} 
  \right|_{#2} 
  }}

\newcommand{\Ui}{\accentset{\circ}{\clusterindex}}
\newcommand{\Uprime}{\ensuremath{\clusterindex'_{s'}}\xspace}
\newcommand{\vi}{\accentset{\circ}{\clustertwo}}
\newcommand{\Si}{\accentset{\circ}{S}}

\newcommand{\dcluster}{\ensuremath{\partial\clusterindex}\xspace}
\newcommand{\icluster}{\ensuremath{\accentset{\circ}{\clusterindex}\xspace}}

\begin{document}

\title{Multiscale Markov Decision Problems:\\ Compression, Solution, and Transfer Learning}

\author{\name{Jake Bouvrie}\thanks{Corresponding author. Current address: Laboratory for Computational and Statistical Learning, Massachusetts Institute of Technology, Bldg. 46-5155, 77 Massachusetts Avenue, Cambridge, MA 02139. Email: \texttt{jvb@csail.mit.edu}}\email jvb@math.duke.edu \\
\addr Department of Mathematics \\
Duke University \\
Durham, NC 27708
\AND
\name Mauro Maggioni \email mauro@math.duke.edu \\
\addr Departments of Mathematics, Computer Science, and Electrical and Computer Engineering\\
Duke University \\
Durham, NC 27708
}
\editor{TBD}

\maketitle



\begin{abstract}
Many problems in sequential decision making and stochastic control often have natural multiscale structure: sub-tasks are assembled together to accomplish complex goals. Systematically inferring and leveraging hierarchical structure, particularly beyond a single level of abstraction, has remained a longstanding challenge. We describe a fast multiscale procedure for repeatedly compressing, or homogenizing, Markov decision processes (MDPs), wherein a hierarchy of sub-problems at different scales is automatically determined. Coarsened MDPs are themselves independent, deterministic MDPs, and may be solved using existing algorithms. The multiscale representation delivered by this procedure decouples sub-tasks from each other and can lead to substantial improvements in convergence rates both locally within sub-problems and globally across sub-problems, yielding significant computational savings. A second fundamental aspect of this work is that these multiscale decompositions yield new transfer opportunities across different problems, where solutions of sub-tasks at different levels of the hierarchy may be amenable to transfer to new problems. Localized transfer of policies and potential operators at  arbitrary scales is emphasized. Finally, we demonstrate compression and transfer in a collection of illustrative domains, including examples involving discrete and continuous statespaces.
%
\end{abstract}


\section{Introduction}
Identifying and leveraging hierarchical structure has been a key, longstanding challenge for sequential decision making and planning research~\citep{SuttonOptions,DietterichHRL,ParrRussell}. Hierarchical structure generally suggests a decomposition of a complex problem into smaller, simpler sub-tasks, which may be, ideally, considered independently~\citep{BartoHRLReview}. One or more layers of abstraction may also provide a broad mechanism for reusing or transferring commonly occurring sub-tasks among related problems~\citep{BarryKael:IJCAI:11,Taylor09,SoniSingh:06,FergusonMahadevan:06}. These themes are  restatements of the divide-and-conquer principle: it is usually dramatically cheaper to solve a collection of small problems than a single big problem, when the solution of each problem involves a number of computations super-linear in the size of the problem. Two ingredients are often sought for efficient divide-and-conquer approaches: a hierarchical subdivision of a large problem into disjoint subproblems, and a procedure merging the solution of subproblems into the solution of a larger problem.

This paper considers the discovery and use of hierarchical structure -- {\em multiscale} structure in particular -- in the context of discrete-time Markov decision problems. Fundamentally, inferring  multiscale decompositions, learning abstract actions, and planning across scales are intimately related concepts, and we couple these elements tightly within a unifying framework. Two main contributions are presented:
\begin{itemize}
\item The first is an efficient  multiscale procedure for partitioning and then repeatedly {\em compressing} or {\em homogenizing} Markov decision processes (MDPs).
\item The second contribution consists of a means for identifying {\em{transfer}} opportunities, representing transferrable information, and incorporating this information into new problems, within the context of the multiscale decomposition.
\end{itemize}

Several possible approaches to multiscale partitioning are considered, in which statespace geometry, intrinsic dimension, and the reward structure play prominent roles, although a wide range of existing algorithms may be chosen. Regardless of how the partitioning is accomplished, the statespace is divided into a multiscale collection of ``\clusterstext'' connected via a small set of ``bottleneck'' states. A key function of the partitioning step is to tie {\em computational complexity} to {\em problem complexity}. It is problem complexity, controlled for instance by the inherent geometry of a problem, and its amenability to being partitioned, that should determine the computational complexity, rather than the choice of statespace representation or sampling. 
The decomposition we suggest attempts to rectify this difficulty.

The multiscale compression procedure is local and efficient because only one \clustertext needs to be  considered at a time. The result of the compression step is a representation decomposing a problem into a hierarchy of distinct sub-problems at multiple scales, each of which may be solved efficiently and independently of the others.  The homogenization we propose is perfectly recursive in that a compressed MDP is again another independent, deterministic MDP, and the statespace of the compressed MDP is a (small) subset of the original problem's statespace.  Moreover, each coarse MDP in a multiscale hierarchy is ``consistent in the mean'' with the underlying fine scale problem. The compressed representation coarsely summarizes a problem's statespace, reward structure and Markov transition dynamics, and may be computed either analytically or by Monte-Carlo simulations.  Actions at coarser scales are typically complex, ``macro'' actions, and the coarsening procedure may be thought of as producing different levels of abstraction of the original problem. In an appropriate sense, optimal value functions at homogenized scales are homogenized optimal value functions at the finer scale.

Given such a hierarchy of successively coarsened representations, an MDP may be solved efficiently. We describe a family of multiscale solution algorithms which realize computational savings in two ways: (1) {\em Localization}: computation  is restricted to small, decoupled sub-problems; and (2) {\em conditioning}: sub-problems are comparatively well-conditioned due to improved local mixing times at finer scales and fast mixing globally at coarse scales, and obey a form of global consistency with each other through coarser scales, which are themselves well-conditioned coarse MDPs. The key idea behind these algorithms is that sub-problems at a  given scale decouple conditional on a solution at the next coarser scale, but must contribute constructively towards solving the overarching problem through the coarse solution; interleaved updates to solutions at pairs of fine and coarse scales are repeatedly applied until convergence. We present one particular algorithm in detail: a localized variant of modified asynchronous policy iteration that can achieve, under suitable geometric assumptions on the problem, a cost of $\cO(n\log n)$ per iteration, if there are $n$ states. The algorithm is also shown to converge to a globally optimal solution from any initial condition.

Solutions to sub-problems may be transferred among related tasks, giving a systematic means to approach transfer learning at multiple scales in planning and reinforcement learning domains. If a learning problem can be decomposed into a hierarchy of distinct parts then there is hope that both a ``meta policy'' governing transitions between the parts, as well as some of the parts themselves, may be transferred when appropriate. We propose a novel form of multiscale knowledge transfer between sufficiently related MDPs that is made possible by the multiscale framework: Transfer between two hierarchies proceeds by matching sub-problems at various scales, transferring policies, value functions and/or potential operators where appropriate (and where it has been determined that transfer can help), and finally solving for the remainder of the destination problem using the transferred information.  In this sense knowledge of a partial or coarse solution to one problem can be used to quickly learn another, both in terms of computation and, where applicable, exploratory experience.

The paper is organized as follows. In Section~\ref{sec:bkgnd-prelim} we collect preliminary definitions, and provide a brief overview of Markov Decision Processes with stochastic policies and state/action dependent rewards and discount factors. Section~\ref{sec:mmdp-toplevel} describes partitioning, compression and multiscale solution of MDPs. Proofs and additional comments concerning computational considerations related to this section are collected in the Appendix. In Section~\ref{sec:transfer} we introduce the multiscale transfer learning framework, and in Section~\ref{sec:examples} we provide examples demonstrating compression and transfer in the context of three different domains (discrete and continuous). We discuss and compare related work in Section~\ref{sec:prior_work}, and conclude with  some additional observations, comments and open problems in Section~\ref{sec:discussion}.

\section{Background and Preliminaries}
\label{sec:bkgnd-prelim}

The following subsections provide a brief overview of Markov decision processes as well as some definitions and notation used in the paper. 
 
\subsection{Markov Decision Processes}
\label{sec:mdp-defns}
Formally, a Markov decision process (MDP) (see e.g.~\citep{Puterman}, \citep{BertsekasVol2}) is a sequential decision problem defined by a tuple $(S,A,P,R,\Gamma)$ consisting of a statespace $S$, an action (or ``control'') set $A$, and for $s,s'\in S, a\in A$, a transition probability tensor $P(s,a,s')$, reward function $R(s,a,s')$ and collection of discount factors $\Gamma(s,a,s')\in(0,1)$. We will assume that $S, A$ are finite sets, and that $R$ is bounded. The definition above is slightly more general than usual in that we allow state and action dependent rewards and discount factors; the reason for adopting this convention will be made clear in Section~\ref{sec:compression}.
The probability $P(s,a,s')$ refers to the probability that we transition to $s'$ upon taking action $a$ in $s$, while $R(s,a,s')$ is the reward collected in the event we transition from $s$ to $s'$ {\em{after}} taking action $a$ in $s$. 

\subsubsection{Stochastic Policies}
Let $\cP(A)$ denote the set of all discrete probability distributions on $A$. 
A {\em{stationary stochastic policy}} (simply a {\em{policy}}, from now on) \mbox{$\pi:S\to \cP(A)$} is a function mapping states into distributions over the actions.  Working with this more general class of policies will allow for convex combinations of policies later on. 
A policy $\pi$ may be thought of as a non-negative function on $S\times A$ satisfying $\sum_{a\in A}\pi(s,a)=1$ for each $s\in S$, where $\pi(s,a)$ denotes the probability that we take action $a$ in state $s$. We will often write $\pi(s)$ when referring to the {\em distribution} on actions associated to the (deterministic) state $s\in S$, so that $a\sim\pi(s)$ denotes the $A$-valued random variable $a$ having law $\pi(s)$. Deterministic policies can be recovered by placing unit masses on the desired actions\footnote{We will allow the set of actions available in state $s$ to be limited to a nonempty state-dependent subset $\mathcal{A}(s)\subseteq A$ of {\em feasible} actions, but do not explicitly keep track of the sets $\mathcal{A}(s)$ to avoid cluttering the notation. As a matter of bookkeeping, we assume that these constraints are enforced as needed by setting $P(s,a,s')=0$ for all $s'$ if $a\notin\mathcal{A}(s)$, and/or by assigning zero probability to invalid actions in the case of stochastic policies (discussed below). If a stochastic policy has been restricted to the feasible actions, then it will be assumed that it has also been suitably re-normalized.
}. 

We may compute policy-specific Markov transition matrices and reward functions by averaging out the actions according to $\pi$:
\begin{subequations}\label{eqn:PpiRpi}
\begin{align}
P^{\pi}(s,s') &= \bbE_{a\sim\pi(s)}[P(s,a,s')]= \sum_{a\in A}P(s, a, s')\pi(s,a) \\
R^{\pi}(s,s') &= \bbE_{a\sim\pi(s)}[R(s,a,s')]= \sum_{a\in A}R(s, a, s')\pi(s,a) \;.
\end{align}
\end{subequations}
For any pair of tensors $X=X(s,a,s'), Y=Y(s,a,s')$ indexed by $s,s'\in S, a\in A$, we define the matrix $(X\circ Y)^{\pi}$ to be the expectation with respect to $\pi$ of the elementwise (Hadamard) product between $X$ and $Y$:
\begin{equation}\label{eqn:hadamard-pi}
\begin{aligned}
[(X\circ Y)^{\pi}]_{s,s'} :=&\; \bbE_{a\sim\pi(s)}[X(s,a,s')Y(s,a,s')]
= \sum_{a\in A} X(s,a,s')Y(s,a,s')\pi(s,a).
\end{aligned}
\end{equation}
Note that $(X\circ Y)^{\pi}=(Y\circ X)^{\pi}$.

Finally, we will often make use of the uniform random or {\em diffusion} policy, denoted $\pi^u$, which always takes an action drawn randomly according to the uniform distribution on the feasible   actions. In the case of continuous action spaces, we assume a natural choice of ``uniform'' measure has been made: for example the Haar measure if $A$ is a group, or the volume measure if $A$ is a Riemannian manifold.

\subsubsection{Value Functions and the Potential Operator}
\label{sec:value_fns}
Given a policy, we may define a value function $V^{\pi}:S\to\R$ assigning to each state $s$ the expected sum of discounted rewards collected over an infinite horizon by running the policy $\pi$ starting in $s$:
\begin{equation}\label{eqn:pi-discounted-rewards}
V^{\pi}(s) = \bbE\left[R(s_0,a_1,s_1) + 
\sum_{t=1}^{\infty}\left\lbrace\prod_{\tau=0}^{t-1}\Gamma(s_{\tau},a_{\tau+1},s_{\tau+1})\right\rbrace R(s_t,a_{t+1},s_{t+1}) ~\Bigl|\Bigr.~ s_0=s\right],
\end{equation}
where the sequence of random variables $(s_i)_{i=1}^{\infty}$ is a Markov chain with transition probability matrix $P^{\pi}$. The expectation is taken over all sequences of state-action pairs $\{(s_t,a_t)\}_{t\geq 1}$, where $a_t$ is an $A$-valued random variable representing the action which brings the Markov chain to state $s_t$ from $s_{t-1}$: if $s_{t-1}$ is observed, then $a_t\sim\pi(s_{t-1})$. Thus, the expectation in~\eqref{eqn:pi-discounted-rewards} should be interpreted as
$
\bbE_{a_1\sim\pi(s_0)}\bbE_{s_1\sim P(s_0,a_1,\cdot)}\bbE_{a_2\sim\pi(s_1)}\cdots.
$ The state- and action-dependent discount factors accrue in a path-dependent fashion leading to the product in \eqref{eqn:pi-discounted-rewards}.
When the discount factors are state dependent, it is possible to define different optimization criteria; the choice~\eqref{eqn:pi-discounted-rewards} is commonly selected because it defines a value function which may be computed via dynamic programming. This choice is also natural in the context of financial applications\footnote{Consider the present value of an infinite stream of future cash payments $q_0,q_1,\ldots$, paid out at discrete time instances $t=0,1,\ldots$. If the risk-free interest rate over the period $[t,t+1]$ is given by $r_t$, then the present value of the payments is given by $C=q_0 + \sum_{t=1}^{\infty}\prod_{\tau=0}^{t-1}\gamma_{\tau}q_t$, where $\gamma_t=(1+r_t)^{-1}$.}.
The optimal value function $V^*$ is defined as $V^*(s)=\sup_{\pi\in\Pi}V^{\pi}(s)$ for all $s\in S$, where $\Pi$ is the set of all stationary stochastic policies, and the corresponding optimal policy $\pi^*$ is any policy achieving the optimal value function. Under the assumptions we have imposed here, a deterministic optimal policy exists whenever an optimal policy (possibly stochastic) exists~\citep[Sec. 1.1.4]{BertsekasVol2}. We will make use of stochastic policies primarily to regularize a class of MDP solution algorithms, rather than to achieve better solutions. 


The process of computing $V^{\pi}$ given $\pi$ is known as {\em value determination}. Following the usual approach, we may solve for $V^{\pi}$ by conditioning on the first transition in~\eqref{eqn:pi-discounted-rewards} and applying the Markov property. However, when $\pi$ is stochastic, the first transition also involves a randomly selected action, and when the discount factors are state/action dependent, the particular discounting seen in~\eqref{eqn:pi-discounted-rewards} must be adopted in order to obtain a linear system. One may derive the following equation for $V^\pi$ (details given in the Appendix)
\begin{equation}\label{eqn:vpi-linsys}
V^{\pi}(s) = \sum_{s',a}P(s,a,s')\pi(s,a)\bigl[R(s,a,s') + \Gamma(s,a,s')V^{\pi}(s')\bigr],\quad s\in S.
\end{equation}
In matrix-vector form this system may be written as
\[
V^{\pi} = \bigl(I- (\Gamma\circ P)^{\pi}\bigr)^{-1}r
\]
where $r:= (P\circ R)^{\pi}\bbone$.
The matrix
$
\bigl(I- (\Gamma\circ P)^{\pi}\bigr)^{-1}
$
will be referred to as the {\em potential operator}, or fundamental matrix, or Green's function, in analogy with the naming of the matrix $(I-P^{\pi})^{-1}$ for the Markov chain $P^{\pi}$.

\subsection{Notation}
\label{sec:notation}
We denote by $\{X_t\}_{t=0}^{\infty}$ a Markov chain, not necessarily time-homogeneous, governed by an appropriate transition matrix $P$.
For $S'\subseteq S$, we define the {\em restriction} of $P:S\times A\times S\to\bbR_{+}$ to $S'$ to be the transition tensor $P_{S'}:S'\times A\times S'\to\bbR_{+}$ defined by
\begin{equation}\label{eqn:P_cluster}
P_{S'}(s,a,s') = 
\begin{cases}
P(s,a,s') & \text{if } s,s'\in S', s\neq s'\\
P(s,a,s) + \sum_{s''\notin S'}P(s,a,s'') & \text{if } s = s', s\in S' \;.
\end{cases}
\end{equation}
The rewards from $R:S\times A\times S\to\bbR$ associated to transitions between states in the subset $S'$ remain unchanged:
\[
R_{S'}(s,a,s') = R(s,a,s'), \quad\text{for all } (s,a,s')\text{ such that }s,s'\in S', a\in A.
\]
We will refer to this operation as {\em truncation}, to distinguish it from restriction as defined by~\eqref{eqn:P_cluster}.
The sub-tensor $\Gamma_{S'}$ is similarly defined from $\Gamma$.  Note that, by definition, $P_{S'},R_{S'},\Gamma_{S'}$ do not include t-uples which start from a state $s$ in the \clustertext but which end at a state $s'$ outside of the  \clustertext.

The restriction operation introduced above does {\em{not}} commute with taking expectations with respect to a policy. The matrix $P_{S'}^{\pi}$ will be defined by {\em first} restricting to $S'$ by Equation~\eqref{eqn:P_cluster}, and {\em then} averaging $P_{S'}$ with respect to $\pi$ as in Equation~\eqref{eqn:PpiRpi}. Truncation {\em{does}} commute with expectation over actions, so $R_{S'}^{\pi},\Gamma_{S'}^{\pi}$ may be computed by truncating or averaging in any order, although it is clearly  more efficient to truncate before averaging. In fact, to define these and other related quantities,  $\pi$ need only be defined locally on the  \clustertext of interest. For quantities such as $(P_{S'}\circ R_{S'})^{\pi}$, we will always assume that truncation/restriction occurs before expectation.

Lastly, $\text{diag}(v)$ will denote a diagonal matrix with the elements of vector $v$ along the main diagonal, and $a\wedge b$ will denote the minimum of the scalars $a,b$.
\section{Multiscale Markov Decision Processes}
\label{sec:mmdp-toplevel}
The high-level procedure for efficiently solving a problem with a multiscale MDP hierarchy, which we will refer to as an {\em``MMDP''}, consists of the following steps, to be described individually in more detail below:
\begin{enumerate}\itemsep 0pt
 \item[{\em{Step 1}}] {\bf{Partition}} the statespace into subsets of states (``clusters'') connected via ``bottleneck'' states.
 \item[{\em{Step 2}}] Given the decomposition into clusters by bottlenecks, {\bf{compress}} or {\bf{homogenize}} the MDP into another, smaller and {\em{coarser}} MDP, whose statespace is the set of bottlenecks, and whose actions are given by following certain policies in \clusterstext connecting bottlenecks (``subtasks'').

Repeat the steps above with the compressed MDP as input, for a desired number of compression steps, obtaining a hierarchy of MDPs.
\item[{\em{Step 3}}] {\bf{Solve the hierarchy of MDPs}} from the top-down (coarse to fine) by pushing solutions of coarse MDPs to finer MDPs, down to the finest scale. 
\end{enumerate}
We say that the procedure above compresses or homogenizes, in a multiscale fashion, a given MDP. The construction is perfectly recursive, in the sense that the same steps and algorithms are used to homogenize one scale to the next coarser scale, and similarly for the refinement steps of a coarse policy into a finer policy. We may, and will, therefore focus on a single compression step and in a single refinement step. The compression procedure also enjoys a form of consistency across scales: for example, optimal value functions at homogenized scales are good approximations to homogenized optimal value functions at finer scales.
Moreover, actions at coarser scales are typically, as one may expect, complex, ``higher-level'' actions, and the above procedure may be thought of as producing different levels of ``abstraction'' of the original problem. While automating the process of hierarchically decomposing, in a novel fashion, large complex MDPs, the framework we propose may also yield significant computational savings: if at a scale $j$ there are $r_j$ \clusterstext of roughly equal size, and $n_j$ states, the solution to the MDP at that scale may be computed in time $\cO\bigl(r_j(n_j/r_j)^3\bigr)$. If $r_j=n_j/C$ and $n_j=n/C^j$ (with $n$ being the size of the original statespace), then the computation time across $\log n$ scales is $\cO\bigl(n\log n\bigr)$. We discuss computational complexity in Section~\ref{sec:ms-alg-complexity}, and establish convergence of a particular solution algorithm to the global optimum in Section~\ref{sec:ms-alg-proof}. 
Finally, the framework facilitates knowledge transfer between related MDPs. Sub-tasks and coarse solutions may be transferred anywhere within the hierarchies for a pair of problems, instead of mapping entire problems. We discuss transfer in Section~\ref{sec:transfer}.

The rest of this section is devoted to providing details and analysis for Steps~$(1)-(3)$ above in three subsections. Each of these subsections contains an overview of the construction needed in each step followed by  a more detailed and algorithmic discussion concerning specific algorithms used to implement the construction; the latter may be skipped in a first reading, in order to initially focus on ``the big picture''. Proofs of the results in these subsections are all postponed until the Appendix.

\subsection{{\em{Step 1:}} Bottleneck Detection and statespace Partitioning}
\label{sec:clustering}
The first step of the algorithm involves partitioning the MDP's statespace $S$ by identifying a set $\B\subseteq S$ of bottlenecks. The bottlenecks induce a partitioning\footnote{A partitioning $\clusters=\{\clusterindex_i\}_{i=1}^C$ is a family of disjoint sets $\clusterindex_i$ such that $S=\cup_{i=1}^C \clusterindex_i$.} of $S\setminus\B$ into a family $\clusters$ of connected components. Typically $\B$ depends on a policy $\pi$, and when we want to emphasize this dependency, we will write $\Bpi$. We always assume that {\em{$\Bpi$ includes all terminal states of $P^\pi$.}}
The partitioning of $S\setminus\Bpi$ induced by the bottlenecks is the set of equivalence classes $S/\mathord{\sim}$, under the relation  
\[
s_i\sim s_j, \quad \text{if } s_i,s_j\notin\Bpi \text{ and there is a path from $s_i$ to $s_j$ not passing through any $\bindex\in \Bpi$}.
\]
Clearly these equivalence classes yield a partitioning of $S\setminus\Bpi$.
The term {\em cluster} will refer to an equivalence class plus any  bottleneck states connected to states in the class: 
if $[s]:=\{s'~|~s\sim s'\}$ is an equivalence class, 
\[
\clusterindex([s]) := [s]\cup \bigl\{\bindex\in\Bpi~|~ P^{\pi}(s',\bindex)>0 \text{ or }P^{\pi}(\bindex,s')>0\text{ for some }s'\in[s]\bigr\}.
\]
The set of \clusterstext is denoted by $\clusters$.
If $\clusterindex=\clusterindex([s])$, $[s]$ will be referred to as the \clustertext's {\em interior}, denoted by $\Ui$, and the bottlenecks attached to $[s]$ will be referred to as the \clustertext's {\em boundary}, denoted by $\boundary\clusterindex$.

To each \clustertext\ $\clusterindex$, and policy $\pi$ (defined on at least $\clusterindex$), we associate the Markov process with transition matrix $P^\pi_\clusterindex$, defined according to  Section~\ref{sec:notation}.

We also assume that a set of designated policies $\picluster_{\clusterindex}$ is provided for each cluster $\clusterindex$. For example $\picluster_{\clusterindex}$ may be the singleton consisting of the diffusion policy in \clusterindex. Or $\picluster_{\clusterindex}$ could be the set of locally optimal policies in $\clusterindex$ for the family of MDPs, parametrized by $s'\in\dcluster$ with reward equal to the original rewards plus an additional reward when $s'$ is reached (this approach is detailed in Section~\ref{sec:cluster_pols}).

Finally, we say that $\dcluster$ is $\pi$-{\em{reachable}}, for a policy $\pi$, if  the set $\dcluster$ can be reached in a finite number of steps of $P^\pi_\clusterindex$, starting from any initial state $s\in\clusterindex$.

\subsubsection{Algorithms for Bottleneck Detection}
\label{sec:clustering_details}
\label{sec:diffmaps}
In the discussion below we will make use of diffusion map embeddings~\citep{Coifman:PNAS:05} as a means to cluster, visualize and compare directed, weighted statespace graphs. This is by no means the only possibility for accomplishing such tasks, and we will point out other references later. Here we focus on this choice and the details of diffusion maps and associated hierarchical clustering algorithms.

\noindent{\em Diffusion maps} are based on a Markov process, typically the random walk on a graph. The random walks we consider here are of the form $P^\pi$ for some policy $\pi$, and may always be made reversible (by addition of the teleport matrix adding weak edges connecting every pair of vertices, as in Step (2) of Algorithm~\ref{alg:spectral_clustering}), but may still be strongly directed particularly as $\pi$ becomes more directed. In light of this directedness, we will compute diffusion map embeddings of the underlying states from normalized graph Laplacians symmetrized with respect to the underlying Markov chain's invariant distribution, following~\citep{ChungDirected}:
\begin{equation}\label{eqn:directed-lap}
\cL = I - \tfrac{1}{2}\bigl(\Phi^{1/2}P^{\pi}\Phi^{-1/2} + \Phi^{-1/2}(P^{\pi})^{\tr}\Phi^{1/2}\bigr)
\end{equation}
where $\Phi$ is a diagonal matrix with the invariant distribution satisfying $(P^{\pi})^{\tr}\mu = \mu$ placed along
the main diagonal, $\Phi_{ii}=\mu_i$. One can choose an
orthonormal set of eigenvectors $\{\Psi^{(k)}\}_{k\geq 0}$ with corresponding real eigenvalues $\lambda_k$ which diagonalize $\cL$.
If we place the eigenvalues in ascending order $0=\lambda_{0}\leq\cdots\leq\lambda_{n}$, the diffusion map embedding of the state $s_i$ is given by
\begin{equation}\label{eqn:diffmap-embed}
s_i\mapsto \bigl(\Psi^{(k)}_i(1-\lambda_{k})\bigr)_{k\geq 1},\qquad s_i\in S .
\end{equation}
The {\em diffusion distance} between two states $s_i, s_j$ is given by the Euclidean distance between embeddings,
\[
d^2(s_i,s_j) = \sum_{k\geq 1}(1-\lambda_{k})^2\bigl|\Psi^{(k)}_i - \Psi^{(k)}_j\bigr|^2, \qquad s_i ,s_j\in S \;.
\]
See~\citep{Coifman:PNAS:05} for a detailed discussion. Often times this distance may be well approximated
(depending on the decay of the spectrum of $\cL$) by truncating the sum after $p<n$ terms, in which case only
the top $p$ eigenvectors need to be computed.

In some cases we will need to align the signs of the eigenvectors for two given Laplacians $\cL,\widetilde{\cL}$
towards making diffusion map embeddings for different graphs more comparable.
If $\{\Psi^{(k)}\}_{k\geq 0}$ and  $\{\widetilde{\Psi}^{(k)}\}_{k\geq 0}$ denote the respective sets of eigenvectors, and the eigenvalues of both Laplacians are distinct\footnote{The case of repeated eigenvalues may be treated similarly by generalizing the sign flipping operation $\tau$ to an orthogonal transformation of the subspace spanned by the eigenvectors sharing the repeated eigenvalue.}, we can
define the sign alignment vector $\tau$ as
\begin{equation}\label{eqn:tau-signs}
\tau_k =
\begin{cases}
+1 & \text{if } \sgn\Psi^{(k)}_1=\sgn\widetilde{\Psi}^{(k)}_1 \\
-1 & \text{otherwise}
\end{cases}.
\end{equation}

Given an alignment vector $\tau$, one can extend the above diffusion distance to a distance defined on a union of statespaces. If $S,\widetilde{S}$ are statespaces with embeddings~\eqref{eqn:diffmap-embed} respectively defined by $\{(\Psi^{(k)},\lambda_k)\}_{k=0}^p$, $\{(\widetilde{\Psi}^{(k)},\widetilde{\lambda}_k)\}_{k=0}^p$ for some $p\geq 1$, then we can define the distance $d^2: (S\cup\widetilde{S})\times (S\cup\widetilde{S}) \to \bbR_{+}$ as \\
\begin{equation*}
d^2(s_i,s_j) :=
\begin{cases}
\rho(s_i,s_j) & \text{if } s_i\in S, s_j\in\widetilde{S} \\
\rho(s_j,s_i) & \text{otherwise}.
\end{cases}
\end{equation*}
using 
\[
\rho(s_u,s_v)= \sum_{k\geq 1}(1-\lambda_{k})(1-\widetilde{\lambda}_k)\bigl|\tau_k\Psi^{(k)}_u - \widetilde{\Psi}^{(k)}_v\bigr|^2
\]
with $\tau$ defined by~\eqref{eqn:tau-signs}.\\

\noindent{\em Hierarchical clustering}.
 Given a policy $\pi$, we can construct a weighted statespace graph $G$ with vertices corresponding to states, and edge weights given by $P^{\pi}$. A policy that allows thorough exploration, such as the diffusion policy $\pi^u$, can be chosen to define the weighted statespace graph. 

The hierarchical spectral clustering algorithm we will consider recursively splits the statespace graph into pieces by looking for low-conductance cuts. The spectrum of the symmetrized Laplacian for directed graphs~\cite{ChungDirected} is used to determine the graph cuts at each step. The sequence of cuts establishes a partitioning of the statespace, and  bottleneck states are states with edges that are severed by any of the cuts. Algorithm~\ref{alg:spectral_clustering} describes the process. 
Other more sophisticated algorithms may also be used, for example Spielman's~\citep{Spielman:LocalClustering} and that of Anderson \& Peres~\citep{Andersen:2009:ESP,Morris_Peres_2003}. One may also consider ``model-free'' versions of the algorithms above, that only have access to a ``black-box'' computing the results of running a process (truncated random walk, evolving sets process, respectively, for the references above), but we do not pursue this here.
\begin{algorithm}[t]
\caption{Recursive spectral partitioning.}
\fbox{\begin{minipage}{0.984\textwidth}
\begin{enumerate}\itemsep 0pt
\item Restrict $P^{\pi}$ to non-absorbing states.
\item Set $P^{\pi}_{\text{tel}} =(1-\eta)P^{\pi} + \eta n^{-1}\bbone\bbone^{\tr}$, for some small, positive $\eta$.
\item Find the eigenvector (invariant distribution) $\mu$ satisfying $(P^{\pi}_{\text{tel}})^{\tr}\mu= \mu$.
\item Let $\Phi=\text{diag}(\mu)$ and compute the symmetrized Laplacian 
\begin{equation*}
\cL = I - \tfrac{1}{2}\bigl(\Phi^{1/2}P_{\text{tel}}^{\pi}\Phi^{-1/2} + \Phi^{-1/2}(P_{\text{tel}}^{\pi})^{\tr}\Phi^{1/2}\bigr).
\end{equation*}
\item Compute the $K$ eigenvectors of $\cL$ corresponding to the $K$ smallest non-trivial eigenvalues 
$\lambda_1<\cdots<\lambda_K$.
\item Define a set of cuts $\cZ$ by sweeping over thresholds ranging from the smallest entry of $\Psi^{(i)}$ to the largest, for all eigenvectors $\Psi^{(i)}, i=1,\ldots,K$. The points for which $\Psi^{(i)}$ are above/below the given threshold defines the states $Z,Z^{c}\subset S$ on either side of the cut.
\item Choose the cut $Z^*=\arg\min_{Z\in\cZ}\varphi(Z)$ with minimum conductance $$\varphi(Z)=\frac{\sum_{i\in Z}\sum_{j\in Z^c} P^{\pi}_{ij}}{\text{vol}(Z)\wedge\text{vol}(Z^c)},$$ where $\text{vol}(Z)=\sum_{i\in Z}\sum_{j\in S} P^{\pi}_{ij}$.
\item Identify bottleneck states as the states in $Z^*$ on one (and only one) side of the edges in $P^{\pi}$ severed by the cut, choosing the side which gives the smallest bottleneck set.
\item Store the partition of the statespace given by the cut.
\item Unless stopping criteria is met, run the algorithm again on each of the two subgraphs resulting from the cut.
\end{enumerate}
\end{minipage}}
\label{alg:spectral_clustering}
\end{algorithm}
A recursive application of Algorithm~\ref{alg:spectral_clustering} produces a set of bottlenecks $\Bpi$.
Each bottleneck and partition discovered by the clustering algorithm is associated with a spatial scale determined by the recursion depth. The finest scale consists of the finest partition and includes all  bottlenecks. The next coarser scale includes all the bottlenecks and partitions discovered up to but not including the deepest level of the recursion, and so on. In this manner the statespaces and actions of all the MDPs in a multi-scale hierarchy can be pre-determined, although if desired one can also apply clustering to the coarsened statespaces {\em after} compressing using the compressed MDP's transition matrix as graph weights. The addition of a teleport matrix in Algorithm~\ref{alg:spectral_clustering} (Step 2) guarantees that the equivalence classes partition $\{S\setminus\Bpi\}$ and are strongly connected components of the weighted graph defined by $P^{\pi}_{\text{tel}}$.  

Because graph weights are determined by $P^{\pi}$ in this algorithm, which bottlenecks will be identified generally depends on the policy $\pi$. In this sense there are two types of ``bottlenecks'': problem bottlenecks and geometric bottlenecks. Geometric bottlenecks may be defined as interesting regions of the statespace alone, as determined by a random walk exploration if $\pi$ is a diffusion policy (e.g. $\pi^u$). Problem bottlenecks are regions of the statespace which are interesting from a geometric standpoint {\em and} in light of the goal structure of the MDP. If the policy is already strongly directed according to the goals defined by the rewards, then the bottlenecks can be interpreted as choke points for a random walker in the presence of a strong potential.

\subsection{{\em{Step 2:}} Multiscale Compression and the Structure of Multiscale Markov Decision Problems}
\label{sec:compression}
Given a set of bottlenecks $\B$ and a suitable fine scale policy, we can {\em compress} (or {\em homogenize}, or coarsen) an MDP {\em into another MDP with statespace $\B$}. The coarse MDP can be thought of as a low-resolution version of the original problem, where transitions {\em between} \clusterstext are the events of interest, rather than what occurs within each  \clustertext. As such, coarse MDPs may be vastly simpler: the size of the coarse statespace is on the order of the number of \clusterstext, which may be small relative to the size of the original statespace. Indeed, \clusterstext may be generally thought of as geometric properties of a {\em problem}, and are constrained by the inherent complexity of the problem, rather than the choice of statespace representation, discretization or sampling. 

A solution to the coarse MDP may be viewed as a coarse solution to the original fine scale problem. An optimal coarse policy describes how to solve the original problem by specifying which sub-tasks to carry out and in which order. As we will describe in Section~\ref{sec:mdp_solution}, a coarse value function provides an efficient means to compute a fine scale value function and its associated policy. Coarse MDPs and their solutions also provide a framework for systematic transfer learning; these ideas are discussed in detail in Section~\ref{sec:transfer}.  

We have discussed how to identify a set of bottleneck states in Section~\ref{sec:clustering_details} above. As we will explain in detail below, a policy is required to compress an MDP. This policy may encode a priori knowledge, or may be simply chosen to be the diffusion policy. In Section~\ref{sec:cluster_pols} below, we suggest an algorithm for determining good local policies for compression that can be expected to produce an MDP at the coarse scale whose optimal solution is compatible with the gradient of the optimal value function at the fine scale. 

A homogenized, coarse scale MDP will be denoted by the tuple $(\widetilde{S},\widetilde{A},\widetilde{P},\widetilde{R},\widetilde{\Gamma})$.  We first give a brief description of the primary ingredients needed to define a coarse MDP, with a more detailed discussion to follow in the forthcoming subsections.
\begin{itemize}\itemsep 0pt
\item {\textbf{Statespace $\widetilde{S}$}}: The coarse scale statespace $\widetilde{S}$ is the set of bottleneck states $\B$ for the fine scale, obtained by clustering the fine scale statespace graph, for example with the methods described in Section~\ref{sec:clustering}. Note that $\widetilde{S}\subset S$.
\item {\textbf{Action set $\widetilde{A}$}}: A coarse action invoked from $\bindex\in\widetilde{S}=\B$ consists of executing a given fine scale policy $\pi_{\clusterindex}\in\picluster_\clusterindex$ within the fine scale \clustertext\ $\clusterindex$, starting from $\bindex\in\boundary\clusterindex$ (at a time that we may reset to $0$), until the first positive time at which a bottleneck state in $\boundary\clusterindex$ is hit.  Recall that in each \clustertext\ $\clusterindex$ we have a set of policies $\picluster_\clusterindex$.
\item {\textbf{Coarse scale transition probabilities $\widetilde{P}(s,a,s')$}}: If $a\in\widetilde{A}$ is an action executing the policy $\pi_{\clusterindex}\in\picluster_\clusterindex$, then $\widetilde{P}(s,a,s')$ is defined as the probability that the Markov chain $\Ppirestrtocluster$ started from $s\in\widetilde{S}$, hits $s'\in\widetilde{S}$ before hitting any other bottleneck. In particular, $\widetilde{P}(s,a,s')$ may be nonzero only when $s,s'\in\boundary\clusterindex$ for some $\clusterindex\in\clusters$.
\item {\textbf{Coarse scale rewards $\widetilde{R}(s,a,s')$}}: The coarse reward $\widetilde{R}(s,a,s')$ is defined to be the expected total discounted reward collected along trajectories of the Markov chain associated to action $a$ described above, which start at $s\in\widetilde{S}$ and end by hitting $s'\in\widetilde{S}$ before hitting any other bottleneck. 
\item {\textbf{Coarse scale discount factors $\widetilde{\Gamma}(s,a,s')$}}: The coarse discount factor $\widetilde{\Gamma}(s,a,s')$ is the expected product of the discounts applied to rewards along trajectories of the Markov chain $\Ppirestrtocluster$ associated to a action $a\in\widetilde{A}$, starting at $s\in\widetilde{S}$ and ending at $s'\in\widetilde{S}$. 
\end{itemize}

One of the important consequences of these definitions is that {\em the optimal fine scale value function on the bottlenecks is a good solution to the coarse MDP, compressed with respect to the optimal fine scale policy, and vice versa.}
It is this type of consistency across scales that will allow us to take advantage of the construction above and design efficient multiscale algorithms. 

The compression process is reminiscent of other instances of conceptually related procedures: coarsening (applied to meshes and PDEs, for example), homogenization (of PDEs), and lumping (of Markov chains). The general philosophy is to reduce a large problem to a smaller one constructed by ``locally averaging'' parts of the original problem.

The coarsening step may always be accomplished computationally by way of Monte Carlo simulations, as it involves computing the relevant statistics of certain functionals of Markov processes in each of the \clusterstext. As such, the computation is embarrassingly parallel\footnote{Moreover, it does not require a priori knowledge of the fine details of the models in each \clustertext, but only requires the ability to call a ``black box'' which simulates the prescribed process in each \clustertext, and computes the corresponding functional (in this sense coarsening becomes model-free).}. While this gives flexibility to the framework above, it is interesting to note that many of the relevant computations may in fact be carried out analytically, and that eventually they reduce to the solution of multiple independent (and therefore trivially parallelizable) small linear systems, of size comparable to the size of a  \clustertext. In this section we develop this analytical framework in detail (with proofs collected in the Appendix), as they uncover the natural structure of the multiscale hierarchy we introduce, and lead to efficient, ``explicit'' algorithms for the solution of the Markov decision problems we consider. The rest of this section is somewhat technical, and on a first reading one may skip directly to Section~\ref{sec:mdp_solution} where we discuss the multiscale solution of hierarchical MDPs obtained by our  construction.


\subsubsection{Assumptions}
We will always assume that the fine scale policy $\pi$ used to compress has been {\em regularized}, by blending with a small amount of the diffusion policy $\pi^u$:
\begin{equation*}\label{eqn:generic_pol_blending}
\pi(s,\cdot)\gets\lambda\pi^u(s,\cdot) + (1-\lambda)\pi(s,\cdot),\qquad s\in S
\end{equation*}
for some small, positive choice of the regularization parameter $\lambda$. In particular we will assume this is the case everywhere quantities such as $P^{\pi}$ appear below. The goal of this regularization is to address, or partially address, the following issues:
\begin{itemize}
\item The solution process may be initially biased towards one particular (and possibly incorrect) solution, but this bias can be overcome when solving the coarse MDP as long as the regularization described above is included every time compression occurs during the iterative solution process we will describe in Section~\ref{sec:mdp_solution}.
\item Directed policies can yield a fine scale transition matrix which, when restricted to a \clustertext, may render bottleneck (or other) states unreachable. We require the boundary of each \clustertext to be $\pi$-reachable, and this is often guaranteed by the regularization above except in rather pathological situations. If  any interior states violate this condition, they can be added to the \clustertext's boundary and to the global bottleneck set at the relevant scale\footnote{In fact, if any such state is an element of a closed, communicating class, then the entire class can be lumped into a single state and treated as a single bottleneck. Thus, the bottleneck set does not need to grow with the size of the closed class from which the boundary is unreachable. For simplicity however, we will assume in the development below that no lumped states of this type exist.}.
We note that these assumptions are significantly weaker than requiring that the subgraphs induced by the restrictions $P_\clusterindex^{\pi}$, $\pi\in\picluster_\clusterindex$ of $P$ to a \clustertext are strongly connected components; the Markov chain defined by $P_\clusterindex^{\pi}$ need not be irreducible.
\item Compression with respect to a deterministic and/or incorrect policy should not preclude transfer to other tasks. In the case of policy transfer, to be discussed below, errors in a policy used for compression can easily occur, and can lead to unreachable states. Policy regularization helps alleviate this problem.
\end{itemize}

\subsubsection{Actions}

An action $\widetilde{A}$ at $s\in\widetilde{S}$ for the compressed MDP consists of executing a policy $\pi_\clusterindex\in\picluster_\clusterindex$ at the fine scale, starting in $s$, within some \clustertext\ $\clusterindex$ having $s$ on its boundary, until hitting a bottleneck state in $\clusterindex$. The number of actions is equal to the total number of policies across \clusterstext,$|\widetilde{A}|=\sum_{\clusterindex\in\clusters}|\picluster_\clusterindex|$. 
We now fix a \clustertext\ $\clusterindex$ and a policy $\pi_\clusterindex\in\picluster_\clusterindex$.
The corresponding local Markov transition matrix is $\Ppirestrtocluster$, and let $\Rpirestrtocluster$ denote the reward structure, and $\Gpirestrtocluster$ denote the system of discount factors, following Section~\ref{sec:notation}. Let $((X^{\pi_\clusterindex}_{\clusterindex})_n)_{n\geq 0}$ denote the Markov chain with transition matrix $\Ppirestrtocluster$. If the coarse action is invoked in state $s\in\widetilde{S}$, then we set $X_0=s$. The set of actions available at $s\in\widetilde{S}$ for the compressed MDP is given by
\[
\widetilde{A}(s) := \bigcup_{\clusterindex\in\clusters:s\in\boundary\clusterindex} \Bigl\{\text{``run the MRP }(\Ppirestrtocluster,\Rpirestrtocluster,\Gpirestrtocluster)\text{ in } \clusterindex\text{ until the first } n>0 :(X^{\pi_\clusterindex}_{\clusterindex})_n\in\B\text{''}\Bigr\}_{\pi_{_\clusterindex}\in\picluster_\clusterindex} \,.
\]
A Markov reward process (MRP) refers to an MDP with a fixed policy and corresponding $P, R, \Gamma$ restricted to that fixed policy. The actions above involve running an MRP because while the action is being executed the policy remains fixed. 

Consider the simple example shown on the left in Figure~\ref{fig:meta_ex}, where we graphically depict a simple coarse MDP (large circles and bold edges) superimposed upon a stylized fine scale statespace graph (light gray edges; vertices are edge intersections). Undirected edges between coarse states are bidirectional. Dark gray lines delineate four \clusterstext, to which we have associated fine scale policies $\pi_1,\ldots,\pi_4$. The bottlenecks are the states labeled $s_1,\ldots,s_4$. If an agent is in state $s_1$, for example, the actions ``Execute $\pi_1$'' and ``Execute $\pi_4$'' are feasible. If the agent takes the coarse action $a=\text{``Execute }\pi_1\text{''}$, then it can either reach $s_2$ or come back to $s_1$, since this action forbids traveling outside of \clustertext 1 (top right quadrant). On the other hand if $\pi_4$ is executed from $s_1$, then the agent can reach $s_4$ or return to $s_1$, but the probability of transitioning to $s_2$ is zero. 

In general, the compressed MDP will have action and state dependent rewards and discount factors, even if the fine scale problem does not. In Figure~\ref{fig:meta_ex} (left), the coarse states straddle two \clusterstext each, and therefore have different self loops corresponding to paths which return to the starting state within one of the two \clusterstext. So $\widetilde{R}$ and $\widetilde{\Gamma}$ apparently depend on actions. But, we may reach $s_1$ when executing $\pi_1$ starting from either $s_1$ or from $s_2$, so the compressed quantities in fact depend on both the actions and the source/destination states. Figure~\ref{fig:meta_ex} (right) shows another example, where the dependence on source states is particularly clear. Even if the action corresponding to running a fine policy in the center square is the same for all states, each coarse state $s_1,\ldots,s_4$ may be reached from two neighbors as well as itself.

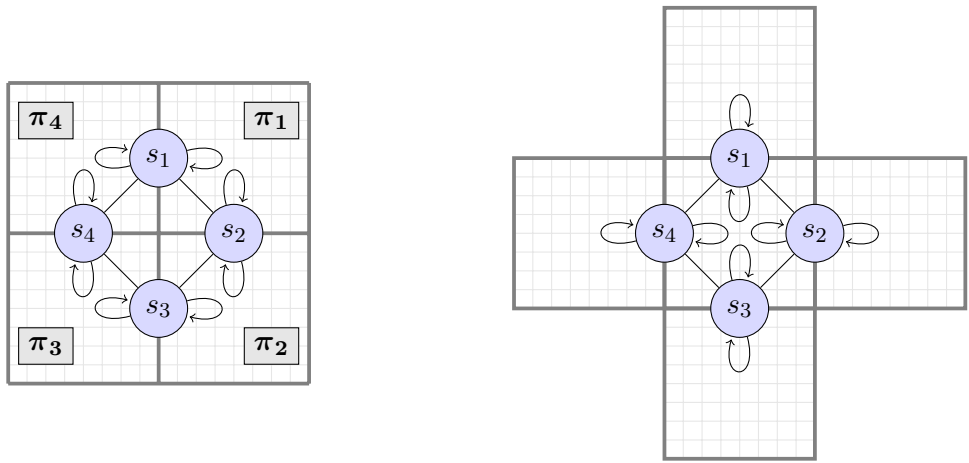
\begin{figure}[t]
\centering
\begin{minipage}[c]{0.4\textwidth}
\centering
\begin{tikzpicture}
\definecolor{ltgray}{rgb}{0.9, 0.9,0.9}
\definecolor{dkgreen}{rgb}{0.2,0.7,0.1}
\draw[step=0.25cm,color=ltgray] (0,0) grid (4,4);
\draw[step=2cm,color=gray,line width=0.5mm] (0,0) grid (4,4);
\tikzstyle{every state}=[fill=blue!15,draw,minimum width=0.5cm,minimum height=0.5cm]
\node[state] (s0) at (2,3)  {$s_1$};
\node[state] (s1) at (3,2)  {$s_2$};
\node[state] (s2) at (2,1)  {$s_3$};
\node[state] (s3) at (1,2)  {$s_4$};
\path (s0) edge (s1)
      (s0) edge [loop left] (s0)
      (s0) edge [loop right] (s0)
      (s1) edge (s2)
      (s1) edge [loop above] (s1)
      (s1) edge [loop below] (s1)
      (s2) edge (s3)
      (s2) edge [loop left] (s2)
      (s2) edge [loop right] (s2)
      (s3) edge (s0)
      (s3) edge [loop above] (s3)
      (s3) edge [loop below] (s3);
\path (0.5,0.5) node[draw,fill=ltgray] {$\boldsymbol{\pi}_{\bf 3}$};
\path (0.5,3.5) node[draw,fill=ltgray] {$\boldsymbol{\pi}_{\bf 4}$};
\path (3.5,0.5) node[draw,fill=ltgray] {$\boldsymbol{\pi}_{\bf 2}$};
\path (3.5,3.5) node[draw,fill=ltgray] {$\boldsymbol{\pi}_{\bf 1}$};
\end{tikzpicture}
\end{minipage}
\hskip 1.5cm
\begin{minipage}[c]{0.4\textwidth}
\centering
\begin{tikzpicture}
\definecolor{ltgray}{rgb}{0.9, 0.9,0.9}
\draw[step=0.25cm,color=ltgray] (-2,0) grid (4,2);
\draw[step=0.25cm,color=ltgray] (0,-2) grid (2,4);
\draw[color=gray,line width=0.5mm] (-2,0) -- (4,0) -- (4,2) -- (-2,2) -- cycle;
\draw[color=gray,line width=0.5mm] (0,4) -- (0,-2) -- (2,-2) -- (2,4) -- cycle;
\draw[color=gray,line width=0.5mm] (0,0) -- (0,2) -- (2,2) -- (2,0) -- cycle;
\tikzstyle{every state}=[fill=blue!15,draw,minimum width=0.5cm,minimum height=0.5cm]
\node[state] (s0) at (1,0)  {$s_3$};
\node[state] (s1) at (2,1)  {$s_2$};
\node[state] (s2) at (1,2)  {$s_1$};
\node[state] (s3) at (0,1)  {$s_4$};
\path (s0) edge (s1)
      (s0) edge [loop above] (s0)
      (s0) edge [loop below] (s0)
      (s1) edge (s2)
      (s1) edge [loop left] (s1)
      (s1) edge [loop right] (s1)
      (s2) edge (s3)
      (s2) edge [loop above] (s2)
      (s2) edge [loop below] (s2)
      (s3) edge (s0)
      (s3) edge [loop left] (s3)
      (s3) edge [loop right] (s3);
\end{tikzpicture}
\end{minipage}
\caption{{\small\em (Left) Graphical depiction of a simple coarse MDP. The states $s_1,\ldots,s_4$ enclosed in circles are bottleneck states, and act as gateways between the four \clusterstext. The bottleneck states each connect bidirectionally to two other neighboring bottleneck states as well as to themselves (thin black edges). Two self loops are shown per state to emphasize that coarse self transitions can be different depending on the action taken. 
(Right) Another example emphasizing that coarse probabilities, rewards and discounts are in general both state and action dependent. See text for details.}}
\label{fig:meta_ex}
\end{figure}

\subsubsection{Transition Probabilities}
\label{sec:trans_probs}
Consider the \clustertext $\clusterindex$ referred to by a coarse action $a\in\widetilde{A}$.
{\em{The transition probability $\widetilde{P}(s,a,s')$ for $s,s'\in\dcluster\subseteq\widetilde{S}$ is defined as the probability that a trajectory in $\clusterindex\subset S$ hits state $s'$ starting from $s$ before hitting any other state in $\B$ (including itself) when running the fine scale MRP restricted to $\clusterindex$ and along the policy determined by the action $a$.}}

If $s$ is a state not in the \clustertext associated to $a$, then $a$ is not a feasible action when in state $s$. For the example shown in Figure~\ref{fig:meta_ex} (left), for instance, the edge weight connecting $s_1$ and $s_2$ is the probability that a trajectory reaches $s_2$ before it can return to $s_1$, when executing $\pi_1$. These probabilities may be estimated either by sampling (Monte Carlo simulations), or computed analytically. The first approach is trivially implemented, with the usual advantages (e.g. parallelism, access to a black-box simulator is all is needed) and disadvantages (slow convergence); here we develop the latter, which leads to a concise set of linear problems to be solved, and sheds light on both the mathematical and computational structure. 
Since the bottlenecks partition the statespace into disjoint sets, the probabilities $\widetilde{P}(s,a,s')$ can be quickly computed in each \clustertext separately.

\begin{proposition}
Let  $a$ be the action corresponding to executing a policy $\pi_\clusterindex$ in \clustertext\ $\clusterindex$. Then
\[
\widetilde{P}(s,a,s') = H_{s,s'},\quad  \text{ for all } s,s'\in \dcluster ,
\]
where $H$ is the {\em minimal non-negative solution}, for each $s'\in\dcluster$, to the linear system
\begin{equation}\label{eqn:trans_probs}
H_{s,s'} =  \Ppirestrtocluster(s,s') + \sum_{s''\in\Ui}\Ppirestrtocluster(s,s'')H_{s'',s'},\quad s\in\clusterindex, s'\in\dcluster \,,
\end{equation}
or in matrix-vector form,
\begin{equation*}\label{eqn:trans_probs_matvec}
\bigl(I - \Ppirestrtocluster(I-J)\bigr)H = \Ppirestrtocluster
\end{equation*}
where $J$ is a matrix of all zeros except $J_{s''s''}=1$ for $s''\in \dcluster $. 
\label{prop:coarsePsas}
\end{proposition}
\begin{corollary}
Consider the partitioning 
$$
\Ppirestrtocluster =
\begin{pmatrix}
Q & B\\
C & D
\end{pmatrix},
\qquad
H =
\begin{pmatrix}
h_q \\
h_b
\end{pmatrix}
$$
where the blocks $Q, D$ describe the interaction among non-bottleneck and bottleneck states within \clustertext\ $\clusterindex$ respectively. The compressed probabilities may be obtained by finding the minimal non-negative solution to 
\begin{equation*}\label{eqn:trprob_submtx_prop}
(I-Q)h_q = B
\end{equation*}
followed by computing 
$$h_b = D + Ch_q,$$ 
where $h_b$ is the transition probability matrix of the compressed MDP given the action $a$.
\end{corollary}
The proof of this Proposition and a discussion are given in the Appendix.

When deriving the the compressed rewards and discount factors below, we will need to reference the set of all pairs of bottlenecks $s,s'$ for which the probability of reaching $s'$ starting from $s$ is positive, when executing the policy $\pi_\clusterindex$ associated to $a$. Having defined $\widetilde{P}$, this set may be easily characterized as
\[
\supp_a(\widetilde{P}):=\{(s,s')\in\dcluster ~|~ \widetilde{P}(s,a,s') > 0\}
\]
where $\clusterindex$ is the \clustertext associated to the coarse action $a$.

\subsubsection{Rewards}
\label{sec:coarse_rewards}
{\em{The rewards $\widetilde{R}=\widetilde{R}(s,a,s')$, with $s,s'\in\dcluster$ and $a\in\widetilde{A}$,
are defined to be the expected discounted rewards collected along trajectories that start from $s$ and hit $s'$ before hitting any other bottleneck state in $\dcluster $, when running the fine-scale MRP restricted to the \clustertext $\clusterindex$ associated to $a$}}. 

In general, rewards under different policies and/or in other \clusterstext are calculated by repeating the process described below for different choices of $\pi\in\picluster_\clusterindex,  \clusterindex\in\clusters$.
As was the case in the examples shown in Figure~\ref{fig:meta_ex}, even if the fine scale MDP rewards do not depend on the source state or actions, the compressed MDP's rewards will, in general, depend on the source, destination and action taken. However as with the coarse transition probabilities, the relevant computations involve at most a single \clustertext's subgraph at a time.

Given a policy $\pi_\clusterindex$ on \clustertext\ $\clusterindex$, consider the Markov chain $(X_t)_{t\geq 0}$ with transition matrix $\Ppirestrtocluster$. Let $T$ and $T'$ be two arbitrary stopping times satisfying $0\leq T<T'<\infty$ (a.s.). The discounted reward accumulated over the interval $T\leq t\leq T'$ is given by the random variable
\[
R_{T}^{T'}:= R(X_T,a_{T+1}, X_{T+1}) + \sum_{t=T+1}^{T'-1}\left[
\prod_{\tau=T}^{t-1}\Gamma\bigl(X_{\tau},a_{\tau+1}, X_{\tau+1}\bigr)\right]
R\bigl(X_t,a_{t+1}, X_{t+1}\bigr) 
\]
where $a_{t+1}\sim\pi_\clusterindex(X_t)$ for $t=T,\ldots,T'-1$, and we set $R_T^T \equiv 0$ for any $T$.

Next, define the hitting times of $\dcluster$: 
\[
T_m =\inf\{t>T_{m-1} ~|~ X_t\in \dcluster \},\qquad m=1,2,\ldots
\]
with $T_0 = \inf\{t\geq 0 ~|~ X_t\in \dcluster \}$. Note that if the chain is started in a bottleneck state $X_0=\bindex\in\dcluster$, then clearly $T_0=0$. We will be concerned with the rewards accumulated between these successive hitting times, and by the Markovianity of $(X_t)_t$, we may, without loss of generality, consider the reward between $T_0$ and $T_1$, namely $R_{T_0}^{T_1}$.

The following proposition describes how to compute the expected discounted rewards by solving a collection of linear systems.

\begin{proposition}\label{prop:coarse_rewards}
Suppose the coarse scale action $a\in\widetilde{A}$ corresponds to executing a policy $\pi_\clusterindex$ in \clustertext\ $\clusterindex$, and let $(X_t)_{t\geq 0}$ denote  the Markov chain with transition matrix $\Ppirestrtocluster$. The rewards $\widetilde{R}$ at the coarse scale may be characterized as
\[
\widetilde{R}(s,a,s') = \bbE_s[R_{0}^{T_1} ~|~ X_{T_1}=s'], \qquad (s,s')\in\supp_a(\widetilde{P})\,.
\]
Moreover, for fixed $a$, $\widetilde{R}(s,a,s')=:H_{s,s'}$ may be computed by finding the (unique, bounded) solution $H$ to the linear system
\begin{subnumcases}{H_{s,s'} =}
\smashoperator[r]{\sum_{\substack{s''\in\Ui\,\cap\,\Uprime,\\ a\in A}}} P_{h_{s'}}(s,a,s'')\Gamma(s,a,s'')H_{s'',s'} 
 + \smashoperator[r]{\sum_{\substack{s''\in \Uprime,\\ a\in A}}} P_{h_{s'}}(s,a,s'')R(s,a,s''),\hspace*{-1em}
  & if $s\in\Ui\cap \Uprime$\label{eqn:coarse_rews_linsys}\\
\smashoperator[r]{\sum_{\substack{s''\in\Ui\,\cap\,\Uprime,\\ a\in A}}} P_{\tilde{h}_{s'}}(s,a,s'')\Gamma(s,a,s'')H_{s'',s'} 
+ \smashoperator[r]{\sum_{\substack{s''\in \Uprime,\\ a\in A}}} P_{\tilde{h}_{s'}}(s,a,s'')R(s,a,s''),\hspace*{-1em} 
& if  $(s,s')\in\supp_a(\widetilde{P})$\vspace{0.5cm}\hspace{3em} \label{eqn:coarse_rews_sums}
\end{subnumcases}
where $\Uprime:=\{s\in \clusterindex~|~h_{s'}(s)>0\}$; 
\begin{equation*}\label{eqn:Ph_defn_prop}
P_{h_{s'}}(s,a,s'') := \frac{P_{\clusterindex}(s,a,s'')\pi_\clusterindex(s,a)h_{s'}(s'')}{h_{s'}(s)}
\end{equation*}
for $s\in \Ui\cap\Uprime$,  $a\in A$, $s''\in \Uprime$; and
\begin{equation*}\label{eqn:Phtilde_defn_prop}
P_{\tilde{h}_{s'}}(s,a,s'') := \frac{P_{\clusterindex}(s,a,s'')\pi_\clusterindex(s,a)h_{s'}(s'')}{\widetilde{P}(s,a,s')}
\end{equation*}
for $(s,s')\in\supp_a(\widetilde{P})$, $a\in A$, $s''\in \Uprime$; with 
 $h_{s'}(s):=\bbP_s(X_{T_0}=s'),$ for $s\in\clusterindex, s'\in\dcluster$ denoting the minimal, non-negative, harmonic function satisfying
\[
h_{s'}(s) =
\begin{cases}
\delta_{s,s'} & s\in\dcluster \\
\Ppirestrtocluster(s,s') + \sum_{s''\in\Ui}\Ppirestrtocluster(s,s'')h_{s'}(s'') & s\in\Ui \,.
\end{cases}
\]
\end{proposition}

Thus, for each $s'\in\dcluster $, the compressed rewards $\widetilde{R}(s,a,s')$ are computed by first solving a linear system of size at most $|\Ui|\times |\Ui|$ given by~\eqref{eqn:coarse_rews_linsys}, and then computing at most $|\dcluster |$ of the sums given by~\eqref{eqn:coarse_rews_sums} (the function $h_{s'}$ has already been computed during course of solving for the compressed transition probabilities).

See the Appendix for a proof, a matrix formulation of this result, and a discussion concerning  computational consequences. 

\subsubsection{Trajectory Lengths}
\label{sec:exp_path_lens}
Assume the hitting times $(T_m)_{m\geq 0}$ are as defined in Section~\ref{sec:coarse_rewards}. 
We note that the average path lengths (hitting times) between pairs of bottlenecks,
\[
L(s,a,s'):=\bbE_s[T_1 ~|~X_{T_1}=s'], \qquad s,s'\in\dcluster \text{ such that } (s,s')\in\supp_a(\widetilde{P})
\] can be computed using the machinery exactly as described in Section~\ref{sec:coarse_rewards} above by   setting 
\[
\Gamma(s,a,s')=1, \qquad R(s,a,s')=1,\qquad\text{for all }s,s'\in S, a\in A
\]
at the fine scale, and then applying the calculations for computing expected ``rewards'' given by Equations~\eqref{eqn:coarse_rews_linsys} and~\eqref{eqn:coarse_rews_sums} subject to a non-negativity constraint. Although expected path lengths are not essential for defining a compressed MDP, these quantities can still provide additional insight into a problem. For example, when simulations are involved, expected path lengths might hint at the amount of exploration necessary under a given policy to characterize a  \clustertext.

\subsubsection{Discount Factors}\label{sec:coarse_discounts}
When solving an MDP using the hierarchical decomposition introduced above, it is important to seek a good approximation for the discount factors at the coarse scale. In our experience, this results in improved consistency across scales, and improved accuracy of and convergence to the solution. 
In the preceding sections, a coarse MDP was computed by averaging over paths between bottlenecks at a finer scale. Depending on the particular source/destination pair of states, the paths will in general have different length distributions. Thus, when solving a coarse MDP, rewards collected upon transitioning between states at the coarse scale should be discounted at different, state-dependent rates. The correct discount rate is a random variable (as is the path length), and transitions at the coarse scale implicitly depend on outcomes at the fine scale. We will partially correct for differing length distributions (and avoid the need to simulate at the fine scale) by imposing a coarse non-uniform discount factor based on the cumulative fine scale discount applied on average to paths between bottlenecks. The coarse discount factors $\widetilde{\Gamma}$ are incorporated when solving the coarse MDP so that the scale of the coarse value function is more compatible with the fine problem, and convergence towards the fine-scale policy may be accelerated. 

The expected cumulative discounts may be computed using a procedure similar to the one given for computing expected rewards in Section~\ref{sec:coarse_rewards}.
As before, given a policy $\pi_\clusterindex$ on \clustertext\ $\clusterindex$, consider the Markov chain $(X_n)_{n\geq 0}$ with transition matrix $\Ppirestrtocluster$, and let $T, T'$ be two arbitrary stopping times satisfying $0\leq T<T'<\infty$ (a.s.).  The cumulative discount applied to trajectories $(X_{T},X_{T+1},\ldots,X_{T'})$ over the interval $T\leq t\leq T'$ is given by the random variable
\[
\Delta_{T}^{T'} := \prod_{t=T}^{T'-1}\Gamma\bigl(X_t,a_{t+1},X_{t+1}\bigr) \,,
\]
where $a_{t+1}\sim\pi_\clusterindex(X_t)$ for $t=T,\ldots,T'-1$.
The following proposition describes how to compute the expected discount factors by solving a collection of linear problems with non-negativity constraints.

\begin{proposition}\label{prop:coarse_discount_factors}
Suppose the coarse scale action $a\in\widetilde{A}$ corresponds to executing the policy $\pi_\clusterindex$ in \clustertext\ $\clusterindex$. Let $(X_t)_{t\geq 0}$ denote  the Markov chain with transition matrix $\Ppirestrtocluster$, and let $(T_m)_{m\geq 0}$ denote the boundary hitting times defined in Section~\ref{sec:coarse_rewards}. The discount factors $\widetilde{\Gamma}$ at the coarse scale may be characterized as
\[
\widetilde{\Gamma}(s,a,s') = \bbE_s[\Delta_{0}^{T_1} ~|~ X_{T_1}=s'], \qquad (s,s')\in\supp_a(\widetilde{P})
\]
and, letting $H_{s,s'}:=\widetilde{\Gamma}(s,a,s')$, may be computed by finding the minimal non-negative solution $H$ to the linear system
\begin{equation*}
H_{s,s'} = 
\begin{dcases}
\smashoperator[r]{\sum_{s''\in\Ui\cap \Uprime, a\in A}} P_{h_{s'}}(s,a,s'')\Gamma(s,a,s'')H_{s'',s'} + \sum_{a\in A}P_{h_{s'}}(s,a,s')\Gamma(s,a,s'),  & \text{if } s\in\Ui\cap \Uprime\\
\smashoperator[r]{\sum_{s''\in\Ui\cap \Uprime, a\in A}} P_{\tilde{h}_{s'}}(s,a,s'')\Gamma(s,a,s'')H_{s'',s'} + \sum_{a\in A}P_{\tilde{h}_{s'}}(s,a,s')\Gamma(s,a,s'), & \text{if }
 (s,s')\in\supp_a(\widetilde{P})
\end{dcases}
\end{equation*}
where $h_{s'}(s)$, $\Uprime$, $P_{h_{s'}}$, and $P_{\tilde{h}_{s'}}$ are as defined in Proposition~\ref{prop:coarse_rewards}.
\end{proposition}
The proof is again postponed until the Appendix.

It is worth observing that if the discount factor at the fine scale is {\em uniform}, $\Gamma(s,a,s')=\gamma$ with no dependence on states or actions, then the expected cumulative discounts may be related to the average path lengths $L(s,a,s')$ between pairs of bottlenecks  described in Section~\ref{sec:exp_path_lens}. In particular, suppose $T(s,a,s')$ is the first passage time of a fine scale trajectory starting at $s\in\dcluster $ and ending at $s'\in\dcluster $ following a policy determined by the coarse action $a\in\widetilde{A}$. Then $L(s,a,s')=\bbE[T(s,a,s')]$, and Jensen's inequality implies that 
\[
\gamma^{\bbE[T]}\leq \bbE[\gamma^T].
\]
Thus $\widetilde{\Gamma}(s,a,s') \geq \gamma^{\bbE[T(s,a,s')]} = \gamma^{L(s,a,s')}$, and this approximation improves as $\gamma\uparrow 1$. However, this is only true for uniform $\gamma$ at the fine scale, and even for $\gamma$ close to $1$, the relationship above may be loose. Although the connection between path lengths and discount factors is illuminating and potentially useful in the context of Monte-Carlo simulations, we suggest calculating coarse discount factors according to Proposition~\ref{prop:coarse_discount_factors} rather than through path length averages.

In this and previous subsections, the approach taken is in the spirit of revealing the structure of the coarsening step and how it is possible to compute many coarser variables, or approximations thereof, by solutions of linear systems. Of course one may always use Monte-Carlo methods, which in addition to estimates of the expected values, may be used to obtain more refined approximations to the law of the random variables  $\Delta_{T_0}^{T_1}$ and  $R_{T_0}^{T_1}$.

\begin{figure}[t]
\centering
\tikzset{
  nonterminal/.style={
    rectangle,
    minimum size=6mm,
    very thick,
    draw=red!50!black!50,         
    top color=white,              
    bottom color=red!50!black!20, 
    font=\itshape
  },
  terminal/.style={
    rounded corners, 
    minimum size=6mm,
    very thick,draw=black!50,
    top color=white,bottom color=black!20,
    font=\sffamily},
  skip loop/.style={to path={-- ++(0,#1) -| (\tikztotarget)}}
}
{
  \tikzset{terminal/.append style={text height=1.5ex,text depth=.25ex}}
  \tikzset{nonterminal/.append style={text height=1.5ex,text depth=.25ex}}
}
\begin{tikzpicture}[
        point/.style={coordinate},>=stealth',thick,draw=black!50,
        tip/.style={->,shorten >=0.007pt},every join/.style={rounded corners},
        hv path/.style={to path={-| (\tikztotarget)}},
        vh path/.style={to path={|- (\tikztotarget)}},
    ]
\sffamily
    \matrix[column sep=4mm, row sep=0.9cm] {
        \node (p0) [point] {}; & \node (p1) [point] {}; &
        \node (comp) [nonterminal] {compress}; &
        \node (p2) [point] {}; & \node (sc) [terminal]    {solve coarse};               &
        \node (sp1) [point] {}; & \node (p3) [point] {}; & \node (sp2) [point] {}; &
         \node (if) [terminal]    {update fine};         &
         \node (sp3) [point,xshift=-2mm] {}; &
        \node (p4) [point] {}; & \node (p5) [point] {};\\
        & & & & & & & \node (sp1_row2) [point,xshift=2mm] {}; & \node (ub)  [terminal]  {update boundary}; &  
        \node (sp_row2) [point,xshift=2mm] {}; & \node (p4_row2) [point] {}; \\
    };

    { [start chain]
    	\chainin (p4) [join];
    	{ [start branch=p4 loop]
        	\chainin (sp_row2) [join=by {vh path}];
        }
    	\chainin (sp_row2);
   		\chainin (ub) [join=by tip];
   		\chainin (sp1_row2) [join];
   		{ [start branch=sp1_row2 loop]
        \chainin (p3) [join=by {hv path,tip}];
        }
    	}
    { [start chain]
        \chainin (p0) [join];
        \chainin (p1) [join];
        \chainin (comp)   [join=by tip];
        \chainin (p2)    [join];
        \chainin (sc)   [join=by tip];
        \chainin (sp1)    [join];
        \chainin (p3)    [join];
        \chainin (sp2)    [join];
        \chainin (if) [join=by tip];
        \chainin (sp3)    [join];
        \chainin (p4)    [join];     
        { [start branch=p4 loop]
        \chainin (p3) [join=by {skip loop=-8mm,tip}];
        }
        { [start branch=p4 loop]
        \chainin (p1) [join=by {skip loop=8mm,tip,dashed}];
        }        
        \chainin (p5)    [join =by tip];
        }
  \draw (p0) node[left, inner sep=.5mm] {$\pi_0$};
  \draw (p3) node[above, inner sep=.5mm] {$V_{\mathrm{coarse}}$};
  \draw (sp3) node[above, inner sep=.5mm] {$\pi_{\mathrm{new}}$};
\end{tikzpicture}
\caption{\small\em Different solution algorithms for solving a pair of coarse/fine MDPs are obtained by
iterating over different paths in this flow graph. See text for details.}
\label{fig:soln-flow}
\end{figure}
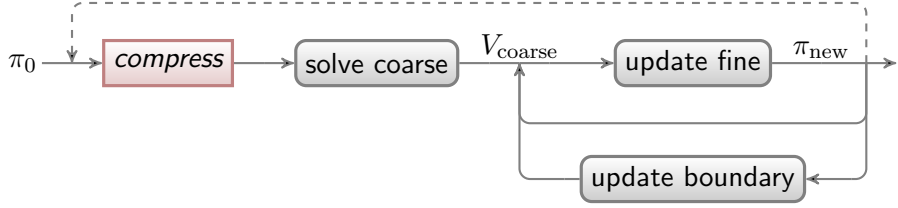

\subsection{{\em{Step 3:}} Multiscale Solution of MDPs}
\label{sec:mdp_solution}
Given a (fine) MDP and a coarsening as above, a solution to the fine scale MDP may be obtained by applying one of several possible  algorithms suggested by the flow diagram in Figure~\ref{fig:soln-flow}. Solving for the finer scale's policy involves alternating between two main computational steps: (1) updating the fine solution given the coarse solution, and (2) updating the coarse solution given the fine solution. Given a coarse solution defined on bottleneck states, the fine scale problem decomposes into a collection of smaller independent sub-problems, each of which may be solved approximately or exactly. These are iterations along the inner loop surrounding ``update fine'' in Figure~\ref{fig:soln-flow}. After the fine scale problem has been updated, the solution on the bottlenecks may be updated either with or without a re-compression step. The former is represented by the long upper feedback loop in Figure~\ref{fig:soln-flow}, while the latter corresponds to the outer, lower loop passing through ``update boundary''. Updating without re-compressing may, for instance, take the form of the updates (Bellman, averaging) appearing  in any of the asynchronous policy/value iteration algorithms. Updating by re-compression consists of re-compressing with respect to the current, updated fine policy and then solving the resulting coarse MDP. 

The discussion that follows considers an arbitrary pair of successive scales, a ``fine scale'' and a ``coarse scale'', and applies to any such pair within a general hierarchy of scales. A key property of the  compression step is that it yields new MDPs in the same form, and can therefore be iterated. Similarly, the process of coarsening and refining policies and value functions allow one to move from fine to coarse scales and then from coarse to fine, respectively, and therefore may be repeated through several scales, in different possible patterns.

\begin{algorithm}[t]
\caption{Top-down solution of MDPs: Alternating interior-boundary approach.}
\fbox{\begin{minipage}{0.975\textwidth}
Set the initial fine scale policy to random uniform if not otherwise given via transfer.
\begin{enumerate}\itemsep 0pt
\item Compress the MDP using one or more policies.
\item Solve the coarse MDP using any algorithm, and save the resulting value function $V_{\mathrm{coarse}}$.
\item Fix the value function $V_{\mathrm{fine}}$ of the fine MDP at bottleneck states $\B$ to $V_{\mathrm{coarse}}$.
\item\label{list:solve-clusters} Solve the local boundary value problems separately within each \clustertext to fill in the rest of $V_{\mathrm{fine}}$, given the current fine scale policy.
\item\label{list:interior-greedy-policy} Recover a fine scale policy $\pi:\Si\times A\to\R_{+}$ on \clustertext interiors ($\Si:=S\setminus\B$) from the resulting $V_{\mathrm{fine}}$. For $s\in\Si$,
\begin{subequations}\label{eqn:alg-greedy-update}
\begin{align}
a^*(s) &= \arg\max_{a\in A}\smashoperator[r]{\sum_{s'\in\clusterindex([s])}}P(s,a,s')\bigl(R(s,a,s') + \Gamma(s,a,s') V_{\mathrm{fine}}(s')\bigr)\label{eqn:alg-max-int} \\
\pi(s,\cdot) &= \delta_{a^*(s)} \;.
\end{align}
\end{subequations}
\item\label{list:blending} Blend in a regularized fashion with the previous global policy. For $s\in\Si$,
\begin{equation}\label{eqn:alg-pol-blending}
\pi_{\mathrm{new}}(s,\cdot) = \lambda\pi(s,\cdot) + (1-\lambda)\pi_{\mathrm{old}}(s,\cdot) \;.
\end{equation}
where $\lambda\in(0,1]$ is a regularization parameter.
\item\label{list:alg-local-PI} (Optional - Local Policy Iteration) Set $\pi_{\mathrm{old}} =\pi_{\mathrm{new}}$. Repeat from step~(\ref{list:solve-clusters}) until convergence criteria met.
\item\label{list:alg-bn-pol} Update the fine policy on bottleneck states by applying Equations~\eqref{eqn:alg-greedy-update}-\eqref{eqn:alg-pol-blending} for $s\in\B$.
\item\label{list:alg-bn-val} Update the boundary states' values either exactly, or by repeated local averaging,
\[
V_{\mathrm{fine}}(s) = 
\bbE_{a\sim \pi_{\mathrm{new}}(s)}\left[
\sum_{s'}P(s,a,s')\bigl(R(s,a,s') + \Gamma(s,a,s') V_{\mathrm{fine}}(s')\bigr) \right], \, s\in\B
\]
where the number of averaging passes $N$ for each bottleneck state $s\in\B$, satisfies $N > \log_{\bar{\gamma}}\tfrac{1}{2}$ with $\bar{\gamma}:=\max_{s,a,s'}\bigl\{\Gamma(s,a,s') \indic_{[P(s,a,s')>0]}\bigr\}$.
\item\label{list:alg-last} Set $\pi_{\mathrm{old}} =\pi_{\mathrm{new}}$. Repeat from step~(\ref{list:solve-clusters}) until convergence criteria met.
\end{enumerate}
\end{minipage}
}
\label{alg:hierarchy-solve}
\end{algorithm}

In a problem with many scales, the hierarchy may be traversed in different ways by recursive applications of the solution steps discussed above. A particularly simple approach is to solve {\em top-down} (coarse to fine) and update {\em bottom-up}. In this case the solution to the coarsest MDP is pushed down to the previous scale, where we may solve again and push this solution downwards and so on, until reaching a solution to the bottom, original problem. 
It is helpful, though not essential, when solving top-down if the magnitude of coarse scale value functions are directly compatible with the optimal value function for the fine-scale MDP. What is important, however, is that there is sufficient gradient as to direct the solution along the correct path to the goal(s), stepping from \clustertext-to-\clustertext at the finest scale. In Algorithm~\ref{alg:hierarchy-solve}, solving top-down will enforce the coarse scale value gradient. One can mitigate the possibility of errors at the coarse scale by compressing with respect to carefully chosen policies at the fine scale (see Section~\ref{sec:cluster_pols}). However, to allow for recovery of the optimal policy in the case of imperfect coarse scale information, a {\em bottom-up} pass updating coarse scale information is generally necessary. Coarse scale information may be updated either by re-compressing or by means of other local updates we will describe below. 

Although we will consider in detail the solution of a two layer hierarchy consisting of a fine scale problem and a single coarsened problem, these ideas may be readily extended to hierarchies of arbitrary depth: what is important is the handling of pairs of successive scales. The particular algorithm we describe chooses (localized) policy iteration for fine-scale policy improvement, and local averaging for updating values at bottleneck states. Algorithm~\ref{alg:hierarchy-solve} gives the basic steps comprising the solution process. The fine scale MDP is first compressed with respect to one or more policies. In the absence of any specific and reliable prior knowledge, a {\em collection} of \clustertext-specific stochastic policies, to be described in Section~\ref{sec:cluster_pols}, is suggested. This collection attempts to provide all of the {\em coarse} actions an agent could possibly want to take involving the given  \clustertext. These coarse actions involve traversing a particular \clustertext towards each bottleneck along paths which vary in their directedness, depending on the reward structure within the  \clustertext. The Algorithm is local to \clusterstext, however, so computing these policies is inexpensive. Moreover, if given policies defined on one or more \clusterstext a priori, then those policies may be added to the collection used to compress, providing additional actions from which an agent may choose at the coarse scale. Solving the coarse MDP amounts to choosing the best actions from the available pool.

The next step of Algorithm~\ref{alg:hierarchy-solve} is to solve the coarse MDP to convergence. Since the coarse MDP may itself be compressed and solved efficiently, this step is relatively inexpensive. The optimal value function for the coarse problem is then assigned to the set of bottleneck states for the fine problem\footnote{Recall that the statespace of the coarse problem is exactly the set of bottlenecks for the fine problem.}. With bottleneck values fixed, policy iteration is invoked within each \clustertext independently (Steps (\ref{list:solve-clusters})-(\ref{list:blending})). These local policy iterations may be applied to the \clusterstext in any order, and any number of times. 
The value determination step here can be thought of as a boundary value problem, in which a \clustertext's boundary (bottleneck) values are interpolated over the interior of the  \clustertext. Section~\ref{sec:solve-local} explains how to solve these problems as required by Step~(\ref{list:solve-clusters}) of  Algorithm~\ref{alg:hierarchy-solve}. Note that only values on the interior of a \clustertext are needed so the  policy does not need to be specified on bottlenecks for local policy iteration. 

Next, a greedy fine scale policy on a \clustertext's interior states is computed from the interior values (Step (\ref{list:interior-greedy-policy})). The new interior policy is a blend between the greedy policy and the previous policy (Step (\ref{list:blending})). Starting from an initial stochastic fine scale policy,  policy blending allows one to regularize the solution and maintain a degree of stochasticity that can repair coarse scale errors. 

Finally, information is exchanged between \clusterstext by updating the policy on bottleneck states  (Step~(\ref{list:alg-bn-pol})), and then using this (globally defined) policy in combination with the interior values to update bottleneck values by local averaging (Step~(\ref{list:alg-bn-val})). Both of these steps are computationally inexpensive. Alternating updates to \clustertext interiors and boundaries are executed until convergence. This algorithm is guaranteed to converge to an optimal value function/policy pair (it is a variant of modified asynchronous policy iteration, see~\citep{BertsekasVol2}), however in general 
convergence may not be monotonic (in any norm). Section~\ref{sec:ms-alg-proof} gives a proof of convergence for arbitrary initial conditions.

We note that often approximate solutions to the top level or \clustertext-specific problems is sufficient. Empirically we have found that single policy iterations applied to the \clusterstext in between bottleneck updates gives rapid convergence (see Section~\ref{sec:examples}). We emphasize that at each level of the hierarchy below the topmost level, the MDP may be decomposed into distinct pieces to be solved locally and  independently of each other. Obtaining solutions at each scale is an efficient process and at no point do we solve a large, global problem. 

In practice, the multi-scale algorithm we have discussed requires fewer iterations to converge than global, single-scale algorithms, for primarily two reasons. First, the multiscale algorithm starts with a coarse approximation of the fine solution given by the solution to the compressed MDP. This provides a good global warm start that would otherwise require many iterations of a global, single-scale algorithm. Second, the multiscale treatment can offer faster convergence since sub-problems are decoupled from each other given the bottlenecks. Convergence of local (within \clustertext) policy iteration is thus constrained by what are likely to be faster mixing times {\em within} \clusterstext, rather than slow global times across \clusterstext, as conductances are comparatively large within \clusterstext by construction.


\begin{algorithm}[t]
\caption{Determining good policies to be used for compression.}
\fbox{\begin{minipage}{0.975\textwidth}
\begin{algorithmic}[1]
\For{each \clustertext\ $\clusterindex$}
 \State Set $S_\clusterindex$ to be \clustertext\ $\clusterindex$ 
 \For{each bottleneck $\bindex\in\boundary\clusterindex$}
 \State Set $P_{\clusterindex,\bindex}$ to be $P_\clusterindex$ modified so that $\bindex$ is absorbing
 \State Set $R_{\clusterindex,\bindex,r}=R_{\clusterindex}$
  \For{each $r\in R\mathrm{int}_\clusterindex$}
   \State $R_{\clusterindex,\bindex,r}(s,a,\bindex)\gets R_{\clusterindex}(s,a,\bindex) + \Gamma_{\clusterindex}(s,a,\bindex)r$ for all $s\in\clusterindex, a\in A$
   \State Solve $\text{MDP}_{\clusterindex,\bindex,r}:=(S_{\clusterindex},A, P_{\clusterindex,\bindex},R_{\clusterindex,\bindex,r},\Gamma_{\clusterindex})$ for a policy $\pi_{\clusterindex,\bindex,r}$ on \clustertext\ $\clusterindex$
  \EndFor
 \EndFor
\EndFor
\end{algorithmic}
\end{minipage}
}
\label{alg:room_comp_pols}
\end{algorithm}

\subsubsection{Selecting Policies for Compression}
\label{sec:cluster_pols}
In the context of solving an MDP hierarchy, the ideal coarse value function is one which takes on the exact  values of the optimal fine value function at the bottlenecks. Such a value function clearly respects the fine scale gradient, falls within a compatible dynamic range, and can be expected to lead to rapid convergence when used in conjunction with Algorithm~\ref{alg:hierarchy-solve}. Indeed, the best possible coarse value function that can be realized is precisely the solution to a coarse MDP compressed with respect to the optimal fine scale policy. We propose a {\em local} method for selecting a collection of policies at the fine scale that can be used for compression, such that the solution to the resulting coarse MDP is likely to be close to the best possible coarse solution.

Algorithm~\ref{alg:room_comp_pols} summarizes the proposed policy selection method, and is local in that only a single \clustertext at a time is considered. The idea behind this algorithm is that  useful coarse actions involving a given \clustertext generally consist of getting from one of the \clustertext's bottlenecks to another. The best fine scale path from one bottleneck to another, however, depends on the reward structure within the  \clustertext. In fact, if there are larger rewards within the \clustertext than outside, it may not even be advantageous to leave it. On the other hand, if only local rewards within a \clustertext are visible, then we cannot tell whether locally large rewards are also globally large. Thus, a collection of policies covering all the interesting alternatives is desired.

For \clustertext\ $\clusterindex$, let $\bar{r}:=\max_{s\in\clusterindex,a,s'\in\clusterindex}R_\clusterindex(s,a,s')$, $\underline{r}:=\min_{s\in\clusterindex,a,s'\in\clusterindex}R_\clusterindex(s,a,s')$, and $\bar{\gamma}=\max_{s\in\clusterindex,a,s'\in\clusterindex}\Gamma_\clusterindex(s,a,s')$. Let $\text{diam}(\clusterindex)$ denote the longest graph geodesic between any two states in \clustertext\ $\clusterindex$.
Then for each bottleneck $\bindex\in\B\cap\clusterindex$, and any choice of $r\in R\mathrm{int}_\clusterindex$, where
$$R\mathrm{int}_\clusterindex=\frac{1-\bar{\gamma}^{\text{diam}(\clusterindex)}}{1-\bar{\gamma}}[\min\{0, \underline{r}\}, \max\{0,\bar{r}\}],$$
we consider the following $\text{MDP}_{\clusterindex,\bindex,r}:=(S_{\clusterindex},A, P_{\clusterindex,\bindex},R_{\clusterindex,\bindex,r},\Gamma_{\clusterindex})$:
\begin{itemize}\itemsep 0pt
\item The statespace $S_{\clusterindex}$ is \clusterindex;
\item The transition probability law $P_{\clusterindex,\bindex}$ is the transition law of the original MDP restricted to \clusterindex, but modified so that $\bindex$ is an absorbing state\footnote{If the \clustertext already contains absorbing (terminal) states, then those states remain absorbing (in addition to $\bindex$).} for $P_{\clusterindex,\bindex}^\pi$, regardless of the policy $\pi$;
\item The rewards $R_{\clusterindex,\bindex,r}$, for fixed $\bindex$, are the rewards $R$ of the original MDP truncated to \clusterindex, with the modification $R_{\clusterindex,\bindex,r}(s,a,\bindex)=R(s,a,\bindex)+\Gamma(s,a,\bindex)r$, for all $s\in\clusterindex$ and $a\in A$;
\item The discount factors $\Gamma_{\clusterindex}$ are the discount factors of the original MDP truncated to \clusterindex.
\end{itemize}
The optimal (or approximate) policy $\pi^*_{\clusterindex,\bindex,r}$ of each $\text{MDP}_{\clusterindex,\bindex,r}$ is computed. As $r$ ranges in a continuous interval, we expect to find only a small number\footnote{which may be found by bisection search of $R_\clusterindex$} of {\em{distinct}} optimal policies $\{\pi^*_{\clusterindex,\bindex,r}\}_{r\in R\mathrm{disc}_{\clusterindex,\bindex}}$, for each fixed $\bindex$, where $R\text{disc}_{\clusterindex,\bindex}$ is the set of corresponding rewards placed at $\bindex$ giving rise to the distinct policies. 
 Therefore the cardinality of this set of policies is $\cO(\sum_{\bindex\in\boundary\clusterindex}|R\mathrm{disc}_{\clusterindex,\bindex}|)$. The set of policies $\picluster_\clusterindex:=\cup_{\bindex\in\boundary\clusterindex}\{\pi^*_{\clusterindex,\bindex,r}\}_{r\in R\mathrm{disc}_{\clusterindex,\bindex}}$ is our candidate for the set of actions available at the coarser scale when the agent is at a bottleneck adjacent to \clustertext\ $\clusterindex$, and for the set of actions which was assumed and used in Sections \ref{sec:clustering} and \ref{sec:compression}.

Finally, we note that Algorithm~\ref{alg:room_comp_pols} is trivially parallel, across \clusterstext and across bottlenecks within \clusterstext. In addition, solving for each policy is comparatively inexpensive because it involves a single  \clustertext.

\subsubsection{Solving Localized Sub-Problems}
\label{sec:solve-local}
Given a solution (possibly approximate) to a coarse MDP in the form of a value function $V_{\mathrm{coarse}}$, one can efficiently compute a solution to the fine-scale MDP by fixing the values at the fine scale's bottlenecks to those given by the coarse MDP value function. The problem is partitioned into \clusterstext where we can solve locally for a value function or policy within each \clustertext independently of the other \clusterstext, using a variety of MDP solvers. Values inside a \clustertext are conditionally independent of values outside the \clustertext given the \clustertext's bottleneck values.

As an illustrative example we show how policy iteration may be applied to learn the policies for each  \clustertext. 
Let $\pi_\clusterindex$ be an initial policy at the fine scale defined on at least $\Ui$. Determination of the values on $\Ui$ given values on $\dcluster $ amounts to solving a {\em boundary value problem}: 
a continuous-domain physical analogy to this situation is that of solving a Poisson boundary value problem with Neumann boundary conditions. The connection with Poisson problems is that if $P_\clusterindex^{\pi_\clusterindex}$ is the transition matrix of the Markov chain $(X_n)_{n\geq 0}$ following $\pi_\clusterindex$ in \clustertext\ $\clusterindex$, then we would like to compute the function 
\[
V(s):=\bbE\bigl[R_0^{T_0} +  \Delta_0^{T_0} V_{\mathrm{coarse}}(X_{T_0}) ~|~ X_0=s\bigr], \qquad s\in\Ui,
\]
where $T_0:=\inf\{n\geq 0 ~|~ X_n\in\dcluster \}$ is the first passage time of the boundary of \clustertext \clusterindex, and
$R_0^{T_0}, \Delta_0^{T_0}$ are respectively defined in Sections~\ref{sec:coarse_rewards} and~\ref{sec:coarse_discounts}. It can be shown that $V(s)$ is unique and bounded 
under our usual assumption that the boundary $\boundary\clusterindex$ be $\pi_\clusterindex$-reachable from any interior state $s\in\clusterindex$. The value function we seek is computed by solving a linear system similar to Equation~\eqref{eqn:vpi-linsys}. We have,
\begin{equation*}
V(s) = 
\begin{cases}
V_{\mathrm{coarse}}(s) & \text{if } s\in\dcluster \\
\sum_{s'\in \clusterindex, a'\in A}P_\clusterindex(s,a,s')\pi_{\clusterindex}(s,a)\bigl[R(s,a,s') + \Gamma(s,a,s')V(s')\bigr] & 
  \text{if } s\in\Ui
\end{cases}
\end{equation*}
where $P_\clusterindex$ is the restriction of $P$ to $\clusterindex$ defined by Equation~\eqref{eqn:P_cluster}. 
It is instructive to consider a matrix-vector formulation of this system. Let $R_\clusterindex, \Gamma_\clusterindex$ denote the respective truncation of $R(s,a,s'), \Gamma(s,a,s')$ to the triples $\{(s,a,s') ~|~ s\in\Ui, s'\in \clusterindex, a\in A\}$. Defining the quantities $(P_\clusterindex\circ R_\clusterindex)^{\pi_\clusterindex}, (\Gamma_\clusterindex\circ P_\clusterindex)^{\pi_\clusterindex}$ using~\eqref{eqn:hadamard-pi}, assume the partitioning 
\[
(\Gamma_\clusterindex\circ P_\clusterindex)^{\pi_\clusterindex} =
\begin{pmatrix}
B & C\\
D & Q
\end{pmatrix},
\quad
(P_\clusterindex\circ R_\clusterindex)^{\pi_\clusterindex} =
\begin{pmatrix}
B_r & C_r\\
D_r & Q_r
\end{pmatrix},
\quad
V = 
\begin{pmatrix}
V_b\\ V_q
\end{pmatrix}
\]
where interactions among bottlenecks attached to \clustertext\ $\clusterindex$ are captured by $B$ labeled blocks and interactions among non-bottleneck interior states by $Q$ labeled blocks. Fix $V_b=V_{\mathrm{coarse}}$, and let $V_q$ denote the (unknown) values of the \clustertext's interior states. The  value function on interior states $V_q$ is given by
\[
V_q = \bigl[D_r ~~ Q_r]\bbone + [D ~~ Q]
\begin{pmatrix}
V_b \\
V_q
\end{pmatrix}
\]
so that we must solve the $|\Ui|\times|\Ui|$ linear system
\begin{equation}\label{eqn:local-system}
(I - Q)V_q = [D_r ~~ Q_r]\bbone + DV_b\,.
\end{equation}

Given a value function for a \clustertext, the policy update step is unchanged from vanilla policy iteration except that we do not solve for the policy at bottleneck states: only the policy restricted to interior states is needed to update the $Q$ and $D$ blocks of
the matrices above, towards carrying out another iteration of value determination inside the \clustertext (the blocks $C,D$ are not needed). This shows in yet another way that policy iteration within a given \clustertext is completely independent of other \clusterstext. 
When policy iteration has converged to the desired tolerance within each \clustertext independently, or if the desired number of iterations has been reached, the individual \clustertext-specific value functions may be simply concatenated together along with the given values at the bottlenecks to obtain a globally defined value function.

As mentioned above, solving for a value function on a \clustertext's interior does not require the initial policy to be defined on bottlenecks\footnote{We will discuss why this situation could arise in the context of transfer learning (Section~\ref{sec:transfer}).}, however a policy on bottleneck states can be quickly determined from the global value function. This step is computationally inexpensive when bottlenecks are few in number and have comparatively low out-degree. A policy defined on \clustertext interiors is obtained either from the global value function, or automatically during the solution process if, for example, a policy-iteration variant is used.


\subsubsection{Bottleneck Updates}
Given any value function $V$ on \clustertext interiors and any globally defined policy $\pi$, values at bottleneck states may be updated using similar asynchronous iterative algorithms: we hold the value function fixed on all \clustertext interiors, and update the bottlenecks. Combined with interior updating, this step comprises the second half of the alternating solution approach outlined in Algorithm~\ref{alg:hierarchy-solve}.

Local averaging of the values in the vicinity of a bottleneck is a particularly simple update,
\[
V(s) \gets 
\sum_{s',a}P(s,a,s')\pi(s,a)\bigl(R(s,a,s') + \Gamma(s,a,s') V(s')\bigr), \quad s\in\B.
\]
Value iteration and modified value iteration variants may be defined analogously. Value determination at the bottlenecks may be characterized as follows. Consider the (global) quantities $(P\circ R)^{\pi}, (\Gamma\circ P)^{\pi}$ (these do not need to be computed in their entirety), and the partitioning 
\[
(\Gamma\circ P)^{\pi} =
\begin{pmatrix}
B & C\\
D & Q
\end{pmatrix},
\quad
(P\circ R)^{\pi} =
\begin{pmatrix}
B_r & C_r\\
D_r & Q_r
\end{pmatrix},
\quad
V = 
\begin{pmatrix}
V_b\\ V_q
\end{pmatrix}
\]
where, as before, interactions among bottlenecks are captured by $B$ labeled blocks and interactions among non-bottleneck (interior) states by $Q$ labeled blocks. Let $V_q$ be held fixed to the known interior values, and let $V_b$ denote the unknown values to be assigned to bottlenecks. Then $V_b$ is obtained by solving the $|\B|\times|\B|$ linear system
\begin{equation}\label{eqn:local-bn-system}
(I - B)V_b = [B_r ~~ C_r]\bbone + CV_q\,.
\end{equation}
In the ideal situation, by virtue of the spectral clustering Algorithm~\ref{alg:spectral_clustering}, the blocks $C$ and $C_r$ are likely to be sparse (bottlenecks should have low out-degree) so the matrix-vector products $CV_q$ and $C_r\bbone$ are inexpensive. Furthermore, even though $|\B|$ is already small, by similar arguments bottlenecks should not ordinarily have many direct connections to other bottlenecks, and $B$ is likely to be block diagonal.
Thus, solving~\eqref{eqn:local-bn-system} is likely to be essentially negligible.

\subsubsection{Proof of Convergence}
\label{sec:ms-alg-proof}
Algorithm~\ref{alg:hierarchy-solve} is an instance of modified asynchronous policy iteration  (see~\citep{BertsekasVol2}  for an overview), and can be shown to recover an optimal fine scale policy from any  initial starting point.
\begin{theorem}
Fix any initial fine-scale policy $\pi_0$, and any collection of compression policies 
$\{\picluster_{\clusterindex}\}_{\clusterindex\in{\sf C}}$ such that for each ${\clusterindex\in{\sf C}}$, $\dcluster$ is $\pi_{\clusterindex}$-reachable for all $\pi_{\clusterindex}\in\picluster_{\clusterindex}$. Let $V^k$ denote the global fine scale value function after $k>0$ passes of Steps~(\ref{list:solve-clusters})-(\ref{list:alg-last}) in  Algorithm~\ref{alg:hierarchy-solve}.
For an appropriate number of updates $N$ per bottleneck per algorithm iteration satisfying
\begin{equation}\label{eqn:num-bn-updates}
N > \log_{\bar{\gamma}}\tfrac{1}{2}
\end{equation}
with $\bar{\gamma}:=\max_{s,a,s'}\bigl\{\Gamma(s,a,s') \indic_{[P(s,a,s')>0]}\bigr\}$,
the sequence $V^k$ generated by the alternating interior-boundary policy iteration Algorithm~\ref{alg:hierarchy-solve} satisfies
\[
\adjustlimits\lim_{k\to\infty} \max_{s\in S}|V^*(s)-V^k(s)| = 0
\]
where $V^*$ is the unique optimal value function.
\end{theorem}
\begin{proof}
We first note that the value function updates in Algorithm~\ref{alg:hierarchy-solve}, on both interior and boundary states at the fine scale may be written as one or more applications of (locally defined) averaging operators $T_{\pi}$ of the form
\begin{equation}\label{eqn:Tpi_op}
(T_{\pi}V)(s) = \sum_{s',a}\pi(s,a)P(s,a,s')\bigl(R(s,a,s') + \Gamma(s,a,s')V(s')\bigr) \,.
\end{equation}
Value determination is equivalent to an ``infinite'' number of applications. The main challenge is that we require convergence to optimality from any initial condition $(V^0,\pi_0)$. Under the current assumptions on the policy, modified asynchronous policy iteration is known to converge (monotonically) in the $L^{\infty}$ norm to a unique optimal $V^*$, with corresponding optimal $\pi^*$, provided the initial pair $(V^0,\pi_0)$ satisfies $T_{\pi}V^0 \geq V^0$~\citep{BertsekasVol2}, where $T_{\pi}$ is the DP mapping defined by~\eqref{eqn:Tpi_op}. This initial condition is {\em not} in general satisfied here, since $(V^0,\pi_0)$ may be set on the basis of transferred information and/or coarse scale solutions. In Algorithm~\ref{alg:hierarchy-solve} for instance, the initial value function $V^0$ is the initial coarse MDP solution on the bottlenecks, and zero everywhere else. Furthermore, a common fix that shifts $V^0$ by a large negative constant (depending on $(1-\gamma)^{-1})$ does not apply because it could destroy consistency across sub-problems, and moreover can make convergence extraordinarily slow.  

Alternatively, Williams \& Baird show that modified asynchronous policy iteration will converge to optimality from any initial condition, provided enough value updates $T_{\pi}$ are applied~\citep{WilliamsBaird:90,WilliamsBaird:93}. The condition~\eqref{eqn:num-bn-updates} adapts the precise requirement found in~\citep[Thm. 8]{WilliamsBaird:90} to the present setting, where discount factors are state and action dependent. The proof follows~\citep[Thm. 4.2.10]{WilliamsBaird:93} closely, so we only highlight points where differences arise due to this state/action dependence, and due to the use of multi-step operators, $T_{\pi}^n$. We direct the reader to the references above for further details.

First note that if $(s_n)_n$ is the Markov chain with transition law $P^{\pi}$, 
\[
(T_{\pi}^nV)(s) = \bbE\bigl[R_0^{n} + \Delta_0^{n}V(s_n) ~|~ s_0=s\bigr]
\]
where $R_0^{n}, \Delta_0^{n}$ are as defined in Sections~\ref{sec:coarse_rewards} and~\ref{sec:coarse_discounts} above. One can see this by defining a recursion $V_n=T_{\pi}V_{n-1}$ with $V_0=V$, and repeated substitution  of~\eqref{eqn:Tpi_op}. Fix $\varepsilon>0$ and choose $N$ large enough so that 
\[
\liminf_{i\to\infty} V^i(s) - \varepsilon < V^k(s) < \limsup_{i\to\infty}V^i(s) + \varepsilon
\]
for all $k\geq N$ and all $s$. Let $L^*:=\max_{s}\big\{\limsup_{i\to\infty}V^i(s)- \liminf_{i\to\infty} V^i(s)\bigr\}$, and let $s^*$ be any state at which the maximum $L^*$ is achieved. It is enough to show that $L^*=0$ (convergence of the sequence $V^k$) to ensure convergence to optimality~\citep[Thm. 4.2.1]{WilliamsBaird:93}, however we note that this convergence need not be monotonic in any norm.
The action of $T_{\pi}^n$ after $N$ iterations can be bounded as follows:
\begin{align*}
\delta &= \sup_{k\geq N} (T^n_{\pi}V^k)(s^*) - \inf_{k\geq N} (T^n_{\pi}V^k)(s^*)\\
 &<\bbE_{s^*}\Bigl[\Delta_0^{n}(\limsup_{i\to\infty}V^i(s_n) + \varepsilon)\Bigr] - 
   \bbE_{s^*}\Bigl[\Delta_0^{n}(\liminf_{i\to\infty}V^i(s_n) - \varepsilon)\Bigr]\\
 &\leq 2\varepsilon\bbE_{s^*}[\Delta_0^{n}] + L^*\bbE_{s^*}[\Delta_0^{n}]\\
 &\leq 2\varepsilon\bar{\gamma}^n + \bar{\gamma}^n L^* .
\end{align*}
Following the reasoning in~\citep[Thm. 4.2.10, pg. 27]{WilliamsBaird:93}, subsequent policy improvement at state $x^*$ can at most double the length of the interval $\delta$. Hence, $L^*\leq 2\delta < 2\varepsilon\bar{\gamma}^n + \bar{\gamma}^n L^*$, so that
\[
L^* < \frac{4\varepsilon\bar{\gamma}^n}{1-2\bar{\gamma}^n }
\]
for any $\varepsilon > 0 $, giving that $L^*=0$ as long as
\[
\bar{\gamma}^n < 1/2.
\]
For the interior states, the condition $\bar{\gamma}^n < 1/2$ is clearly satisfied, since we perform value determination at those states in Algorithm~\ref{alg:hierarchy-solve}.
\end{proof}


\subsection{Complexity Analysis}
\label{sec:ms-alg-complexity}
We discuss the running time complexity of each of the three steps discussed in the introduction of this section: partitioning, compression, and solving an MMDP. For the latter two steps, the computational burden is always limited to at most a single cluster at a time. In all cases, we consider worst case analyses in that we do not assume sparsity, preconditioning or parallelization, although there are ample, low-hanging opportunities to leverage all three. \\

\noindent\textbf{Partitioning:}
The complexity of the statespace partitioning and bottleneck identification step depends in general on the algorithm used. The local clustering algorithm of Spielman and Teng~\citep{Spielman:LocalClustering} or Peres and Andersen~\citep{Andersen:2009:ESP} finds an approximate cut in time ``nearly'' linear in the number of edges. The complexity of Algorithm~\ref{alg:spectral_clustering} above is dominated by the cost of finding the stationary distribution of $P_{\text{tel}}$, and of finding a small number of eigenvectors of the directed graph Laplacian $\cL$. The first iteration is the most expensive, since the computations involve the full statespace. However, the invariant distribution and eigenvectors may be computed inexpensively. $P$ is typically sparse, so $P_{\text{tel}}$ is the sum of a sparse matrix and a rank-1 matrix, and obtaining an exact solution can be expected to cost far less than that of solving a dense linear system. Approximate algorithms are often used in practice, however. For example, the algorithm of~\citep{ChungPR:LNCS:10} computes a stationary distribution within an error bound of $\epsilon$ in time $\cO\bigl(|E|\log(1/\epsilon)\bigr)$ if there are $|E|$ edges in the statespace graph given by $P$.  The eigenvectors of $\cL$ may also be computed efficiently using randomized approximate algorithms. The approach described in~\citep{Halko:2011} computes $k$ eigenvectors (to high precision with high probability) in $\cO\bigl(|S|^2\log k\bigr)$ time, assuming no optimizations for sparse input. Finding eigenvectors for the subsequent sub-graphs may be accelerated substantially by preconditioning using the eigenvectors found at the previous clustering iteration.\\

\noindent\textbf{Compression:} As discussed above, compression of an MDP involves computations which only consider one \clustertext at a time. This makes the compression step local, and restricts time and space requirements.  But assessing the complexity is complicated by the fact that non-negative solutions are needed when finding coarse transition probabilities and discounts. Various iterative algorithms for solving non-negative least-squares (NNLS) problems exist~\citep{BjorckLSBook,ChenPlemmons:10}, however guarantees cannot generally be given as to how many iterations are required to reach convergence. The recent quasi-Newton algorithm of Kim et al.~\citep{KimNNLS:10} appears to be competitive with or outperform many previously proposed methods, including the classic Lawson-Hanson method~\citep{LawsonHansonBook:74} embedded in MATLAB's \texttt{lsqnonneg} routine. We point out, however, that it is often the case in practice that the unique solution to the {\em unconstrained} linear systems appearing in Propositions~\ref{prop:coarsePsas} and~\ref{prop:coarse_discount_factors} are indeed also minimal, non-negative solutions. In this case, the complexity is
$\cO\bigl(|\dcluster||\icluster|^3 + |\icluster||\dcluster|^2\bigr)$ per \clustertext, for finding the coarse transition probabilities and coarse discounts corresponding to a single fine scale policy.

Solving for the coarse rewards always involves solving a linear system (without constraints), since the rewards are not necessarily constrained to be non-negative. This step also involves 
 $\cO\bigl(|\dcluster||\icluster|^3 + |\icluster||\dcluster|^2\bigr)$ time per \clustertext, per fine policy. 
 
We note briefly, that these complexities follow from naive implementations; many improvements are possible. First of all in many cases the graphs involved are sparse, and iterative methods for the solution of linear systems, for example, would take advantage of sparsity and dramatically reduce the computational costs. For example, solving for the transition probabilities involves solving for multiple ($|\dcluster|$) right-hand sides, and the left-hand side of the linear systems determining compressed rewards and discounts are the same. The complexities above also do not reflect savings due to sparsity. In addition, the calculation of the compressed quantities above are embarrassingly parallel both at the level of \clusterstext as well as bottlenecks attached to each \clustertext (elements of $\dcluster$). The case of compression with respect to multiple fine policies is also trivially parallelized.  \\

\noindent\textbf{MMDP Solution:}
As above, the complexity of solving an MMDP will depend on the algorithm selected to solve coarse MDPs and local sub-problems. Here we will consider  solving with the exact, (dynamic programming based) policy iteration algorithm, Algorithm~\ref{alg:hierarchy-solve}. In the worse case, policy iteration can take $|S|$ iterations to solve an MDP with statespace $S$, but this is pathological. We will assume that the number of iterations required to solve a given MDP is much less than the size of the problem. This is not entirely unreasonable, since we can assume policies are regularized with a small amount of the diffusion policy, and moreover, if there is significant complexity at the coarse scale, then further compression steps should be applied until arriving at a simple coarse problem where it is unlikely that the worse-case number of iterations would be required. Similarly, solving for the optimal policy within \clusterstext should take few iterations compared to the size of the \clustertext because, by construction of the bottleneck detection algorithm, the Markov chain is likely to be fast mixing within a \clustertext.

Assume the MDP has already been compressed, and consider a fine/coarse pair of successive scales. Given a coarse scale solution, the cost of solving the fine local boundary value problems  (Step~\ref{list:solve-clusters}) is $\sum_{\clusterindex\in\clusters} \cO\bigl(|\icluster|^3\bigr)$ (ignoring sparsity). Updating the policy everywhere on $\Si$ (Step~\ref{list:interior-greedy-policy}) involves solving $|\Si|$ maximization problems, but these problems are also local because the \clustertext interiors partition $\Si$ by construction. The cost of updating the policy on $\Si$ is therefore the sum of the costs of locally updating the policy within each \clustertext's interior. The cost for each \clustertext $\clusterindex\in\clusters$ is $\cO\bigl(|A||\icluster||\clusterindex| +|A||\icluster|\bigr)$ time to compute the right-hand side of Equation~\eqref{eqn:alg-max-int} and search for the maximizing action. The cost of updating the policy and value function at bottleneck states is assumed to be negligible, since ordinarily $|\B|\ll|S|$. The cost of each iteration of Algorithm~\ref{alg:hierarchy-solve} is therefore dominated by the cost of solving the collection of boundary value problems.

The cost of solving an MMDP with more than two scales depends on just how ``multiscale'' the problem is. The number of possible scales, the size and number of clusters at each scale, and the number of bottlenecks at each scale, collectively determine the computational complexity. These are all strongly problem-dependent quantities so to gain some understanding as to how these factors impact cost, we proceed by considering a reasonable scenario in which the problem exhibits good multiscale structure. 
For ease of the notation, let $n$ be the size of the original statespace. If at a scale $j$ (with $j=0$  the finest scale) there are $n_j$ states and $r_j$ \clusterstext of roughly equal size, an iteration of Algorthm~\ref{alg:hierarchy-solve} at that scale has cost $\cO\bigl(r_j(n_j/r_j)^3\bigr)$. If the sizes of the \clusterstext are roughly constant across scales, then we can say that $r_j=n_j/C$ for all $j$ and some size $C$.  If, in addition, the number of bottlenecks at each scale is about the number of clusters, then $n_j=n/C^j$, and the computation time across $\log n$ scales is $\cO\bigl(n\log n\bigr)$ per iteration. By contrast, global DP methods typically require $O(n^3)$ time per iteration. The assumption that there are $\log n$ scales corresponds to the assumption that we compress to the maximum number of possible levels, and each level has multiscale structure. If we adopt the assumption above that the number of iterations required to reach convergence at each scale is small relative to $n_j$, then the cost of solving the problem is $\cO(n\log n)$.

\section{Transfer Learning}
\label{sec:transfer}
Transfer learning possibilities within reinforcement learning domains have only relatively recently begun to receive significant attention, and remains a long-standing challenge with the potential for substantial impact in learning more broadly. We define transfer here as the process of transferring some aspect of a solution to one problem into another problem, such that the second problem may be solved faster or better (a better policy) than would otherwise be the case. Faster may refer to either less exploration (samples) or fewer computations, or both.

Depending on the degree and type of relatedness among a pair of problems, transfer may entail small or large improvements, and may take on several different forms. It is therefore important to be able to {\em systematically}:
\begin{enumerate}\itemsep 0pt
\item Identify transfer opportunities;
\item Encode/represent the transferrable information;
\item Incorporate transferred knowledge into new problems.
\end{enumerate}

We will argue that {\em a novel form of systematic knowledge transfer between sufficiently related MDPs is made possible by the multiscale framework discussed above}. In particular, if a learning problem can be decomposed into a hierarchy of distinct parts then there is hope that both a ``meta policy'' governing transitions between the parts, as well as the parts themselves, may be transferred when appropriate. In the former setting, one can imagine transferring among problems in which a sequence of tasks must be performed, but the particular tasks or their order may differ from problem to problem. The transfer of distinct sub-problems might for instance involve a database of pre-solved tasks. A new problem is solved by decomposing it into parts, identifying which parts are already in the database, and then stitching the pre-solved components together into a global policy. Any remaining unsolved parts may be solved for independently, and learning a  meta policy on sub-tasks is comparatively inexpensive.

A key conceptual distinction is {\em the transfer of policies rather than value functions}. Value functions reflect, in a sense, the global reward structure and transition dynamics specific to a given problem. These quantities may not easily translate or bear comparison from one task to another, while a policy may still be very much applicable and is easier to consider locally. Once a policy is transferred (Section~\ref{sec:policy-xfer}) we may, however, assess goodness of the transfer (Section~\ref{sec:xfer-detect}) by way of value functions computed with respect to the destination problem's transition probabilities and rewards. As transfer can occur at coarse scales or within a single partition element at the finest scale, conversion between policies and value functions is inexpensive. If there are multiple policies in a database we would like to test in a \clustertext, it is possible to quickly compute value functions and decide which of the candidate policies should be transferred.

If the transition dynamics governing a pair of (sub)tasks are similar (in a sense to be made precise later), then one can also consider {\em transferring potential operators} (defined in Section~\ref{sec:mdp-defns}). In this case the potential operator from a source problem is applied to the reward function of a destination problem, but along a suitable pre-determined mapping between the respective statespaces and action sets. The potential operator approach also provides two advantages over value functions: reward structure independence and localization. The reward structure of the destination problem need not match that of the source problem. And a potential operator may be localized to a specific sub-problem at any scale, where locally the transition dynamics of the two problems are comparable, even if globally they aren't compatible or a comparison would be difficult. 

Both policy transfer and potential operator transfer provide a systematic means for identifying and transferring information where possible. At a high-level, the transfer framework we propose consists of the general steps given in Algorithm~\ref{alg:transfer-general}. Transfer between two hierarchies proceeds by matching sub-problems at various scales, testing whether transfer can actually be expected to help, transferring policies and/or potential operators where appropriate, and finally solving the  unsolved problem using the transferred information. Each of these steps is discussed in detail in the sections below. 

%
%

\begin{algorithm}[t]
\caption{High-level transfer learning steps.}
\fbox{\begin{minipage}{0.975\textwidth}
 Given a pair of problems $(\problemone,\problemtwo)$ and a solved multiscale MDP hierarchy for the first problem:
\begin{enumerate}\itemsep 0pt
\item Compute a hierarchy of MDPs for the second, new problem as described in 
Section~\ref{sec:compression}.
\item For each scale $j\geq 0$, select the sub-problems (\clusterstext) $\bigl\{r_k^{(j)}\bigr\}_{k}$ in the destination hierarchy where transfer should be attempted.
\item Match the selected sub-problems $\bigl\{r_k^{(j)}\bigr\}_{j,k}$ in $\problemtwo$ to sub-problems in $\problemone$. (Section~\ref{sec:xfer-cluster-correspond}) \label{item:cluster-match-xfer-alg}
\item (Optional) Match the statespace graphs of paired sub-problems using a suitable graph-matching algorithm. This can be done globally at a given scale, or locally for each matched pair  of \clusterstext from Step~(\ref{item:cluster-match-xfer-alg}). (Section~\ref{sec:graph-matching}) \label{item:corr-xfer-alg}
\item Proceeding bottom-up, from the finest scale upwards:
\begin{enumerate}\itemsep 0pt
\item If a pair of matched problems at the current scale is contained within a region of the previous, finer statespace where transfer has already occurred, remove the pair from the list of transfer possibilities.
\item For each remaining pair of matched sub-problems at the current scale, perform action correspondence and {\em tentatively} transfer the candidate sub-problem policies and/or potential operators from $\problemone$ to $\problemtwo$, along the statespace correspondences determined in Step~(\ref{item:corr-xfer-alg}). (Sections~\ref{sec:policy-xfer}, \ref{sec:potential-xfer})
\item Determine transferability of the policies and/or potential operators solving the first problem to the matched sub-problems in the second problem from Step~(\ref{item:cluster-match-xfer-alg}). (Section~\ref{sec:xfer-detect}) \label{item:detect-xfer-alg}
\item Retain only the transferred sub-problem policies/potential operators identified in Step~(\ref{item:detect-xfer-alg}) to be transferrable.
\end{enumerate}
\item Solve $\problemtwo$ with an algorithm such as Algorithm~\ref{alg:hierarchy-solve} (or  other variants discussed in Section~\ref{sec:mdp_solution}), starting from the transferred policies and potential operator derived values as initial policies and guesses for $V_{\textrm{coarse}}$ (respectively), within the appropriate \clusterstext and scales.
\end{enumerate}
\end{minipage}
}
\label{alg:transfer-general}
\end{algorithm}


%
\subsection{Notation and Assumptions}
In the following, $\problemone,\problemtwo$ will denote two distinct MDP hierarchies with underlying statespaces $S_1,S_2$ and action sets $A_1,A_2$, respectively. The notation $P_i,R_i,\Gamma_i$  for $i\in\{1,2\}$ will in this section refer to the respective transition, reward and discount tensors for problems $\problem{i}, i\in\{1,2\}$. To simplify the notation we will not explicitly attach \clustertext indices to these quantities, and assume an appropriate truncation/restriction (see Section~\ref{sec:notation}) that will be clear from the context. The notation $\clusterindex\in \problem{i}$ indicates that a \clustertext\ \clusterindex is a \clustertext at some scale of the hierarchy $\problem{i}$. As before, $\icluster, \dcluster$ denote the interior and boundary of the  \clustertext\ $\clusterindex$, respectively. For all objects, the scale in question will either be arbitrary or clear from the context. Unless otherwise noted, $\pi^*$ refers to the optimal policy for $\problemone$ at the appropriate scale. Throughout this section, we will assume for simplicity that optimal source problem policies $\pi^*$ are deterministic maps from states to actions. This assumption is not important for the main ideas discussed here, and is natural since transferred information is assumed to pass from a pre-solved source problem to an unsolved or partially solved destination problem. In this case the policy for the solved source problem may be chosen deterministic\footnote{In any event, if the optimal policy  is non-deterministic, one can still consider taking $\pi^o(s)=\arg\max_a\pi^*(s,a)$ as the optimal policy for transfer purposes.}.

At the coarsest scale in a hierarchy, there is only one ``\clustertext'' and there are no local bottlenecks. To see the the coarsest scale as just a special case falling within a more general framework, we will treat  the coarsest scale as a single \clustertext consisting of only interior states; the boundary will be the empty set. As will be explained below, partial transfer of a policy to a subset of the states in a \clustertext\ is possible, but since the coarsest scale usually involves a small statespace, full statespace graph matching between $\problemone$ and $\problemtwo$ should be inexpensive and potentially less error-prone. If this is the case, we may transfer into the entire scale, although the transfer will be seen as a transfer into the interior of a single \clustertext in order to fit within a common transfer framework.

We set some further ground rules, since the range of possible transfer scenarios is large and diverse. We will restrict our attention to transferring policies and potential operators from a \clustertext\ $\clusterone\in \problemone$ at scale $j$ belonging to a solved source problem, to a \clustertext\ $\clustertwo\in \problemtwo$, also at scale $j$, belonging to an unsolved destination problem. We assume scales have been suitably matched, and do not address transfer between different scales. We will assume that if a matching $\eta:S_2'\subset S_2\to S_1'\subset S_1$ between the statespaces of the source and destination problems is given (see Section~\ref{sec:graph-matching} below), it is a bijection onto its image. We do not treat degenerate situations, such as $\eta(\clustertwo)\cap \clusterone = \emptyset$, and only consider transfer between subsets of \clusterone and \clustertwo for which there is a given correspondence.

Finally, when considering \clustertext-to-\clustertext policy transfers, we will focus our attention on transfer to the {\em interior} of a \clustertext only -- bottlenecks on a \clustertext's boundary will not receive a transferred policy. Unless prior knowledge regarding the matching between \clusterstext is available, we do not recommend transferring a policy to bottlenecks. Bottleneck states typically play a pivotal role in transitions across \clusterstext, and transfer errors at bottlenecks can slow down the solution process. Assessing transferability (Section~\ref{sec:xfer-detect}) at bottleneck states forming the boundary of a given \clustertext is also more involved because one has to decide how to keep the problem of determining transferability for one \clustertext separate from that of the other \clusterstext. 
Instead, solutions on the bottlenecks at a given scale should ordinarily be computed jointly as the solution to the next higher (compressed) MDP, or one can pursue transfer of an entire coarse scale (all bottlenecks simultaneously) if possible.

\subsection{Cluster Correspondence}
\label{sec:xfer-cluster-correspond}
A correspondence between \clusterstext at a given scale is established by pairing \clusterstext deemed to be closest
to each other in a suitable metric on graphs. A natural distance between graphs is given by the average pair-wise Euclidean distances between diffusion-map embeddings of the underlying states. Let $G,G'$ be two directed, weighted statespace graphs corresponding to a pair of \clusterstext of size $|S|,|S'|$, and let $\{\xi\}_k, \{\xi'\}_k$ denote the respective collections of diffusion map embeddings computed according to Section~\ref{sec:diffmaps}.
Then we define
\[
d(G,G') := \frac{1}{|S||S'|}\sum_{i,j}\|(\tau\circ\xi_i) - \xi'_j\|_2
\]
where $\tau$ is the sign alignment vector defined by Equation~\eqref{eqn:tau-signs}. Given a pair of problems $(\problemone, \problemtwo)$, a \clustertext\ \clusterone in $\problemone$ is matched to the \clustertext
$$
\clustertwo^* =\arg\min_{\clustertwo\in \problemtwo}d(G_{\clusterone},G_{\clustertwo})
$$ in $\problemtwo$, where $G_{\clusterindex}$ denotes the weighted statespace subgraph corresponding to $P^{\pi}_{\clusterindex}$ for some choice of $\pi$ (e.g. $\pi^u$). We will only compare \clusterstext occurring at similar scales.

\begin{figure}[t]
\centering
\begin{tikzpicture}[->,>=stealth', shorten >= 1pt,semithick, auto]
\tikzstyle{every state}=[fill=blue!20,draw,minimum width=0.8cm,minimum height=0.8cm]
\tikzstyle{MySmall}=[fill=black!10,draw=black!50,minimum width=0.6cm,minimum height=0.6cm]
\tikzstyle{MyDashed}=[dashed,draw=black!75]
\fill[color=black!20!red!5] (-0.7,0.3) rectangle (2.7, 4.5);
\fill[color=black!20!red!5] (4.3,0.3) rectangle (7.7, 4.5);
\node[state] (a) at (1,1)  {$s$};
\node[state] (ap) at (1,3)  {$s'$};
\node[state] (b) at (6,1)  {$w$};
\node[state] (bp) at (6,3)  {$w'$};
\node[state] (ap_l) at (0,3) [MySmall]  {};
\node[state] (ap_r) at (2,3) [MySmall]  {};
\node[state] (bp_l) at (5,3) [MySmall]  {};
\node[state] (bp_r) at (7,3) [MySmall]  {};
\path (a) edge (ap)
      (b) edge (bp)
      (a) edge[draw=black!30] (ap_l)
      (a) edge[draw=black!30] (ap_r)
      (b) edge[draw=black!30] (bp_l)
      (b) edge[draw=black!30] (bp_r);
\path (ap) edge[MyDashed,->,bend left=45] node[below] {$\eta^{-1}$} (bp)
      (a) edge[MyDashed,<-] node {$\eta$} (b);
\path (-0.1,1.9) node {$\pi^*(s)$};
\path (7.1,1.9) node {$\pi_{\clustertwo}(w)$};
\path (0.2,4.17) node {$\mathbf{\problemone}$};
\path (6.85,4.17) node {$\mathbf{\problemtwo}$};
\end{tikzpicture}
\caption{{\small\em Policy transfer: action mapping. We set $\pi_{\clustertwo}(w)=\arg\max_a P_2(w,a,w')$
if $w$ is matched with $s$, and $w'$ is matched with $s'=\arg\max_{\ell} P_1(s,\pi^*(s),\ell)$.}}
\label{fig:policy-xfer}
\end{figure}
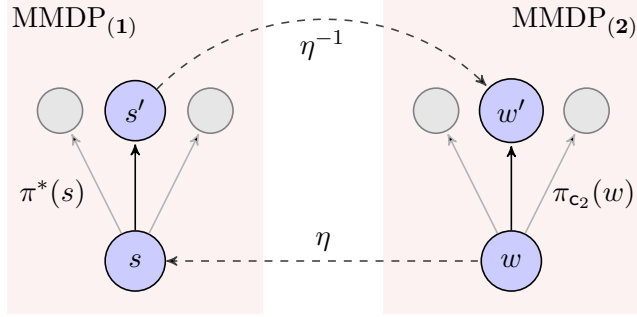

\subsection{Policy Transfer}
\label{sec:policy-xfer}
Given a pair of matched \clusterstext\ $\clusterone\in \problemone, \clustertwo\in \problemtwo$, we describe how the deterministic optimal policy $\pi^*$ from $\problemone$ can be transferred to some or all of $\vi$. Policy transfer may be carried out at any scale, and for subsets of a \clustertext's interior states. For example, one may also find that there are sub-tasks which are not exactly the same as solved tasks in a database, but do nevertheless bear a strong resemblance. In these cases one may pursue  partial policy transfer possibilities. Transferring via policies provides a convenient way to incorporate both full and partial knowledge:  a transferred piece of the policy can serve as an initial condition for computing a value function everywhere in the  \clustertext.

Assume we are given a bijective statespace mapping $\eta:S'_2\to S'_1$ such that 
$S'_1\subseteq S_1$, $S'_2\subseteq S_2$, $\vi\cap S'_2 \neq\emptyset$ and $\eta(\vi\cap S'_2)\cap \clusterone \neq\emptyset$.
That is, we require that $\eta$ matches at least part of $\vi$ with part of \clusterone.
Statespace graph matching will be discussed in Section~\ref{sec:graph-matching} immediately below; here we will assume $\eta$ is either given or simply taken to be the identity map. Next, let
\begin{equation}\label{eqn:W_intersect}
 \cW_{\eta} := \eta^{-1}\bigl(\eta(\vi\cap \dom\eta) \cap \clusterone\bigr)
\end{equation}
denote the subset of $\vi$ with a correspondence in \clusterone\footnote{Recall that, by construction, any state at any level of an MDP hierarchy is a state from the finest scale. Regardless of the scale at which the \clusterstext\ $\clusterone,\clustertwo$ occur, we may always consider subsets of $\clusterone,\clustertwo$ as subsets of the original underlying statespaces $S_1, S_2$ (resp.).}, and assume that $\clusterone,\clustertwo$ are both associated to scale $j$. 
An important aspect of policy transfer is the mapping of actions along $\pi^*$ in $\problemone$ to $\problemtwo$. 
Figure~\ref{fig:policy-xfer} illustrates action mapping for an arbitrary state $w\in \clustertwo$. 
If $\eta(w)=s\in S_1$, we can follow $\pi^*$ by
finding the state $\pi^*(s)$ is trying to transition to, 
$$
s'=\arg\max_{\ell\in S_1} P_1(s,\pi^*(s),\ell) .
$$
Then if $s'\in S_1$ corresponds back to $w'\in \clustertwo$ via $w'=\eta^{-1}(s')$, the $\problemtwo$ action most likely to induce a transition between $w$ and $w'$ is taken to be the transferred policy at $w$:
\begin{align*}
a^* &= \arg\max_{a\in A_2} P_2(w,a,w') \\
\pi_{\clustertwo}(w,\cdot) &= \delta_{a^*}, \qquad w\in\cW_{\eta} .
\end{align*}

Once each $w\in\cW_{\eta}$ has been assigned an action, the remaining missing policy entries in $\vi$
can be set to either the uniform distribution or to a previous policy guess. Abusing notation and using $\pi^u$ to denote either of the random uniform or previous policy,
\[
\pi_{\clustertwo}(s)=\pi^u(s),\qquad s\in \{\vi\setminus\cW_{\eta}\} .
\]

The resulting policy $\pi_{\clustertwo}$ can then be used as a starting point for local policy iteration in \clustertwo  with e.g. Algorithm~\ref{alg:hierarchy-solve}. A high-level summary of policy transfer and the subsequent solution process is given in Algorithm~\ref{alg:policy-xfer}. In particular, a value function $V_j$ everywhere in \clustertwo is computed by solving the local boundary value problem 
\begin{equation}\label{eqn:xfer-values}
V_j(s) =
\begin{dcases}
\bbE_{a\sim \pi_{\clustertwo}(s)}\left[
\sum_{s'\in {\clustertwo}}P_2(s,a,s')\bigl(R_2(s,a,s') + \Gamma_2(s,a,s')V_j(s')\bigr) \right]
& \text{if } s\in\vi, \\
V_{j+1}(s) & \text{if } s\in \partial {\clustertwo}
\end{dcases}
 \end{equation}
where $V_{j+1}$ is the value function associated to the coarse scale $j+1$ (see Section~\ref{sec:solve-local}). The value function $V_j$ and its associated policy can then be propagated up or down the hierarchy using the ideas discussed in Section~\ref{sec:mdp_solution}. For instance, $\problemtwo$ could be re-compressed from the scale at which \clustertwo resides (scale $j$) upwards using an updated policy derived from $V_j$ (possibly blended with a previous policy). The value function in \clustertwo can also be used to solve downwards below the current \clustertext by applying Algorithm~\ref{alg:hierarchy-solve} to the previous scale with $V_j$ serving as the initial coarse data.



\begin{algorithm}[t]
\caption{Policy transfer and subsequent solution process: high-level steps.}
\fbox{\begin{minipage}{0.984\textwidth}
\begin{enumerate}\itemsep -1pt
\item For each $\problemtwo$ state in $\cW_{\eta}$, transfer the policy from the corresponding state in $\problemone$ by mapping actions from $\problemone$ to $\problemtwo$.
\item For remaining states in $\vi$, fill in the policy with either the diffusion policy or a previous guess, if available.
\item Apply Algorithm~\ref{alg:hierarchy-solve} (or its variants) to $\problemtwo$ using the policy on $\vi$ as an initial condition.
\item Push down the resulting value function in \clustertwo to the previous scale, and continue solving downwards, or stop and return the resulting policy on $\vi$ if \clustertwo is at the finest scale $(j=0)$.
\end{enumerate}
\end{minipage}}
\label{alg:policy-xfer}
\end{algorithm}


\subsection{Transfer of Potential Operators}
\label{sec:potential-xfer}
Suppose $S_i^j$ denotes the full statespace for problem $i$ at scale $j$. 
At any (coarse) scale $j>0$ above the finest scale one can consider transferring the potential operator 
\begin{align*}
\cG &: \bbR^{S_1^j}\to \bbR^{S_1^j}\\
\cG &:= \bigl(I- (\Gamma_1\circ P_1)^{\pi^*}\bigr)^{-1}
\end{align*}
associated to the optimal policy $\pi^*$ at scale $j$ of problem $\problemone$. Here we let $P_1, \Gamma_1$ generically denote the Markov transition tensor and compressed discount factors (respectively) at the relevant scale of problem $\problemone$. We will consider for simplicity of the exposition the transfer of entire scales, and require for the moment that the statespace correspondence satisfy $\eta(S_2^j)=S_1^j$. In general, potential operators specific to \clusterstext may be readily transferred by extending the development here to include the ideas discussed in Section~\ref{sec:solve-local} (see Equation~\eqref{eqn:local-system} in particular).

Sub-problems at scales below $j$ of $\problemtwo$ decouple from each other given values at the states in $S_2^j$. A value function $V_2^j$ on $S_2^j$, may be computed from the transferred potential operator  and $\problemtwo$ rewards as follows. Let $\pi_{\clustertwo}$ denote the policy on $S_2^j$ transferred from $\pi^*$ according to Section~\ref{sec:policy-xfer}. Next, consider the $\problemtwo$ rewards $R_2$, aggregated over one step with respect to $(\pi_{\clustertwo}, P^{\pi^*})$, and mapped back to $\problemone$:
\begin{equation}\label{eqn:potential-rewards}
R_{2,1}(s) = \sum_{s'\in S_1^j}P^{\pi^*}(s,s')\bbE_{a\sim \pi_{\clustertwo}(\eta^{-1}(s))}\Bigl[
R_2\bigl(\eta^{-1}(s),a,\eta^{-1}(s')\bigr)\Bigr],\quad s\in S_1^j .
\end{equation}
The expectation on the right-hand side defines a system of rewards on $S_1^j$ by collecting the (one-step) rewards in $\problemtwo$ following the policy on $S_2^j$ determined by mapping $\pi^*$ to $\problemtwo$. Since the policy $\pi^*$ is deterministic, we do not take an expectation with respect to $\problemone$ actions  anywhere in~\eqref{eqn:potential-rewards}.
The value function $V_2^j$ is computed by applying $\cG$ to these rewards and mapping back to $\problemtwo$  along $\eta$,
\begin{equation}\label{eqn:potential-values}
V_j(s) = \bigl(\cG R_{2,1}\bigr)\bigl(\eta(s)\bigr), \qquad s\in S_2^j.
\end{equation}
We draw attention to the fact that if the reward system for $\problemtwo$ depends on actions, then as shown in Equation~\eqref{eqn:potential-rewards} computing a value function requires a set of rewards aggregated with respect to a transferred policy $\pi_{\clustertwo}$. In general, transferring a potential operator therefore also entails mapping actions across problems, and transferring a policy. If the $\problemtwo$ rewards in \clustertwo do not depend on actions however, then~\eqref{eqn:potential-rewards} reduces to a simpler expression not involving $\pi_{\clustertwo}$, however this situation is unlikely at any coarse scale.

Using a potential operator from $\problemone$ to compute values for $\problemtwo$ sub-problems provides three major advantages. Value determination is fast, $\cO(|S_1^j|^2)$ worst case, because the operator is given. The resulting value function for $\problemtwo$ at scale $j$ also respects the specific reward structure of $\problemtwo$. The third advantage is more subtle: the coarse MDP initially associated to scale $j$ of $\problemtwo$ results from compression with respect to a stochastic policy guess, and may not be compatible with the optimal policy at scale $j$ of $\problemone$. If we simply transferred $\pi^*$ from $\problemone$, but used $\problemtwo$'s coarse transition dynamics to perform value determination, it is likely that any improvement arising from knowledge of $\pi^*$ would be eliminated.  Assuming the transfer is viable,  $\problemone$'s potential operator will determine a value function that obeys the ``correct'' coarse-scale Markov process and provides a warm start towards finding the optimal fine-scale policy.

\subsection{Detecting Transferability}
\label{sec:xfer-detect}
Given a pair of matched \clusterstext\ $(\clusterone\in \problemone, \clustertwo\in \problemtwo)$, we would like to know whether transfer of a policy or potential operator from $\clusterone\in \problemone$ to some or all of $\vi$ can be reasonably expected to help solve $\problemtwo$. As in Section~\ref{sec:policy-xfer}, we will restrict our attention to detecting opportunities for transfer only to the interior of \clustertwo. 


\subsubsection{Policy Transferability}
One way to approximately determine transferability of a policy is to check whether running $\problemone$'s optimal policy $\pi^*$ in $\clustertwo\in \problemtwo$ results in more aggregate reward on average than executing the current policy we have for $\problemtwo$. In most cases the ``current'' policy will just be the diffusion policy $\pi^u$, so we will overload this notation to indicate either of the current policy or the diffusion policy. If the reward collected following $\problemone$'s policy is lower than that collected using the current policy, this suggests that $\pi^*$ does not provide a warm start, and transfer should not be attempted.  In addition, if no statespace correspondence is given then this test can be used to check whether assuming the default identity correspondence can still support transfer from \clusterone to \clustertwo.

The value function in \clustertwo given the current policy $\pi^u$ for $\problemtwo$ is computed in the usual way. 
Letting $V_j^u$ denote the desired value function and $V_{j+1}^u$ denote the current value function at the next coarser scale $j+1$, we must solve the system
\begin{equation}\label{eqn:vu-detect}
V_j^u(s) =
\begin{dcases}
\bbE_{a\sim\pi^u(s)}\left[\sum_{s'\in {\clustertwo}}P_2(s,a,s')\bigl(R_2(s,a,s') + \Gamma_2(s,a,s')V_j^u(s')\bigr) \right] & \text{if } s\in\vi, \\
V_{j+1}^u(s) & \text{if } s\in \partial {\clustertwo} .
\end{dcases}
 \end{equation}

The test for transferability compares $V_j^u$ to the value function $V_j$ describing rewards collected in ${\clustertwo}\in \problemtwo$ running $\pi^*$ from $\problemone$ computed according to Equation~\eqref{eqn:xfer-values}. In other words, we transfer following Section~\ref{sec:policy-xfer}, and then check whether we see any improvement relative to the current policy in \clustertwo. The result of the computations in  Equations~\eqref{eqn:xfer-values} and~\eqref{eqn:vu-detect} can be reused during the first iteration of Algorithm~\ref{alg:hierarchy-solve} (or its variants) when solving the transfer problem.
If similar underlying states of the environment play different roles in the different tasks $\problemone,\problemtwo$, then $V_j$ could differ significantly from $V_j^u$. The two functions should be compared on all of $\vi$ (not just $\cW_{\eta}$ defined in Equation~\eqref{eqn:W_intersect}). One can take a conservative approach and only pursue transfer if $V_j(s)\geq V_j^u(s), \forall s\in\vi$. Or, if the situation is less clear, assessing the improvement may involve other heuristics. For example, a relative comparison such as
\begin{equation}\label{eqn:detect-test-matching}
 \sum_{s\in\vi}\sgn\bigl(V_j(s) - V_j^u(s)\bigr)\log\left(\frac{|V_j(s) - V_j^u(s)|}{|V_j^u(s)| 
+ \bbone_{(V_j^u(s)=0)}} + 1\right) \stackrel{?}{>} 0 .
\end{equation}
This test checks whether the policy $\pi^*$ provides a ``warm start'' relative to $\pi^u$ given the transition dynamics and reward structure for $\problemtwo$. If the inequality~\eqref{eqn:detect-test-matching} above is satisfied, then we can proceed with transferring the policy from \clustertext\ $\clusterone\in \problemone$ to $\clustertwo\in \problemtwo$. Note that since interior \clustertext values are computed in Equation~\eqref{eqn:xfer-values} with the \clustertext boundary fixed, the problem of assessing transferability from \clusterone to \clustertwo is independent of other \clusterstext in $\problemone$ or $\problemtwo$.


\subsubsection{Potential Operator Transferability}
The process for determining whether transferring a potential operator will be helpful or not is similar to the procedure for policies. Transferring a potential operator is equivalent to assuming that the dynamics of ${\clusterone}\in \problemone$ apply to ${\clustertwo}\in \problemtwo$. Thus, as with policies, it is worthwhile to consider comparing the expected discounted rewards collected while following the transition dynamics governing $\problemone$ in $\problemtwo$ to the no-transfer alternative. The expected reward starting from states in $\cW_{\eta}$ under potential operator transfer is given by Equation~\eqref{eqn:potential-values}.
If $\cW_{\eta}$ is a proper subset of $\vi$, then a value function everywhere on $\vi$ can be obtained by solving a small boundary value problem. In this case, $\cW_{\eta}$ is added to the boundary (in addition to \clustertwo's bottlenecks) and the values computed by Equation~\eqref{eqn:potential-values} serve as the boundary values for states in $\cW_{\eta}$:
\begin{equation*}\label{eqn:potfn-xfer-bvp}
V_j(s) =
\begin{dcases}
\bbE_{a\sim\pi^u(s)}\left[\sum_{s'\in {\clustertwo}}P_2(s,a,s')\bigl(R_2(s,a,s') + \Gamma_2(s,a,s')V_j(s')\bigr) \right], & \text{if } s\in\vi\setminus\cW_{\eta}, \\
V_{j+1}(s), & \text{if } s\in \partial {\clustertwo} \\
\bigl(\cG R_{2,1}\bigr)\bigl(\eta(s)\bigr), & \text{if } s\in \cW_{\eta}
\end{dcases}
 \end{equation*}
where $\pi^u$ is the initial policy (diffusion or otherwise), and $V_{j+1}$ is the initial coarse value function. Values for the remaining states are computed according to $\problemtwo$'s rewards and transition probabilities, as this the only possibility in the absence of a wider correspondence. Analogous to the case of policy transfer, the computations in Equation~\eqref{eqn:potfn-xfer-bvp} may be reused when solving the transfer problem.
 
If we do not transfer the potential operator, we would otherwise just follow the current (or uniform stochastic) policy in \clustertwo with the usual $\problemtwo$ dynamics. The expected discounted reward when there is no transfer is determined by solving Equation~\eqref{eqn:vu-detect} as before. The final comparison between transfer/no-transfer can be performed on all of $\vi$, and may involve a heuristic such as Equation~\eqref{eqn:detect-test-matching}. If the reward system in $\problemtwo$ is strongly dependent on actions, then the quality of a potential operator transfer may also depend on the quality of the mapped policy $\pi_{\clustertwo}$ by way of Equation~\eqref{eqn:potential-rewards}. In such situations one can also assess the quality of $\pi_{\clustertwo}$ by using the procedure above.

\subsection{Statespace Graph Matching}
\label{sec:graph-matching}
Establishing a correspondence between the discrete, finite statespaces of two problems can be an important prerequisite for some, if not most, types of transfer. Recall that a problem's statespace graph is a graph with states as its vertices and edges/weights defined by a transition probability kernel. Such a graph may, for instance, be characterized by a graph Laplacian of the type defined in Section~\ref{sec:diffmaps}. The goal of a statespace graph matching is to establish a correspondence between the {\em roles} played by states in each problem. Consider for example two
related problems $\problemone,\problemtwo$, each with a single terminal (goal) state. It would be desirable to be able to match the terminal states as ``goals'', even if the terminal states are different in the
sense that they have different representations in some underlying space (e.g. as features or coordinates in a Euclidean space). The same could be true for other states that play a pivotal role, such as ``gateway'' states directly connected to goal states. Graph matching ultimately seeks to abstract away problem-specific roles from the underlying state representations, and then match similar roles across problems. We will further illustrate this concept by way of several examples in Section~\ref{sec:examples}.

Although statespace matching can be important for a transfer problem, it can also be expensive computationally and
imprecise in practice. For some problems defined on discrete domains, key correspondences may need to be correct, otherwise the transferred information may actually diminish performance. For these reasons we do not require graph matching nor do we propose a full solution to the matching problem. We will restrict our attention to transfer scenarios where:
\begin{enumerate}\itemsep -1pt
\item[(i)] It is possible to use a default ``identity'' correspondence, or detect that that the identity is a poor choice and transfer should not be attempted (using the ideas in Section~\ref{sec:xfer-detect}).
\item[(ii)] The graph matching is relatively simple, and there is a limited potential for catastrophic errors. For example, matching at coarse scales between small collections of states.
\end{enumerate}

The algorithm we will use to match statespace graphs is heuristic. Given two sets of states $\clusterone\in \problemone, \clustertwo\in \problemtwo$,
\begin{enumerate}\itemsep -1pt
\item Compute the pairwise diffusion distances $d(s_i,s_j)$, $s_i\in {\clusterone}, s_j\in {\clustertwo}$, according to Section~\ref{sec:diffmaps}. Note that this does not involve any particular underlying representation associated to the statespaces.
\item Build the affinity matrix
$W_{ij} = \exp\bigl(-d^2(s_i,s_j)/\sigma^2\bigr)$, for some appropriate choice of $\sigma$ (e.g. the median pairwise distance in the set).
\item Apply any graph matching algorithm based on affinities.
\end{enumerate}
Graph matching is itself an area of active research, and several algorithms exist~\citep{Fremuth-Paeger99,SanghaviMW07,Huang:AISTATS:07,Huang:AISTATS:11}. For a graph with $|V|$ vertices and $|E|$ edges, the min-cost flow algorithm of \citep{Fremuth-Paeger99} has $\cO(|V||E|)$ running time. Jebara and collegues have improved upon this with a belief-propagation algorithm giving a running time of $\cO(|V|^{2.5})$ on average~\citep{Huang:AISTATS:07,Huang:AISTATS:11} (but $\cO(|V||E|)$ in the worst case).

The diffusion map embeddings may be computed either locally within \clusterstext if \clusterone and \clustertwo are contained in \clusterstext at a coarser scale, or approximately with a small number of eigenvectors when \clusterone and/or \clustertwo is the entire problem statespace. Specific problem knowledge may guide in many cases the choice of \clusterone and \clustertwo. For example, we may need to match a \clustertext\ $\clustertwo\in \problemtwo$ to states in $\problemone$, but might reasonably expect that \clustertwo can only correspond to a small number of states in $\problemone$ rather than all of $S_1$. The examples discussed in Section~\ref{sec:examples} below illustrate graph matching and the transfer procedures suggested above in more detail.

\section{Experiments}
\label{sec:examples}
We will illustrate compression and transfer learning in the case of three examples: a discrete $50\times 50$ gridworld domain with multiscale structure, a 3-dimensional continuous two-task inverted pendulum problem, and pair of problems based on the ``playroom'' domain of~\cite{Singh2004,Barto:ICDL:04}. The gridworld tasks require an agent to navigate to a goal location in a 2D environment. The inverted pendulum problem involves first moving a cart to a desired position, and then moving the cart while balancing the pendulum to another position. Finally, the playroom domain examples involve learning to carry out sequences of specific interactions with various objects and actuators in a desired order. The setup of compression, transfer and transfer detection is the focus of this section, rather than an exhaustive performance comparison with other algorithms. For this reason, most of the performance plots below show error versus the number of algorithm iterations, even if different algorithms have dramatically different computational complexity per iteration (in particular, the proposed multiscale algorithms have a cost per iteration much smaller than global algorithms such as policy iteration).

We will consider several multiscale algorithms, obtained by choosing different paths along the flow diagram in Figure~\ref{fig:soln-flow}, and different numbers of interior or boundary update iterations. Each variant has the basic structure of Algorithm~\ref{alg:hierarchy-solve}, however the particular updates applied may differ. The following table summarizes the multiscale algorithms we will consider:

\begin{table}[h]
\centering
\label{tab:expt-algs}
\begin{tabular}{c|c|c}
Algorithm Name & Interior Update & Boundary Update \\ \hline
\texttt{oo} & once & once\\
\texttt{oc} & once & Alg.~\ref{alg:hierarchy-solve} \\
\texttt{or} & once & recompress \\
\texttt{co} & to convergence & once\\
\texttt{cc} & to convergence & Alg.~\ref{alg:hierarchy-solve} \\
\texttt{cr} & to convergence & recompress
\end{tabular}
\caption{Multiscale algorithms tested in the experiments.}
\end{table}

The designations ``once'' and ``to convergence'' refer to the number of updates applied to the interior/boundary states, before updating the boundary/interior. There are two \clustertext interior update possibilities: either we perform policy iteration in each \clustertext until convergence ({\em ``to convergence''} -- when the relative error between iterates falls below 0.01), or we apply only one policy iteration per \clustertext ({\em ``once''}). To make comparisons fair, for algorithms iterating within \clusterstext to convergence, each pass applying one local policy iteration update to all the \clusterstext is counted as a single outer ``algorithm iteration'' in the plots\footnote{This is overly-conservative because, in general, convergence rates will be different across \clusterstext. We have assumed that every \clustertext has the worst convergence rate.}. The bottlenecks (boundary) are updated either as in Algorithm~\ref{alg:hierarchy-solve}, by way of repeated local averaging steps ({\em ``Alg.~\ref{alg:hierarchy-solve}''}), or by {\em recompressing} the fine scale MDP and then solving the resulting coarse MDP ({\em ``recompress''}). For accounting purposes, a boundary update, regardless of type, is considered part of the same outer algorithm iteration as the immediately preceding interior update (or pass over \clustertext interiors). In all experiments, the cost of initial hierarchy construction and transfer detection/policy-mapping (when applied) is not included, as they are only done once for a problem.

Which of the algorithms is best suited to a given problem strongly depends on whether 
the initial data can be trusted. There are three kinds of initial data in question: the initial fine scale policy, the initial coarse value function, and the policy or policies used to initially compress the fine scale MDP. Empirically, we have observed the latter two types to be the most significant. If the coarse value function is trustworthy, then iterating within \clustertext interiors to convergence before updating the bottlenecks is generally optimal. Initial boundary information is allowed to propagate throughout the fine scale interior, and the boundary values are modified only after the interior cannot be improved further. This situation might arise when pursuing potential operator transfer, or if the boundary value function solves a coarse MDP compressed according to a pool of policies as in Section~\ref{sec:cluster_pols}. By contrast, if the initial coarse data is suspect, then we may choose to iterate the interior once or a small number of times, and then improve the boundary immediately afterwards. This may be the case if, for example, the initial coarse value function solves a coarse MDP compressed with respect to the diffusion policy. Applying many interior iterations can otherwise propagate erroneous information, and slow the solution process considerably. In short, when the initial coarse information is trustworthy, it should be leveraged as far as possible. Otherwise, if it is suspect, the coarse initial data should be imposed lightly.

For those experiments where a pool of \clustertext policies was used to initially compress the fine scale (following Section~\ref{sec:cluster_pols}) {\em and} there is recompression, we will effectively {\em add} the current fine policy to the existing pool of initial guesses, and use the augmented pool to recompress. This allows the solution at the coarse scale to ignore actions invoking the current fine policy if the actions corresponding to the initial guess policies are better. Under these conditions, the coarse value function can only increase, since we are providing additional actions beyond those resulting from the initial compression. Since each fine scale \clustertext policy corresponds to a coarse action, recompression is efficient in practice, and involves compressing only with respect to the new fine policy. One can concatenate new coarse probabilities, discounts, and rewards with those resulting from the initial compression and then proceed to solve at the coarse scale. In experiments involving initial compression with respect to the diffusion policy only, we recompressed using the current fine policy, and discarded coarse actions corresponding to the diffusion policy. In all experiments, compression involved blending each \clustertext policy with a small amount ($\lambda=0.01$) of the diffusion policy in order to preserve the boundary reachability assumption.

\begin{figure}[t]
\centering
\fbox{
\includegraphics[width=0.3\textwidth]{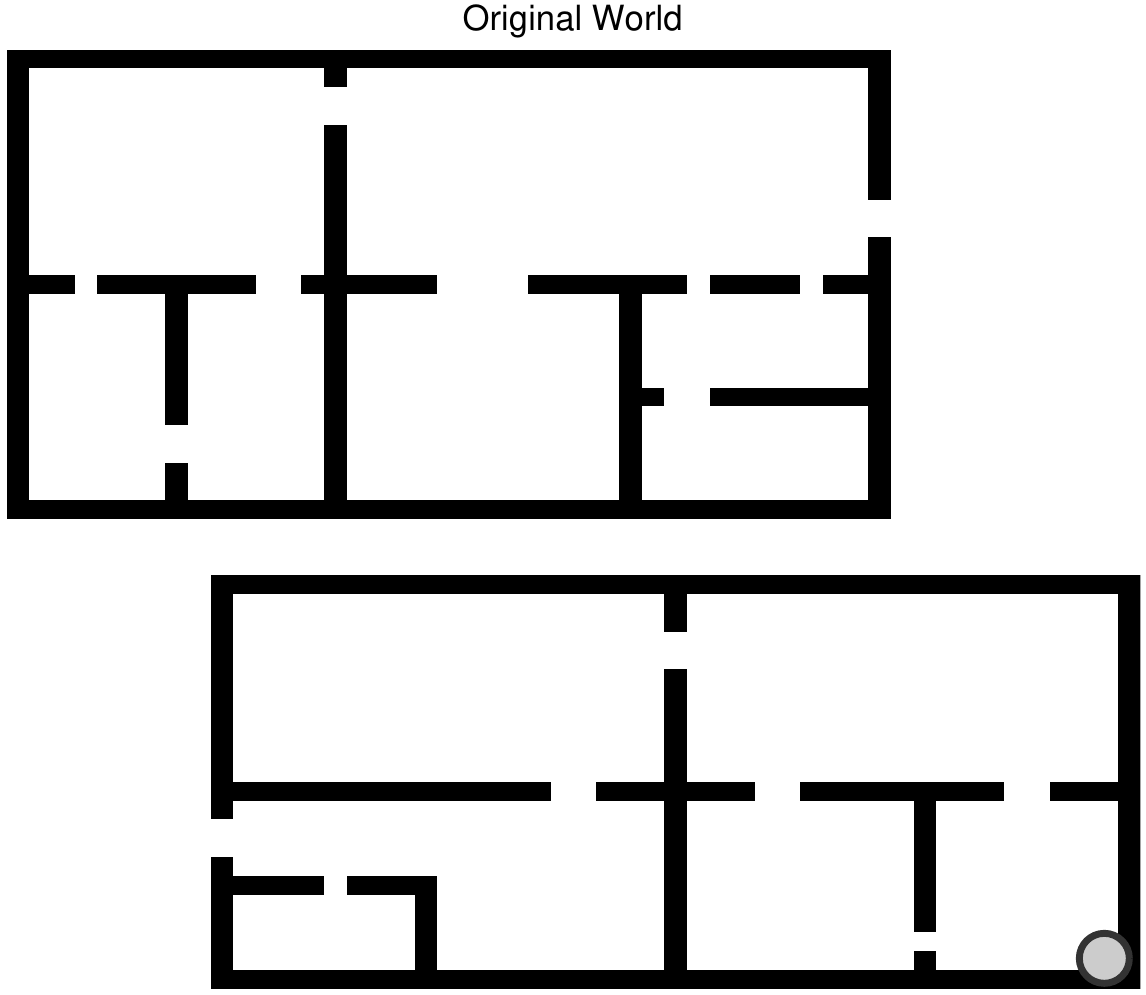}}
\hskip 0.05cm
\fbox{
\includegraphics[width=0.3\textwidth]{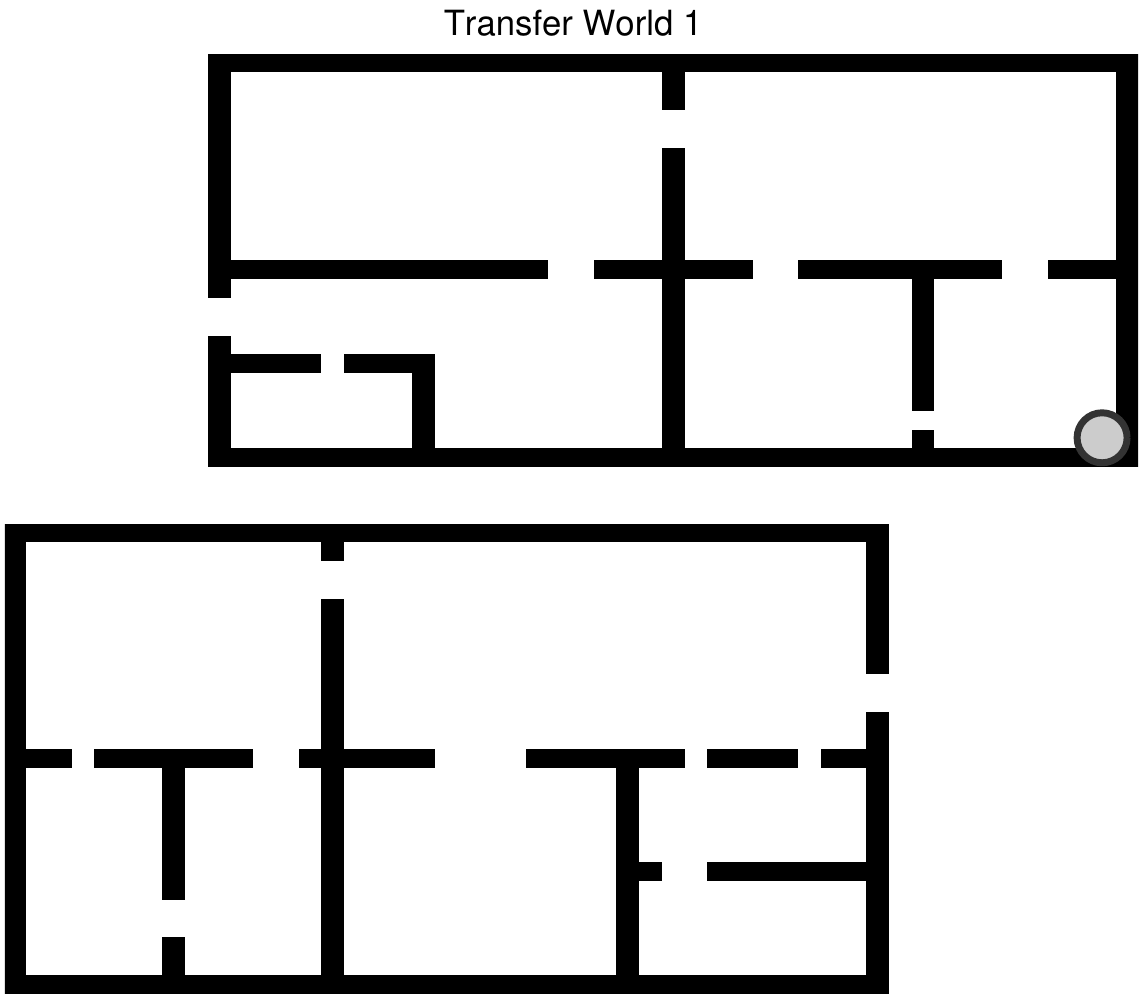}}
\caption{{\small\em (Left) Original grid world. (Right) Transfer to a world with similar multiscale structure and (relative) goals, but a completely different optimal policy. 
}}
\label{fig:grid_transfer}
\end{figure}

\begin{figure}[t]
\centering
\fbox{\includegraphics[width=0.31\textwidth]{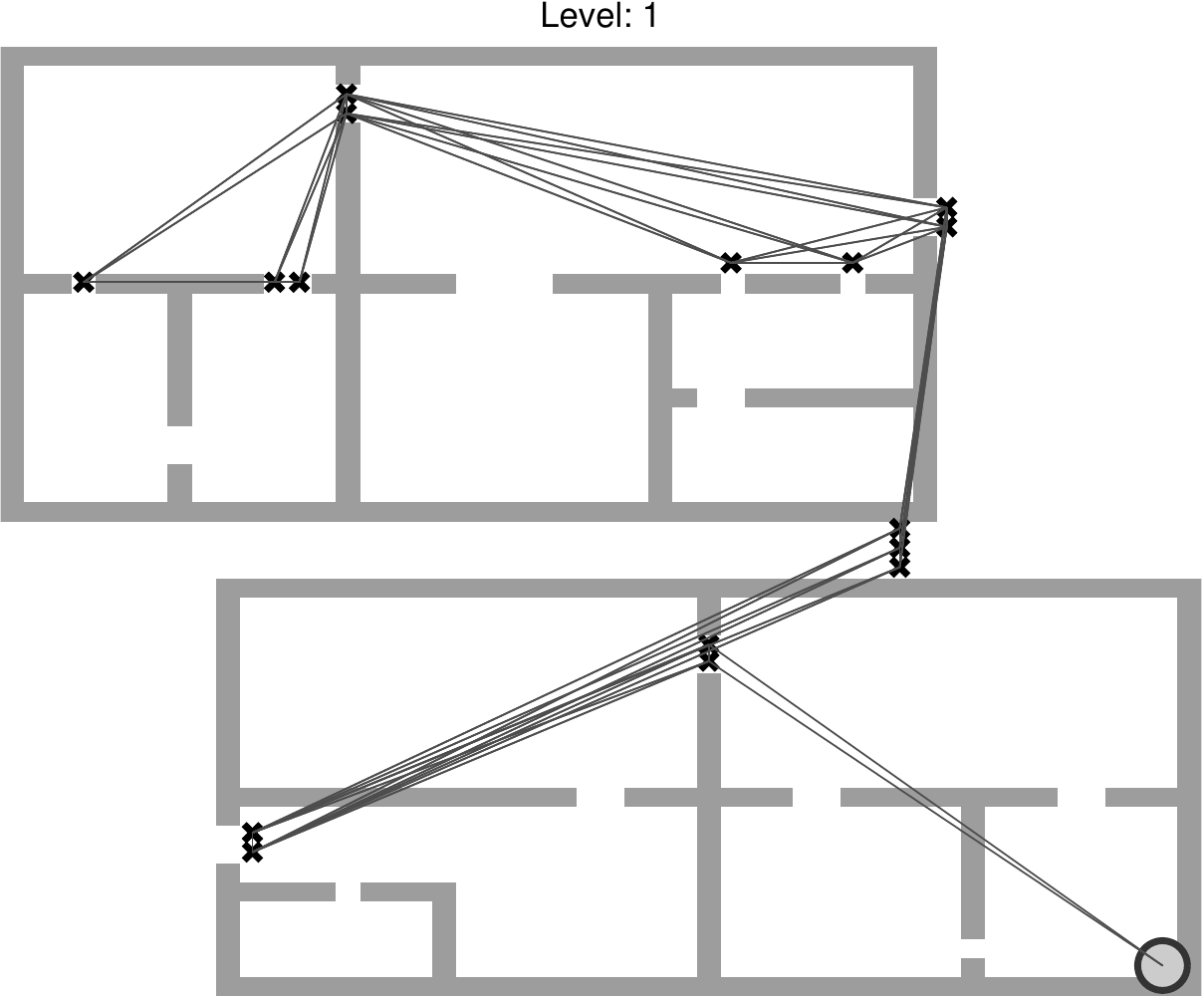}}
\fbox{\includegraphics[width=0.31\textwidth]{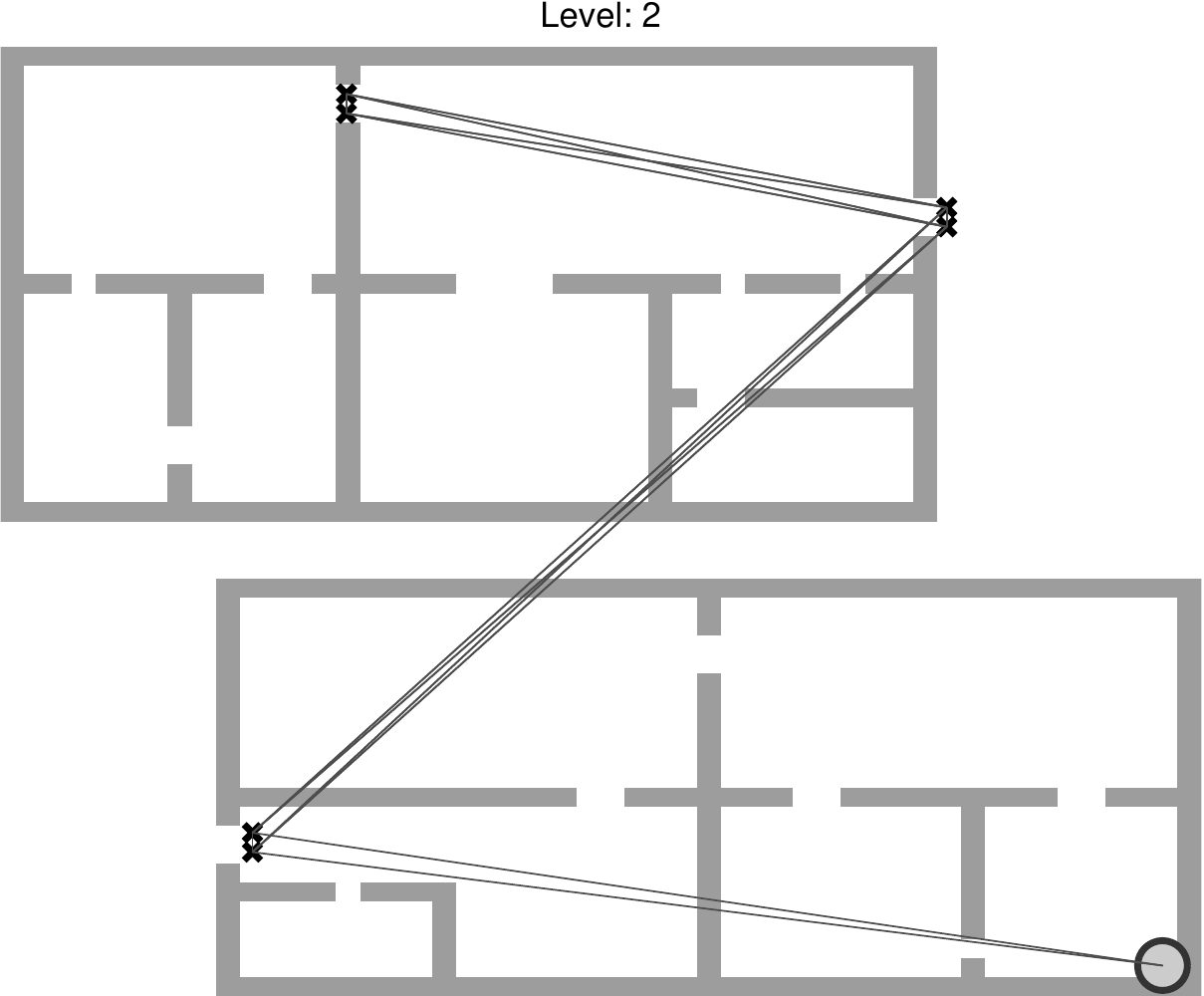}}
\fbox{\includegraphics[width=0.31\textwidth]{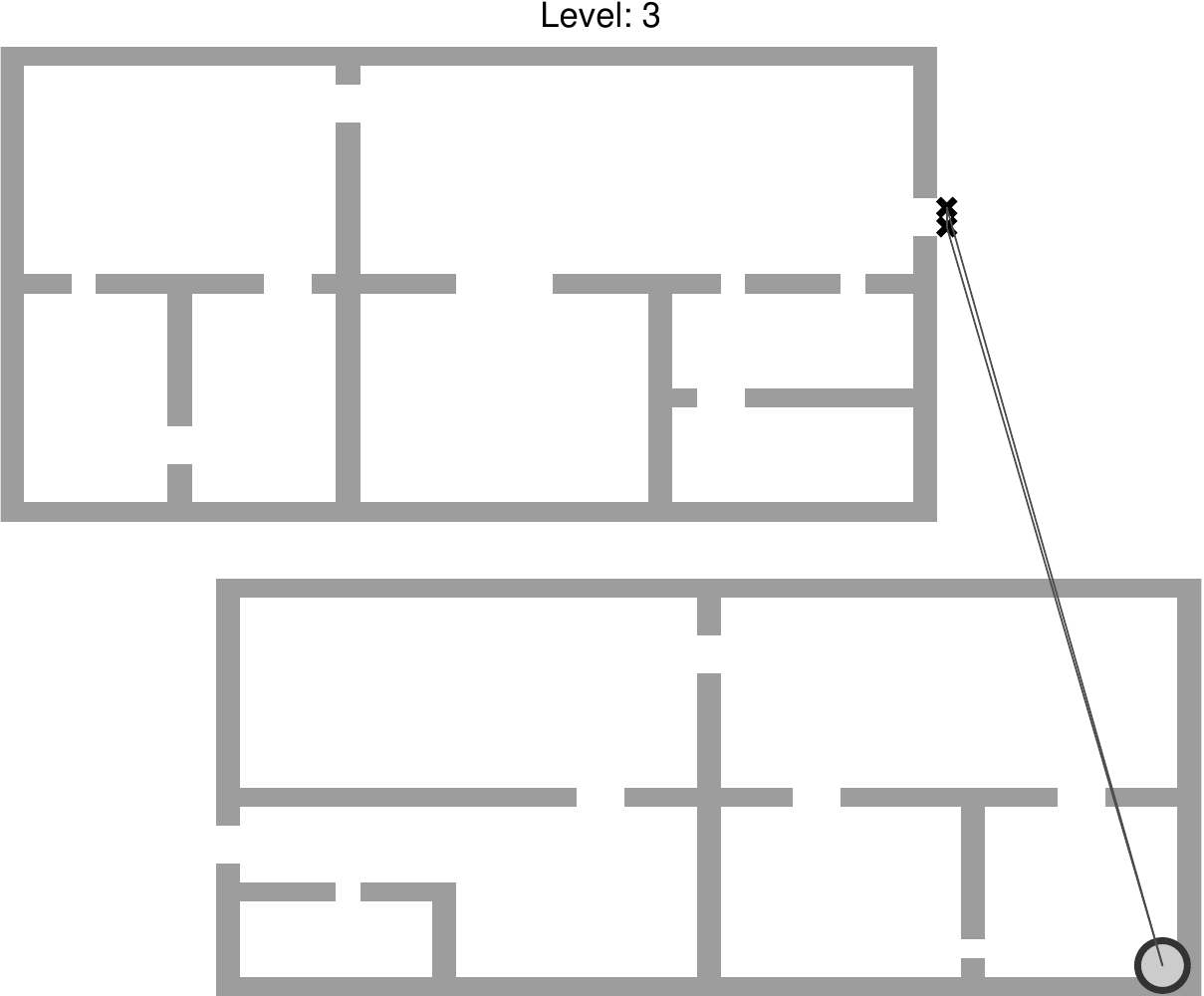}}
\caption{{\small\em  Detected gridworld bottlenecks ('\textnormal{\texttt{x}}') at scales ranging  from fine to coarse (left to right). At each scale, non-zero entries of the corresponding compressed MDP's transition matrix are plotted as links between bottleneck states, and depict the directed statespace graphs. The terminal goal state is marked by a circle. }}
\label{fig:grid_bns}
\end{figure}

\begin{figure}[t]
\centering
\fbox{\includegraphics[width=0.31\textwidth]{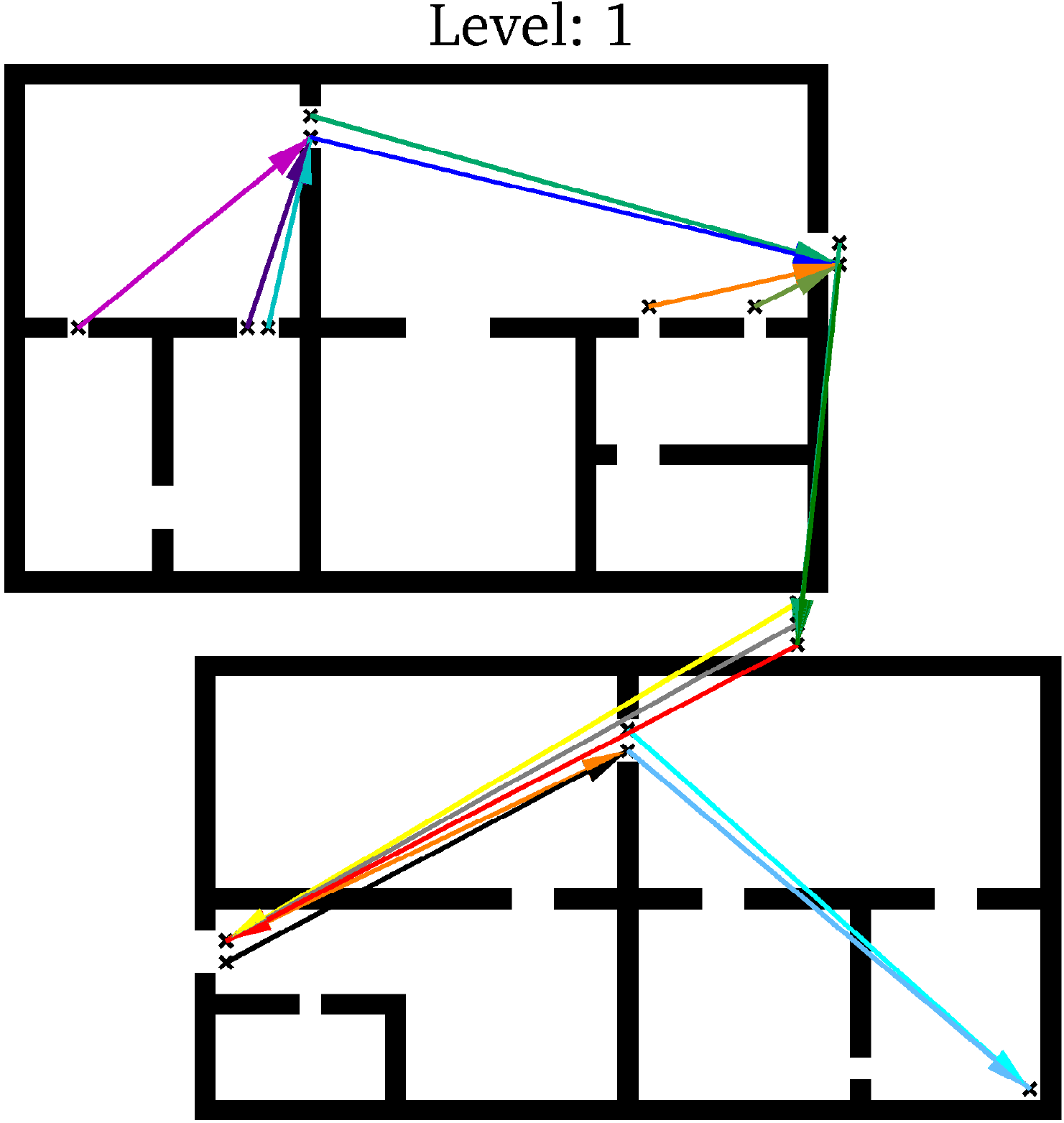}}
\fbox{\includegraphics[width=0.31\textwidth]{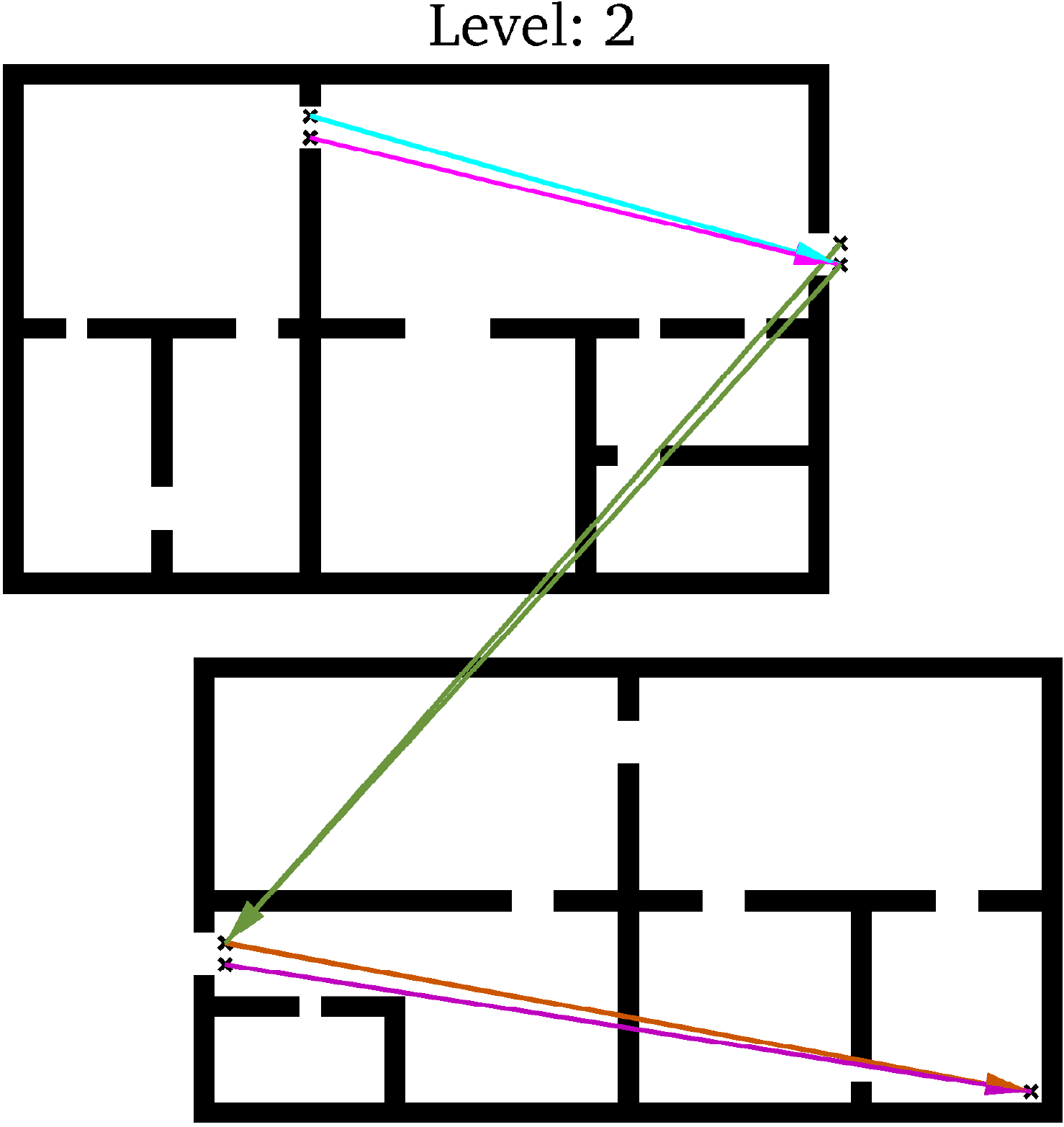}}
\fbox{\includegraphics[width=0.31\textwidth]{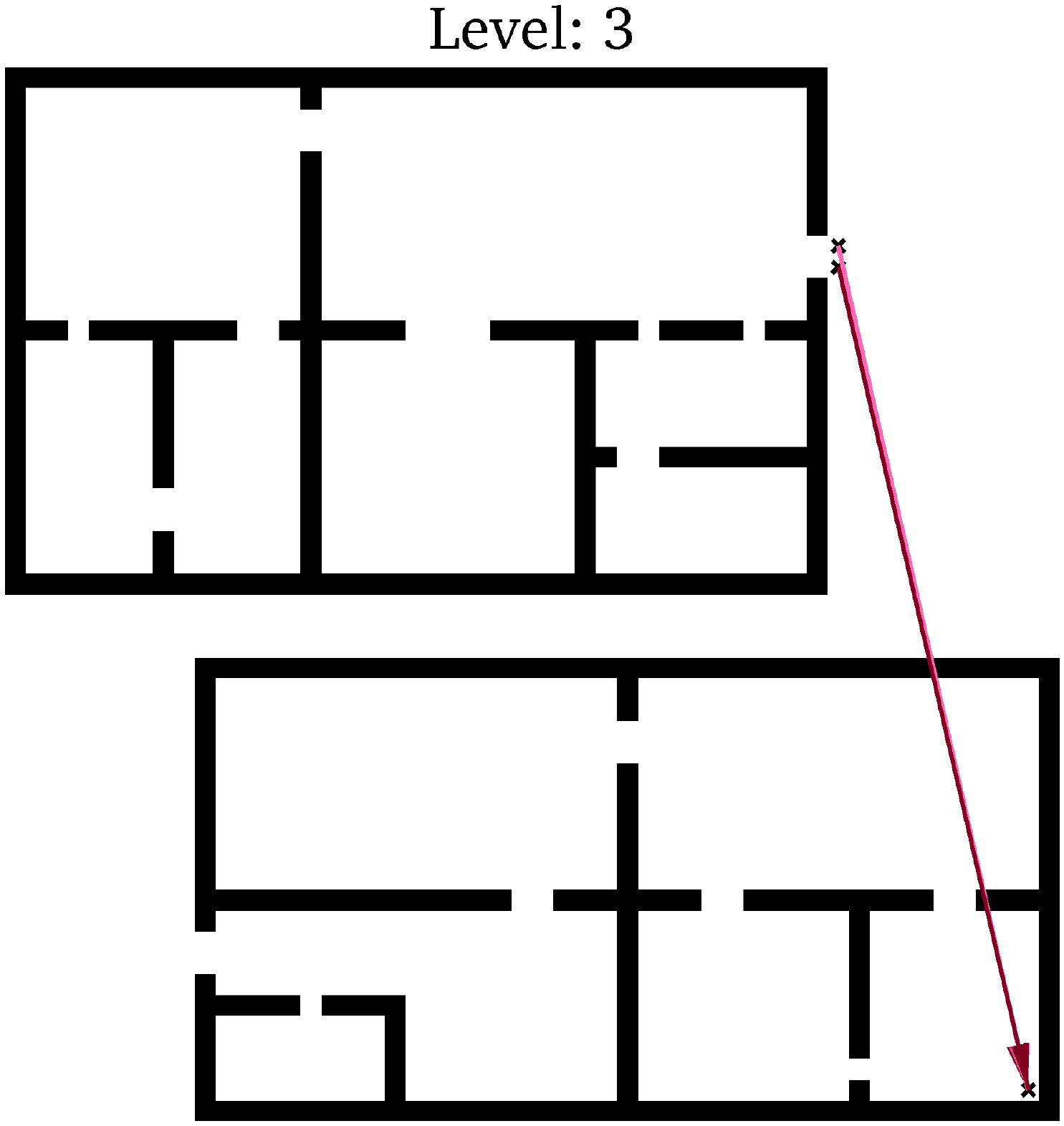}}
\caption{{\small\em Greedy coarse policies solving MDPs at successively coarser scales (left to right), visualized as arrows representing transitions between bottlenecks. The MDPs were compressed with respect to the local guesses (for each scale) described in Section~\ref{sec:cluster_pols}. }}
\label{fig:grid_pols}
\end{figure}

\begin{figure}[t]
\centering
\includegraphics[width=0.99\textwidth]{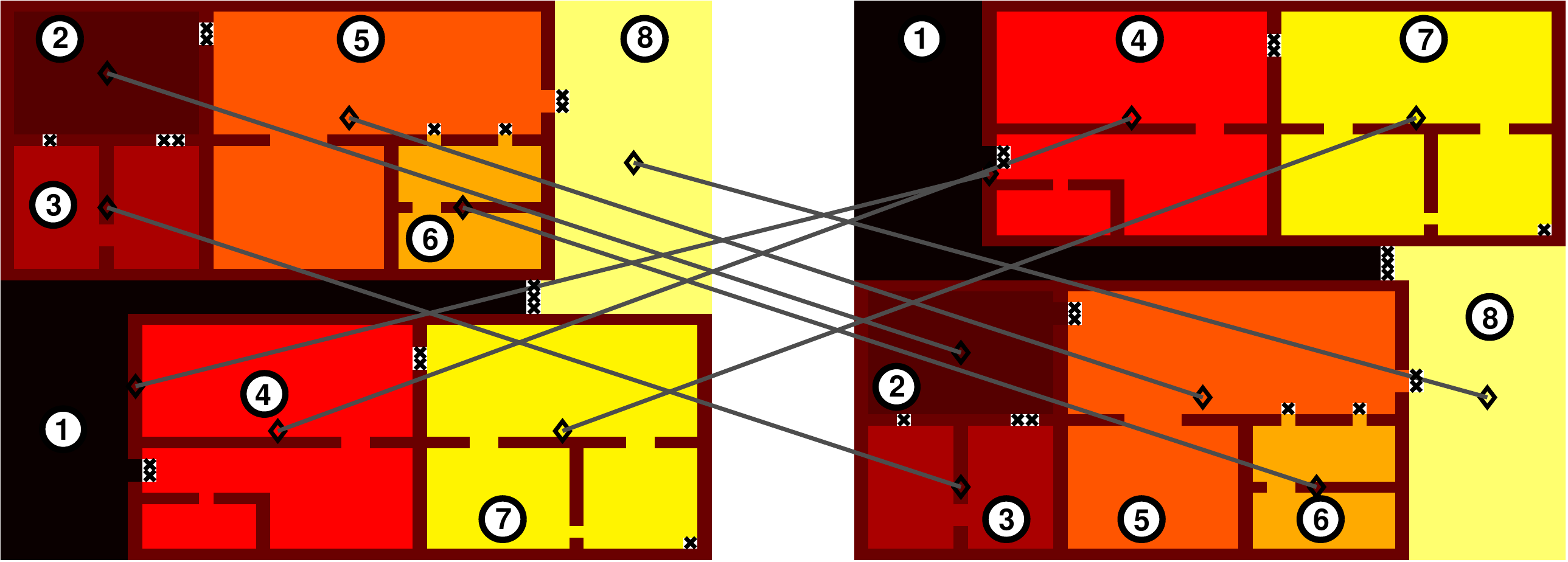}
\caption{{\small\em Output of the \clustertext correspondence algorithm (Section~\ref{sec:xfer-cluster-correspond}) applied to the gridworld transfer problem described in Section~\ref{sec:grid_expt}. All \clusterstext are correctly matched.}}
\label{fig:grid-tr1-matching}
\end{figure}


\subsection{Gridworld Domain}
\label{sec:grid_expt} 
In the gridworld domain, an agent must navigate within a two-dimensional world from an arbitrary starting point to a designated goal state. Two $50\times 50$ gridworlds we will consider are shown in Figure~\ref{fig:grid_transfer}, where grey blocks represent immovable obstacles (walls) and large grey circles denote terminal goal states. The actions available to the agent are {\texttt{up}, \texttt{down}, \texttt{left}, \texttt{right}}. The four movement actions are reversible, and succeed with probability 0.9. Actions that would otherwise allow the agent to step on or through an obstacle fail with probability 1, and the agent remains in place. For all worlds, the reward function is set to $-1$ for all states except the goal state, which is assigned a reward of $+10$.  We assume that these transition probabilities $P$ and rewards $R$ are given. 

Bottleneck detection and partitioning was done once, before compressing, to a maximum of 3 scales. We then compressed the gridworld problem 3 times using the initial local guess policies described in Section~\ref{sec:cluster_pols}. Figure~\ref{fig:grid_bns} shows, clockwise from top left, the detected bottlenecks (marked by `\texttt{x}' characters) and statespace adjacency graphs (from the compressed transition probability matrices) superimposed on the original world for each successively coarsened MDP. The graphs are directed, however for readability we do not show directionality in the plots. One can see that as the problem is repeatedly compressed, \clusterstext become successively lumped together into coarser approximations to the original world. A solution to the MDP at the first compressed level, for example, determines the optimal sequence of \clusterstext the agent should traverse to reach the goal. Figure~\ref{fig:grid_pols} shows policies resulting from solving the coarse MDPs, depicted as directed arrows marking a path along bottleneck states to the goal. For this problem, solutions to the coarse problems are compatible with the optimal fine scale policy.

In Figure~\ref{fig:grid_transfer} (right), we show a gridworld to which knowledge may be transferred given a solution at some scale to the problem on the left. In the transfer world (right) the optimal policy and state-transition behavior is significantly different from that of the original world (left). The reward function is also different, since the goal has moved, however the multiscale structure is similar, and the goal is in the same \clustertext as before (though the \clustertext has moved). The optimal fine scale policy within \clusterstext of similar geometry are also similar across problems. Indeed, for this world some or all of the solution at any of scales 1-4 may be transferred following the process discussed in Section~\ref{sec:transfer}. We will consider a simple transfer scenario in which the optimal fine scale policy is transferred wherever transferability detection (Section~\ref{sec:xfer-detect}) indicates it is advantageous to do so. Details discussing the application of transfer Algorithm~\ref{alg:transfer-general} are given below.\\

\noindent{\em Cluster Correspondence:} For this problem, the \clustertext correspondence algorithm described in Section~\ref{sec:xfer-cluster-correspond} correctly pairs together the \clusterstext in the source and destination problems. Figure~\ref{fig:grid-tr1-matching} shows the partitioning of each world into \clusterstext identified by the recursive spectral clustering step (Algorithm~\ref{alg:spectral_clustering}), as well as the correspondences returned by the \clustertext correspondence algorithm. Within each world, \clusterstext are demarcated by shade of color, and across worlds, gray lines connect the centroids of \clusterstext that have been paired together.\\

\noindent{\em Transfer Detection:} The transfer detection algorithm described in Section~\ref{sec:xfer-detect} was applied to each pair of matched \clusterstext. It is clear that with the exception of \clusterstext 1 and 8 in Figure~\ref{fig:grid-tr1-matching}, the \clusterstext are similarly oriented in both worlds. Thus for this problem, one should be able to skip statespace matching {\em within} \clusterstext and rely on transfer detection to confirm whether this was ultimately a safe thing to do. Omitting a statespace matching at the fine scale is equivalent to assuming that paired subproblems have the same orientation with respect to the problem domain and bottlenecks. In general of course it is hard to know a priori whether identified sub-problems share the same orientation as a pre-solved problem stored in a database of solutions, and statespace matching at {\em all} scales involved in the transfer should be performed. Nevertheless, one can still attempt to assume the orientations are correct, and then detect whether this assumption is valid or not. This approach may be particularly fruitful whenever the fine scale statespaces are large and complex, so that graph matching is difficult and error-prone. For the present  gridworld problem, the detection algorithm identifies \clusterstext\ $2-7$ as policy transfer candidates, and rejects \clusterstext 1 and 8. This result coincides with our earlier visual intuition from Figure~\ref{fig:grid-tr1-matching}.\\

\noindent{\em Fine Scale Policy Transfer:} Within \clusterstext\ $2-7$, the fine scale optimal policy for the source problem was mapped to the destination problem following the mapping procedure described in Section~\ref{sec:policy-xfer}. States in the destination world which did not receive a policy by transfer were  given a deterministic policy that always recommends the \texttt{up} action, rather than a uniform distribution over all actions.  Mapping the actions across worlds is easy in this case, since the action spaces are identical, and pieces of the optimal policy for the original problem largely transfer without modification. Comparing the two worlds in Figure~\ref{fig:grid-tr1-matching}, however, the detected bottleneck states are not always in the same place relative to a given  \clustertext. For example, the  two bottlenecks at the top-left corner of \clustertext 5 are one grid space to the right in the transfer world as compared to the original world. The policy transfer algorithm in~Section~\ref{sec:policy-xfer} maps a  policy between \clustertext interiors along the established correspondence, which in our case is simply an enumeration of the \clustertext's states in column-scan order, from top-left to bottom-right. For \clustertext 5, the difference in relative positioning of the bottlenecks creates a misalignment between the source and destination \clusterstext' interior states. For this particular problem, however, this misalignment imposes little error since the optimal policy is constant over large portions of the  \clustertext. This is likely the reason why \clustertext 5 was identified as a good candidate for transfer, despite alignment errors at the fine scale. In general, policy transfer may be relatively robust to correspondence errors at the fine scale since, by construction, the underlying Markov chain is fast mixing within \clusterstext.\\

\subfiglabelskip=0pt
\begin{figure}[p]
\centering
\subfigure[][]{
\includegraphics[scale=0.533]{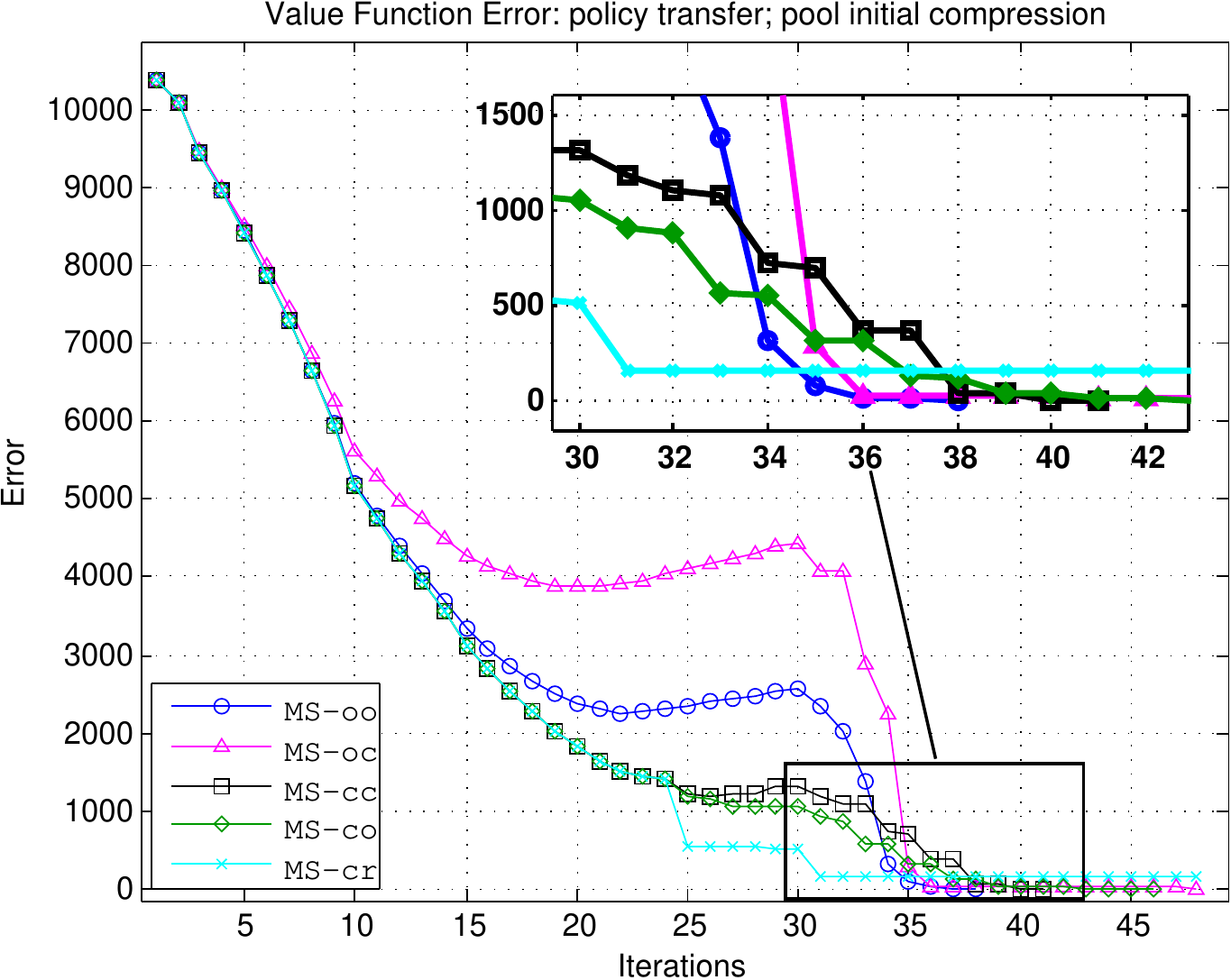}
\label{fig:gridworld-partial-transfer-pool}
}
\subfigure[][]{
\includegraphics[width=0.47\textwidth]{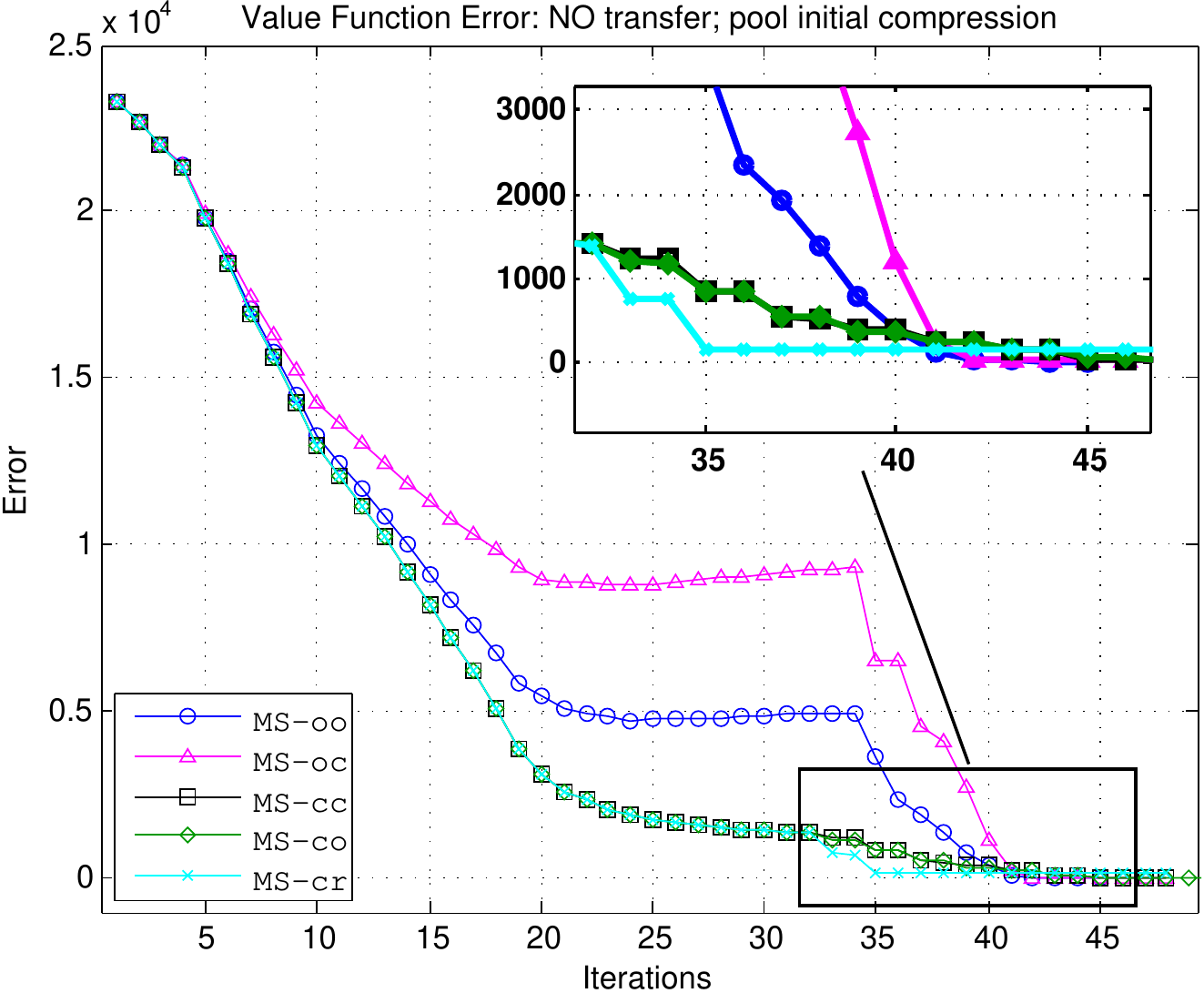}
\label{fig:gridworld-partial-notransfer-pool}
}
\\
\subfigure[][]{
\includegraphics[width=0.47\textwidth]{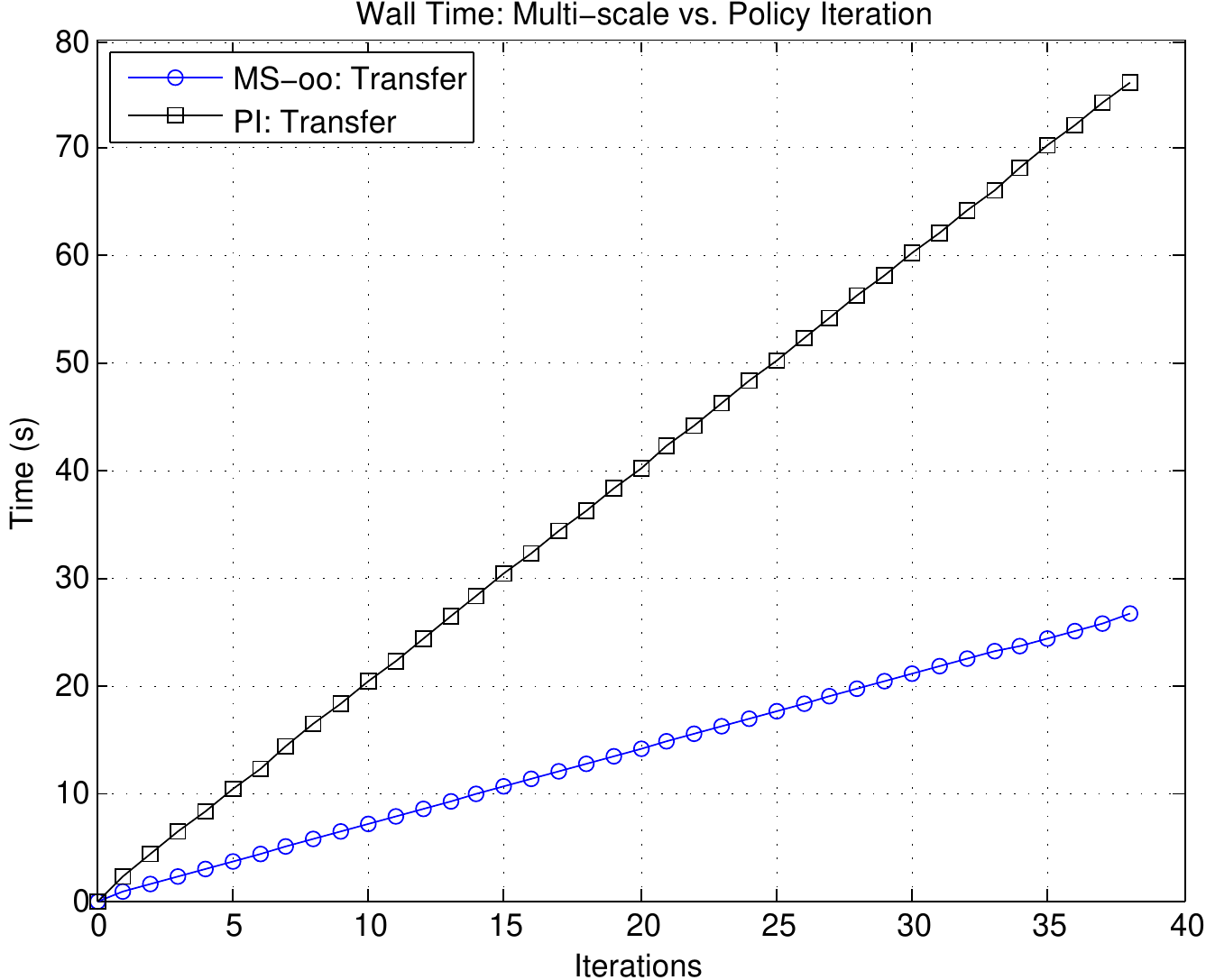}
\label{fig:gridworld-tr1-time}
}
\\
\subfigure[][]{
\includegraphics[width=0.47\textwidth]{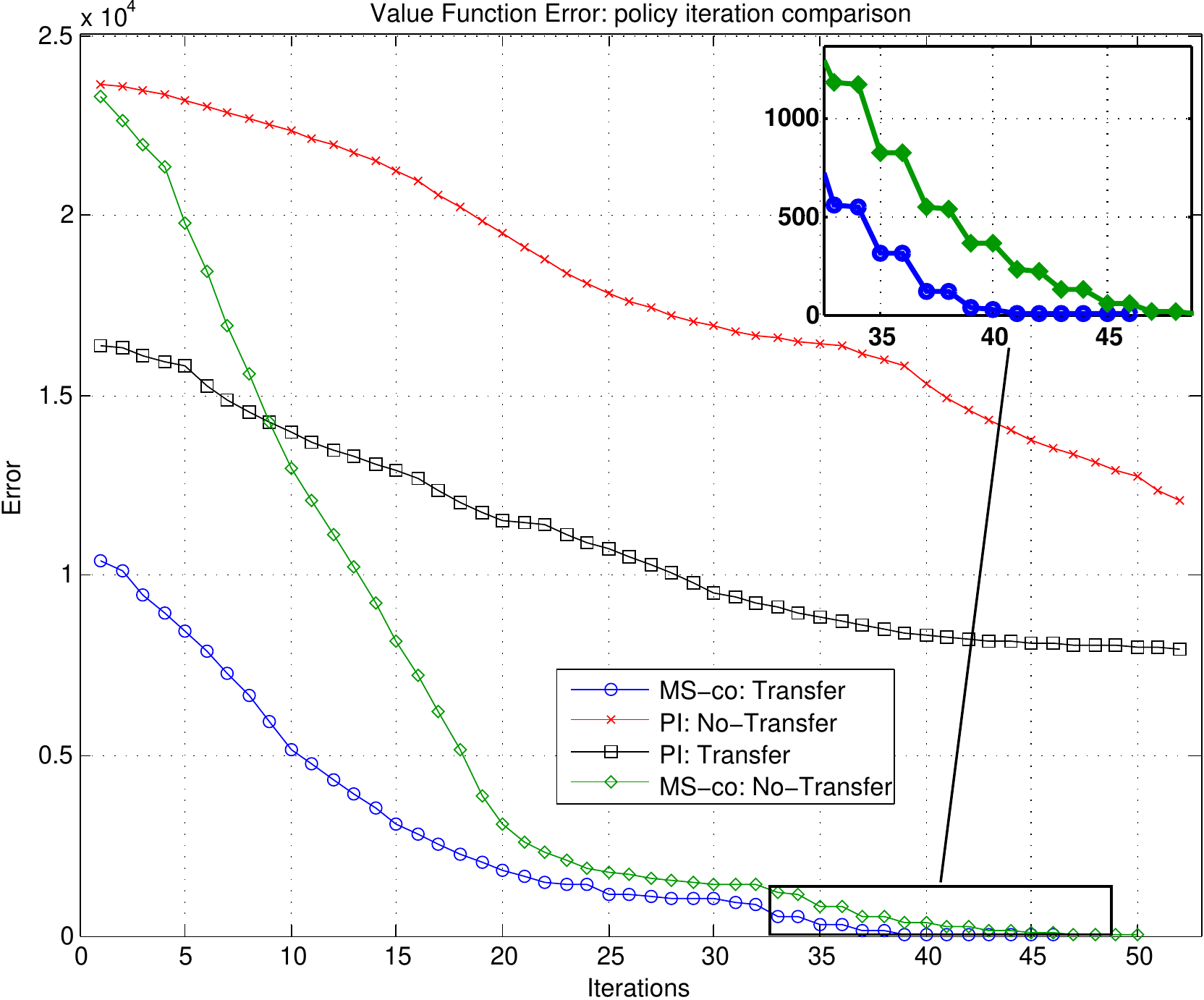}
\label{fig:gridworld-partial-compare}
}
\subfigure[][]{
\includegraphics[width=0.47\textwidth]{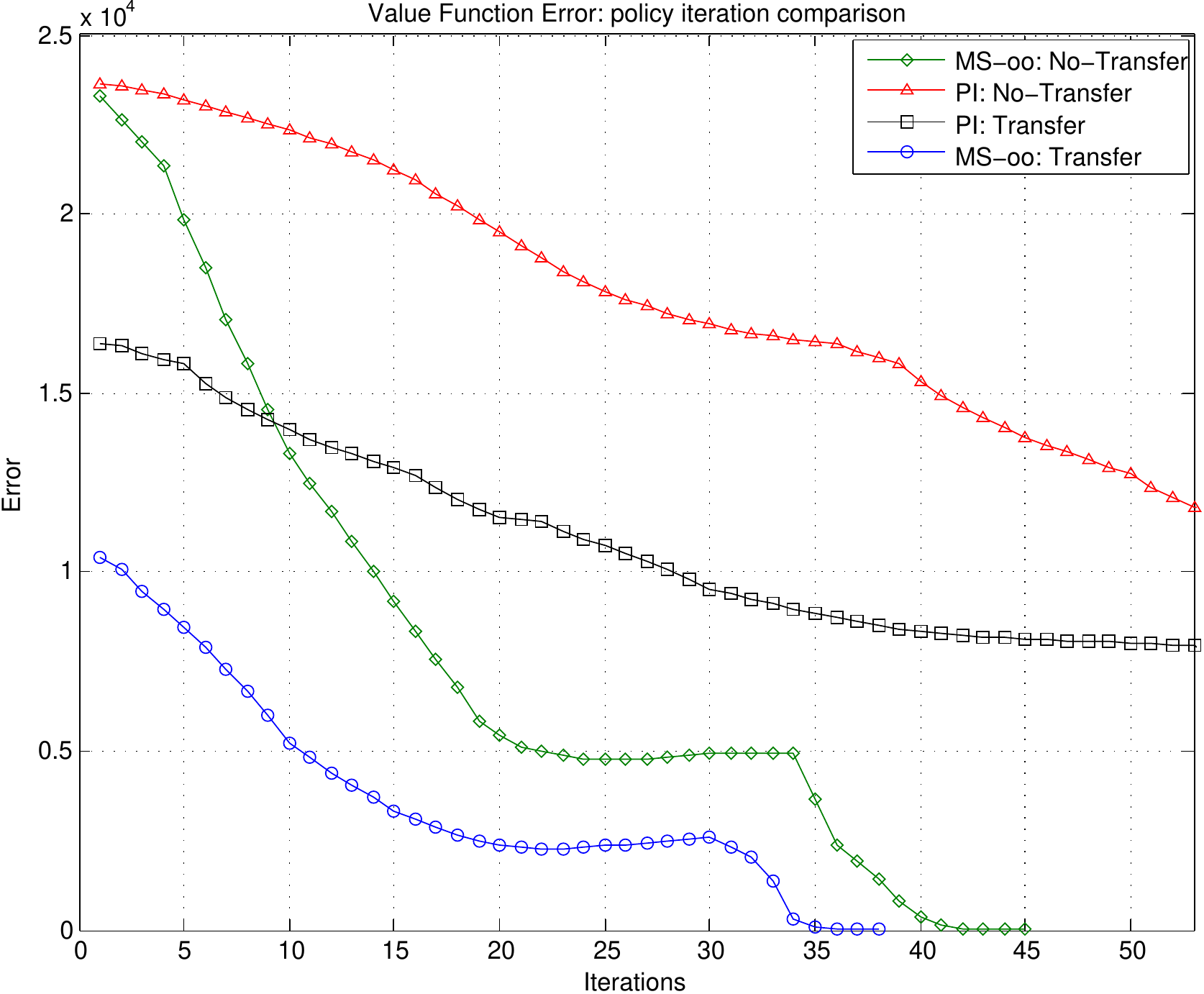}
\label{fig:gridworld-partial-compare-oo}
}
\caption{{\small\em Solution of the gridworld transfer problem described in Section~\ref{sec:grid_expt}: comparison among various multiscale algorithms, both with and without transfer, and comparison to policy iteration. See text for details.}}
\label{fig:gridworld-partial}
\end{figure}

\noindent{\em Transfer Problem Solution:} Several multiscale algorithms listed in Table~\ref{tab:expt-algs} were evaluated, with the transferred (fine scale) policy serving as the initial policy for local policy iteration in the destination problem. In all cases, the initial coarse scale value function was obtained by solving the coarse MDP given by compression with respect to the pool of local policy guesses discussed in Section~\ref{sec:cluster_pols}, and in all experiments, the blending parameter appearing in policy updates  was set to $\lambda=1$, thereby imposing a greedy policy updating convention. 

Plots in Figure~\ref{fig:gridworld-partial} with $y$-axis labeled ``Error'' show the Euclidean
distance between the value function after $t$ iterations ($x$-axis) of the given algorithm and the optimal value function for this problem\footnote{In this experiment and those that follow, we will use the $L_2$ norm to measure error rather than $L_{\infty}$, as it is a more revealing indicator of  progress over the entire statespace in question.}.  Inset plots detail boxed regions. See Table~\ref{tab:expt-algs} and surrounding discussion for a description of the algorithms and their labels. The default, fine-scale initial policy for experiments without transfer is an arbitrary deterministic policy that always chooses the \texttt{up} action.

Figures~\ref{fig:gridworld-partial-transfer-pool} and~\ref{fig:gridworld-partial-notransfer-pool}  show the performance of various multiscale algorithms with and without fine scale policy transfer, respectively.  Comparing the scale of the vertical axes across the two plots,  transfer provides a good warm start for all algorithms. Because the initial coarse value function solves an MDP compressed with respect to a pool of initial policy guesses, we may expect that the initial coarse value function is trustworthy. For gridworld-type problems in particular, this is a reasonable expectation. Comparing curves within the figures, it is clear that the coarse initial data is indeed good: algorithms which leverage the initial condition as far as possible, and iterate inside \clusterstext to convergence before updating the boundary  (``MS-\{\texttt{cc,co,cr}\}'' traces), perform better than the algorithms that only iterate over interiors once  between boundary updates (``\{MS-\texttt{oo,oc}\}'' traces). Algorithms updating the bottlenecks by recompression (``MS-\texttt{cr}'' traces) are seen to  converge to a suboptimal value function, however the corresponding policy is optimal after iteration 29 in Figure~\ref{fig:gridworld-partial-transfer-pool} and after iteration 33 in Figure~\ref{fig:gridworld-partial-notransfer-pool}. MS-\texttt{cr} is the single best algorithm for solving the gridworld problem, both with and without transfer, and MS-\texttt{co} is the best algorithm not involving recompression.

%

Because we have only considered fine-scale policy transfer, we can compare to the performance of the canonical policy iteration algorithm given the transferred policy as the initial condition\footnote{Note that in general if there is transfer at coarse scales, then such a comparison is not possible since the policy iteration algorithm cannot directly take advantage of coarse information.}. Figures~\ref{fig:gridworld-partial-compare} and~\ref{fig:gridworld-partial-compare-oo}  respectively compare policy iteration with and without transfer to multiscale algorithms \texttt{co} and \texttt{oo}, with and without transfer. The multiscale algorithm curves in these figures are the same as the corresponding curves in Figures~\ref{fig:gridworld-partial-transfer-pool} and~\ref{fig:gridworld-partial-notransfer-pool}. Comparing curves within plots, it is clear that the transferred policy also provides policy iteration with a helpful warm start. However, considering the difference in starting error between multiscale transfer/no-transfer curves to those of policy iteration, the multiscale algorithms are better able to take advantage of the transferred information. The improvement of the multiscale no-transfer curve over the policy iteration no-transfer curve reflects the improvement due to both coarse information as well as the multiscale approach. The two multiscale algorithms converge to optimality in about the same number of iterations, however for this domain MS-\texttt{co} exhibits stronger non-monotonicity.

As discussed in Section~\ref{sec:mdp_solution}, there are two primary reasons why the multiscale algorithm can perform better than policy iteration: The multiscale algorithm starts with coarse knowledge of the fine solution given by the solution to the compressed MDP, and the multiscale approach can offer faster convergence since convergence of local (\clustertext) policy iteration is constrained by faster mixing times within \clusterstext rather than slow times across \clusterstext.
These are reasons why Algorithm~\ref{alg:hierarchy-solve} can converge in fewer iterations. But an iteration-count comparison to vanilla policy iteration is not entirely fair because each iteration of the multiscale algorithm is significantly cheaper, as described in Section~\ref{sec:ms-alg-complexity}. Figure~\ref{fig:gridworld-tr1-time} shows elapsed wall time after $t$ iterations for policy iteration and MS-\texttt{oo} algorithms. Experiments were conducted on a dual Intel Xeon E5320 machine running 64-bit MATLAB release 2010b under Linux, with no parallelization or  optimization beyond the default BLAS/Atlas multicore routines embedded into native MATLAB linear algebra calls. It is evident that the multiscale algorithm compares favorably to policy iteration, both in terms of total iterations and per-iteration scaling in time. We note that the ratio of the slopes in Figure~\ref{fig:gridworld-tr1-time} is specific to this example. In general this ratio is determined by the cardinality of the statespace and the number of identified \clusterstext (statespace graph cuts), and will vary across problems.




\subsection{Continuous Two-Task Cart-Pendulum Domain}
\label{sec:pend-expt}
The cart-pendulum problem is a classic continuous control task where an inverted pendulum attached to a cart on a track must be balanced by applying force to the cart. We consider a slightly more complex domain in which there are two additions: the cart must be moved to a particular goal location, and at some portions of the track the pendulum is held fixed and does not need to be balanced, while in other regions the pendulum is free and must be balanced. In all simulations, the length of the pendulum is $0.5$m, its mass is $1$kg, and the mass of the cart is $5$kg. There are three actions, corresponding to applying horizontal forces of $-20$N, $0$N, or $+20$N to the cart. These control inputs are subsequently corrupted at each time step by i.i.d. additive, zero-mean Gaussian noise with standard deviation $\sigma=5$. Three state variables are used, $\{\theta, \dot{\theta}, x\}$, where $\theta$ is the angle of the pendulum from the vertical, $\dot{\theta}=\tfrac{d\theta}{dt}$, and $x$ is the horizontal position along a track spanning the interval $[-30\text{m},30\text{m}]$. If the pendulum falls over ($|\theta| = \pi/2$) or the end of the track is reached ($|x|=30$), a reward of $-1$ is received, and the simulation ends. Unless otherwise noted below, at any other state a reward of $0$ is received. Within this domain we will consider two different tasks:

\begin{description}
\item[{\normalfont Default ($\problemone$):}]
The goal of the default task is to move the cart to position $x=+20$ along the track, whereupon a reward of +100 is received, and the simulation ends. If the cart is at any position $x>0$, the pendulum is held fixed in the upright position $(\theta=0,\dot{\theta}=0)$, but the cart is free to move. Otherwise, if $x\leq 0$ the pendulum is able to move freely as usual and must be balanced. If a simulation is started at some initial position $x_0 < 0$, then two sub-tasks must be solved in order: (1) The pendulum must be balanced while moving right until reaching $x=0$ (``\texttt{balance}''), and (2) the pendulum is held upright but must be carried while moving right towards $x=20$ (``\texttt{carry}'').
\item[{\normalfont Transfer ($\problemtwo$):}] 
The goal of the transfer task is the same: the cart must be moved to the position $x=+20$ along the track, whereupon a reward of $+100$ is received, and the simulation ends. The regions where the pendulum must be carried vs. balanced are swapped, however. If the cart is at any position $x<0$, the pendulum is held fixed at $(\theta=0,\dot{\theta}=0)$. Otherwise, if $x\geq 0$ the pendulum is able to move freely and must be balanced. Thus, for simulations starting at some initial position $x_0 < 0$,  the two sub-tasks that must solved occur in the opposite order (\texttt{carry},\texttt{balance}): (1) Carry the pendulum while moving right until reaching $x=0$, and then (2) balance the pendulum while moving right.
\end{description}

The goal of transfer for this domain is to convey the ability to carry or balance the pendulum. The agent must still learn when to apply these skills, and in which order. Although this particular pair of problems involves only two sub-tasks, it serves to illustrate how transfer within the multiscale framework we have described can be achieved in the context of a continuous control problem. Furthermore, in contrast to the other example domains, the domain described here has a multiscale structure induced by changes in the {\em intrinsic dimension} of the statespace; diffusion geometry is less helpful here. In light of these differences, we will consider alternative choices for the statespace partitioning and graph matching below.\\

\subfiglabelskip=0pt
\begin{figure}[p]
\centering
\subfigure[][]{
\includegraphics[width=0.8\textwidth]{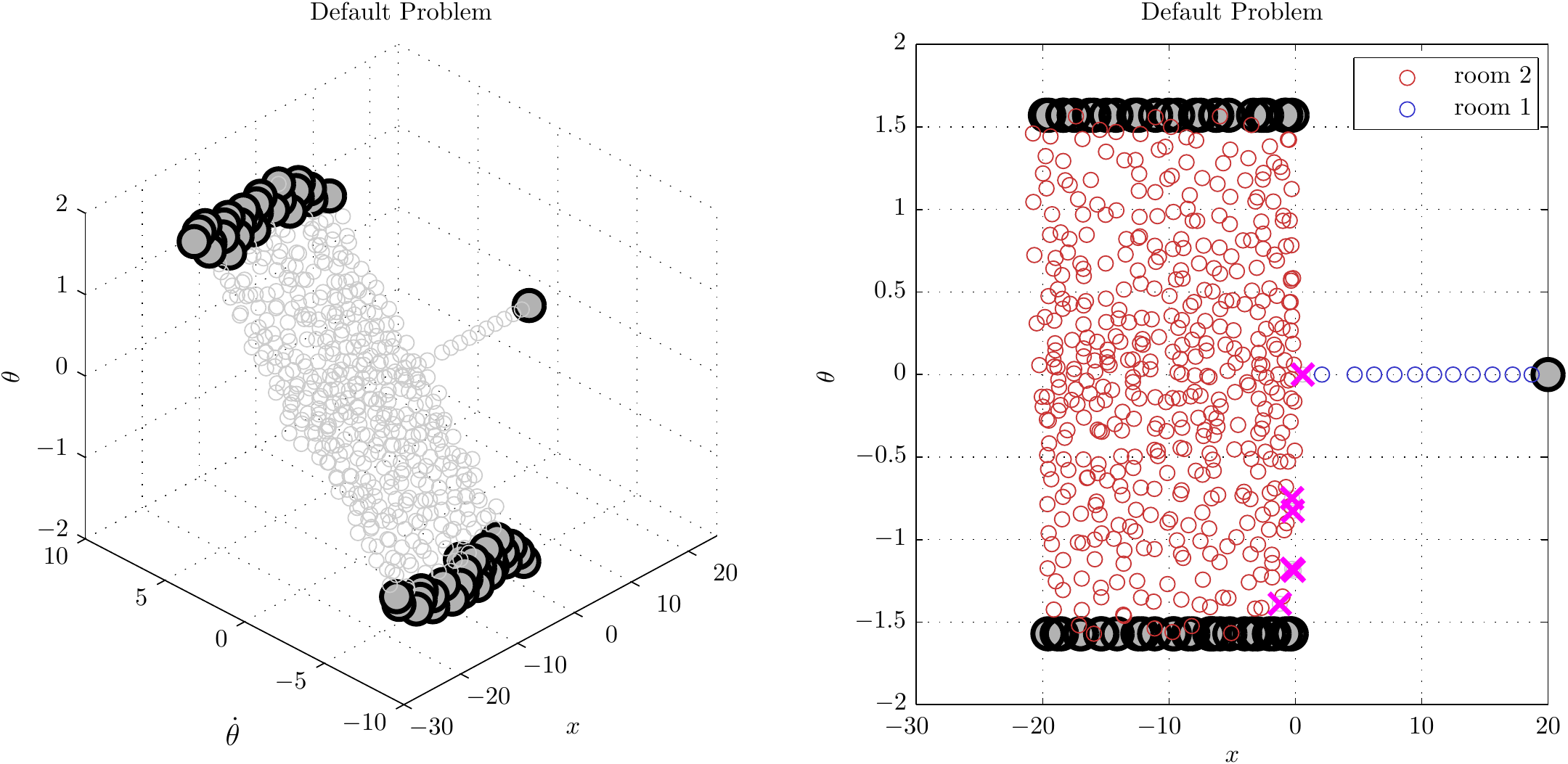}
\label{fig:pend-statespace-default}
}
\\
\subfigure[][]{
\includegraphics[width=0.8\textwidth]{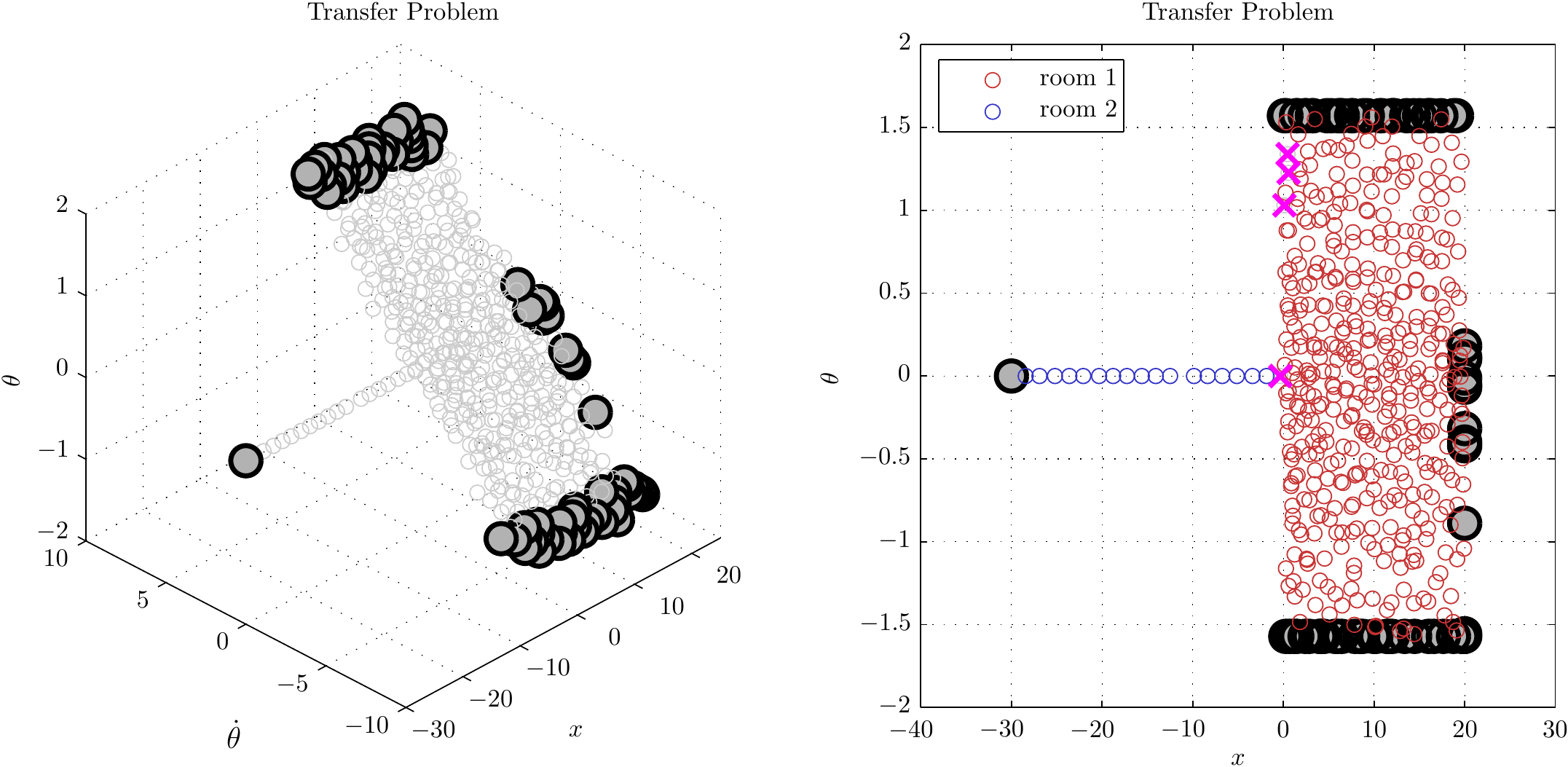}
\label{fig:pend-statespace-transfer}
}
\caption{{\small\em Discretized statespaces for the two different pendulum control problems described in Section~\ref{sec:pend-expt}. A fine scale policy is transferred from the default problem having the  statespace shown in~\subref{fig:pend-statespace-default}, to the transfer problem with statespace in~\subref{fig:pend-statespace-transfer}. Large circles indicate terminal states, while {\normalfont \texttt{x}} indicates non-absorbing bottleneck states. See text for details.}}
\label{fig:pendulum-statespace}
\end{figure}

\noindent{\em Simulation and Statespace Discretization:}
For both problems above, fine scale MDPs were estimated from Monte-Carlo simulations.  For each problem, we simulated the systems with the diffusion (uniform random) policy for 8000 episodes. The initial state\footnote{Our simulator uses the additional state variable $\dot{x}=\tfrac{dx}{dt}$ internally, however this variable is ignored externally (that is, during clustering, policy evaluation, etc.).} for each episode was drawn according to $\theta,\dot{\theta}\sim\mathcal{U}[-0.2,0.2], x\sim\mathcal{U}[-20,20], \dot{x}\sim\mathcal{U}[-0.1,0.1]$, where $\mathcal{U}[a,b]$ denotes the uniform distribution on the interval $[a,b]$. The simulation stepsize was set to $\Delta t=0.1s$, giving a 10Hz control input.  

The resulting samples were normalized to have equal range in each coordinate, and clustered according to a $\delta$-net with $\delta$ chosen so that we obtained approximately 500 representative states from the pool of samples. These 500 states determined the discrete statespace on which the given MDP was defined. Absorbing states reached during the simulation were separately clustered and the resulting cluster representatives added to the previous 500 states. This ensured that the terminal boundaries of the problem were clearly represented in the discretized statespace. The above procedure was applied separately to each problem, and the final size of the statespaces were 502 and 515, of which 56 and 75 were terminal, for the default and transfer problems respectively. Figures~\ref{fig:pend-statespace-default} and~\ref{fig:pend-statespace-transfer} show plots of these discretized statespaces. Within each figure, the left-hand plot shows states as circles in $(x,\theta,\dot{\theta})$ coordinates, with large dark circles indicating terminal states. Right-hand plots graph  states in $(x,\theta)$ coordinates.\\

\noindent{\em MMDP Construction:} Given the simulation samples and clusters defined above, transition and reward statistics between the $\delta$ regions claimed by the representative states were computed to estimate $P(s,a,s')$ and $R(s,a,s')$ for the respective fine scale MDPs. Absorbing boundaries were enforced by forcing MDP states which had absorbing samples as their closest neighbor (out of all samples) to be absorbing. If any states were subsequently rendered unreachable, they were removed from the problem.  Rewards for terminal states were similarly enforced by imposing the reward received at the  neighbors nearest to the designated absorbing MDP states. Collectively these steps ensured that the absorbing rewards and states -- the boundary conditions -- were sufficiently captured in the translation from a continuous to a discrete problem.

To define a coarse scale MDP, we next partitioned the statespace into \clusterstext\ and identified bottlenecks. For the problems described above, however, geometry is not as helpful, and we chose to pursue an approach different from the spectral clustering algorithm described in Section~\ref{sec:clustering}. Here, there is a natural partitioning of the statespace based on intrinsic  dimension; balancing is a 3D task, while carrying (without balancing) is a 1D task. Thus, we detected where dimension changes in the statespace take place, and partitioned accordingly. The right-hand plots in Figures~\ref{fig:pend-statespace-default} and~\ref{fig:pend-statespace-transfer} respectively show the result of this partitioning for the default and transfer problems. States are colored according to membership in one of two resulting \clusterstext. One can see that the \clusterstext clearly correspond to sub-tasks \texttt{carry} and \texttt{balance}. Terminal states are marked by large gray circles, and non-absorbing bottlenecks are indicated with (magenta) \texttt{x}'s. As would be expected, in both problems the state just at the interface of the two \clusterstext is a bottleneck. Some additional bottlenecks also result from the partitioning.

On the basis of the \clusterstext and bottlenecks shown in Figures~\ref{fig:pend-statespace-default} and~\ref{fig:pend-statespace-transfer}, the transfer problem was compressed once with respect to the diffusion policy\footnote{In this simple example, we illustrate the core ideas by carrying out only fine scale policy transfer into the finest scale of a hierarchy; thus, only the destination problem needs to be compressed. For general transfer into arbitrary levels of a hierarchy, compression of both problems would be required.}, assuming a fine scale uniform discount rate $\gamma=0.99$.\\

\noindent{\em Policy Transfer:}
We next established \clustertext and fine scale statespace graph correspondences in order to transfer the fine scale policy from \problemone to \problemtwo. The action sets across problems are already in correspondence, so action mapping was not pursued\footnote{In more complex problems where this may not be as obvious, the procedure described in Section~\ref{sec:policy-xfer} may be applied.}. Since the statespaces partition into \clusterstext on the basis of dimension, we simply matched \clusterstext  having the same  dimension across problems. In this way, the \texttt{carry} and \texttt{balance} subtasks were easily placed into correspondence (respectively).

To construct a fine scale graph matching, we assumed that the statespace coordinate systems were already aligned across problems (that is, we assumed we knew which coordinate in \problemone corresponds to the coordinate for $x$ in \problemtwo, and similarly for $\theta$ and $\dot{\theta}$). Letting $S_i$ denote the statespace of $\problem{i}$, we then considered a state mapping $\phi:S_2\to S_1$ of the form
\[
\phi: (x,\theta,\dot{\theta})\mapsto \bigl(f(x),g(\theta),h(\dot{\theta})\bigr) .
\]
The \clustertext correspondence previously established induces a natural mapping between $x$ coordinates; for instance, we simply mapped the $x$ interval $(-30,0]$ in \problemtwo onto the interval $[0,20)$ for \problemone to define $f$ on states within the \texttt{carry} \clustertext of \problemone. A similar mapping was constructed to define $f$ on states within the \texttt{balance} \clustertext.  The coordinate maps $g,h$ were taken to be the identity, since $\theta,\dot{\theta}$ are directly comparable across problems.  A statespace correspondence $\eta$ was then defined based on the nearest-neighbor Euclidean distance under $\phi$, assuming a neighbor search constrained to fall within matching \clusterstext. If $\clustertwo\in\problemtwo$ has been matched to $\clusterone\in\problemone$, then
\[
\eta(s) = \arg\min_{s'\in \clusterone} \|\mathsf{N}(s')-\mathsf{N}(\phi(s))\|_2,\qquad s\in\clustertwo,
\]
where $\mathsf{N}(\cdot)$ is the same coordinate-wise range normalization used in the state clustering steps above for \problemone. Given the fine scale state mapping $\eta$, the optimal fine policy for the default problem \problemone was transferred to \problemtwo by transferring separately within matched \clusterstext (that is, between matched sub-tasks). \\

\subfiglabelskip=0pt
\begin{figure}[p]
\centering
\subfigure[][]{
\includegraphics[width=0.47\textwidth]{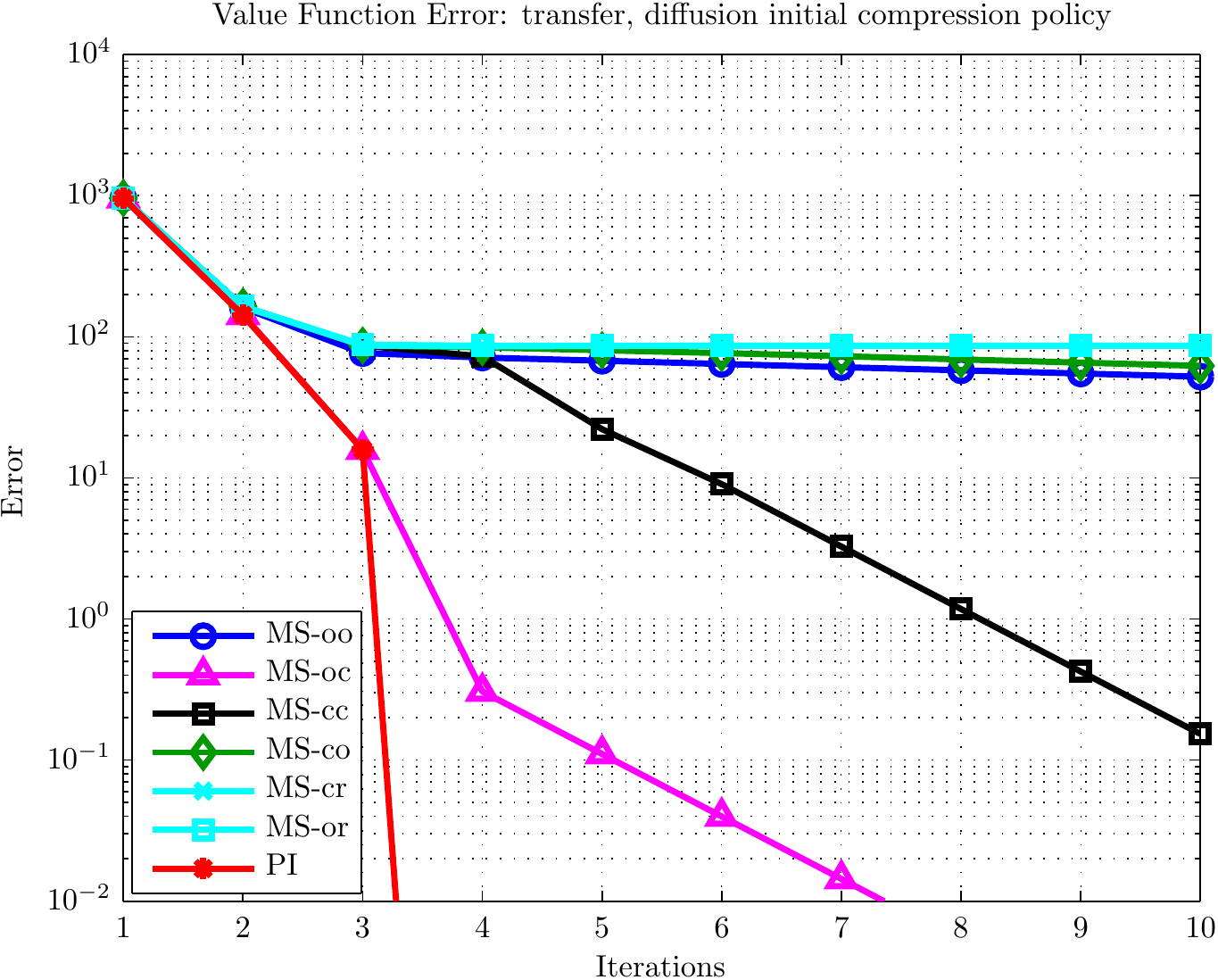}
\label{fig:pend-transfer}
}
\subfigure[][]{
\includegraphics[width=0.47\textwidth]{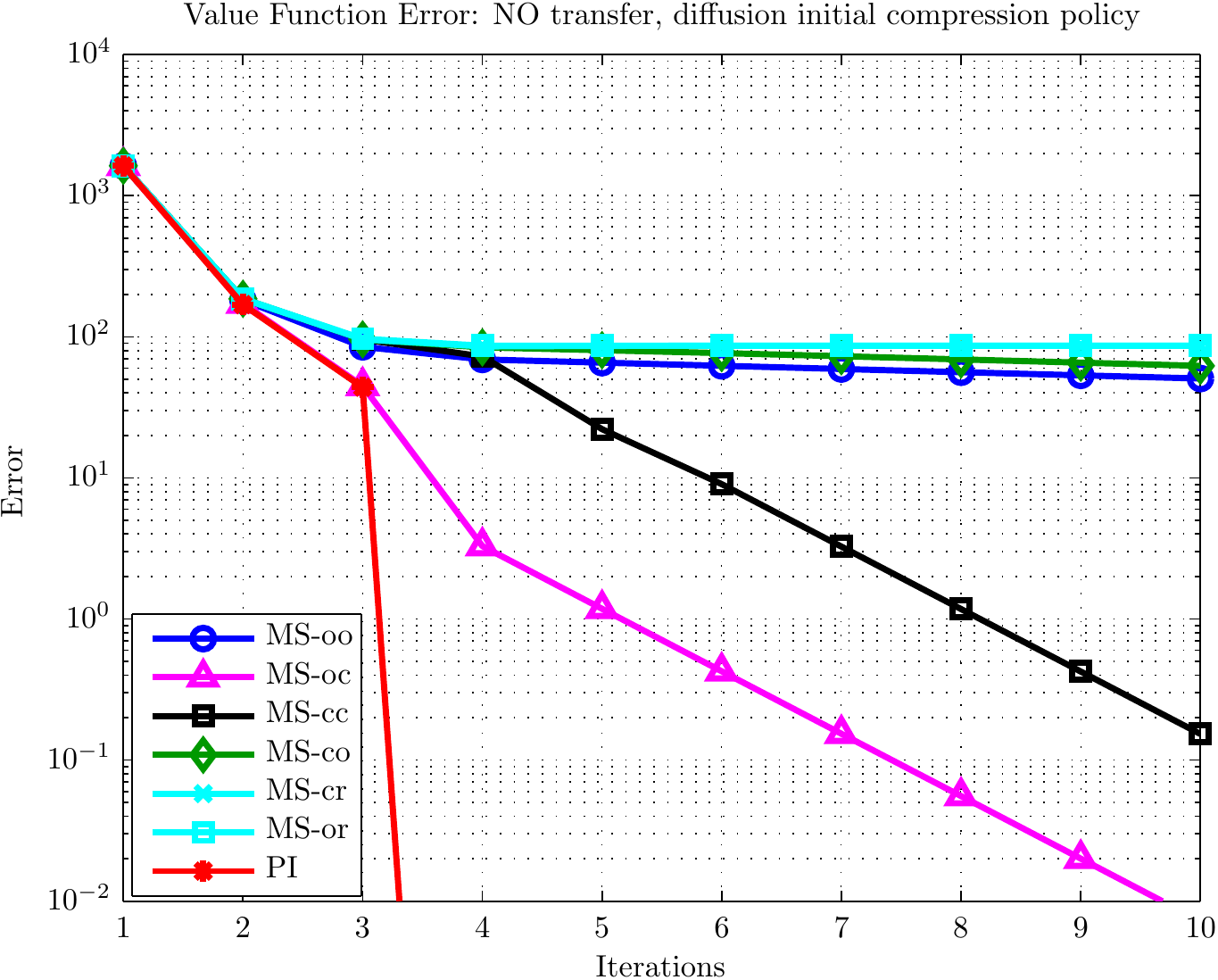}
\label{fig:pend-notransfer}
}
\subfigure[][]{
\includegraphics[width=0.47\textwidth]{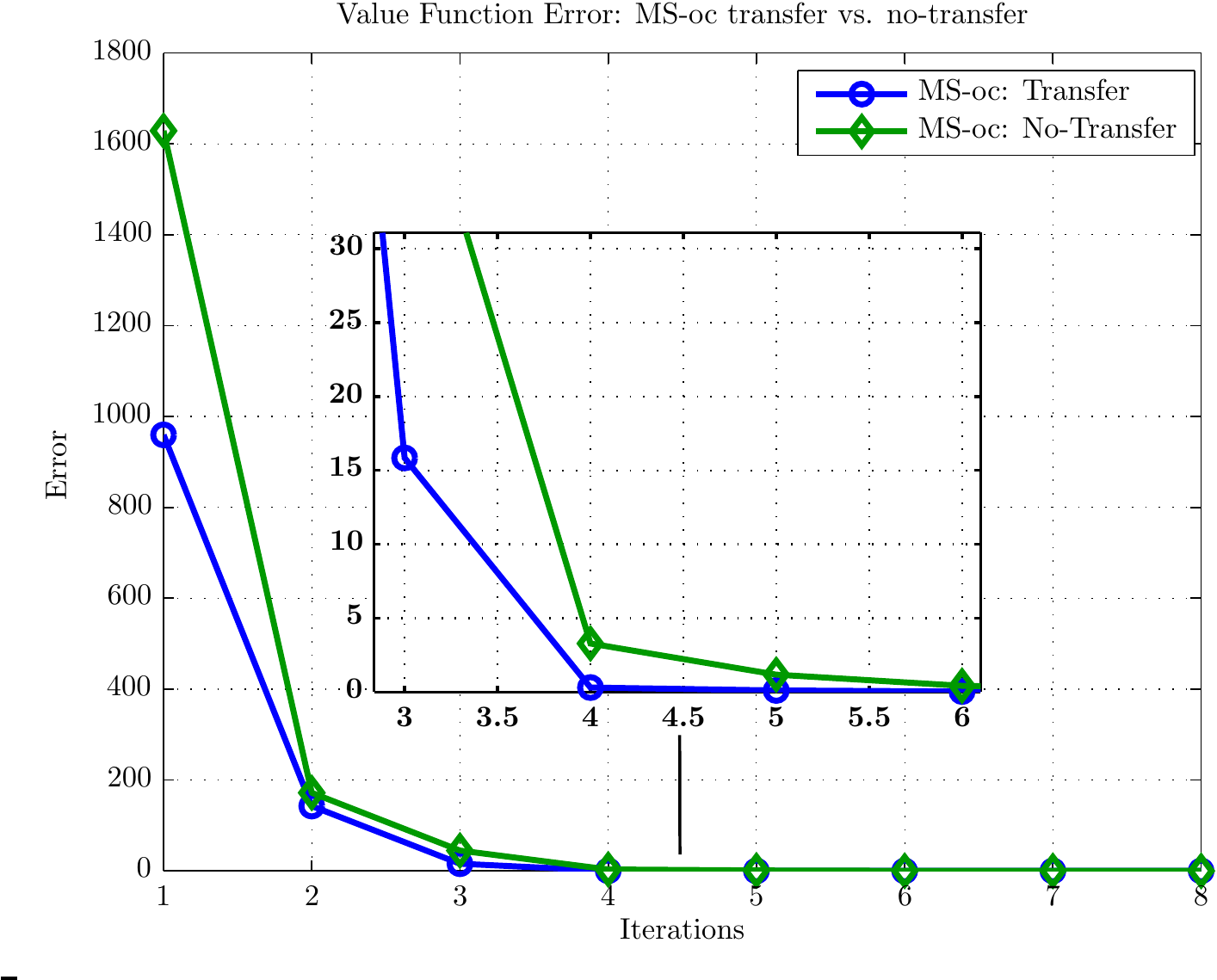}
\label{fig:pend-compare-oc}
}
\subfigure[][]{
\includegraphics[width=0.47\textwidth]{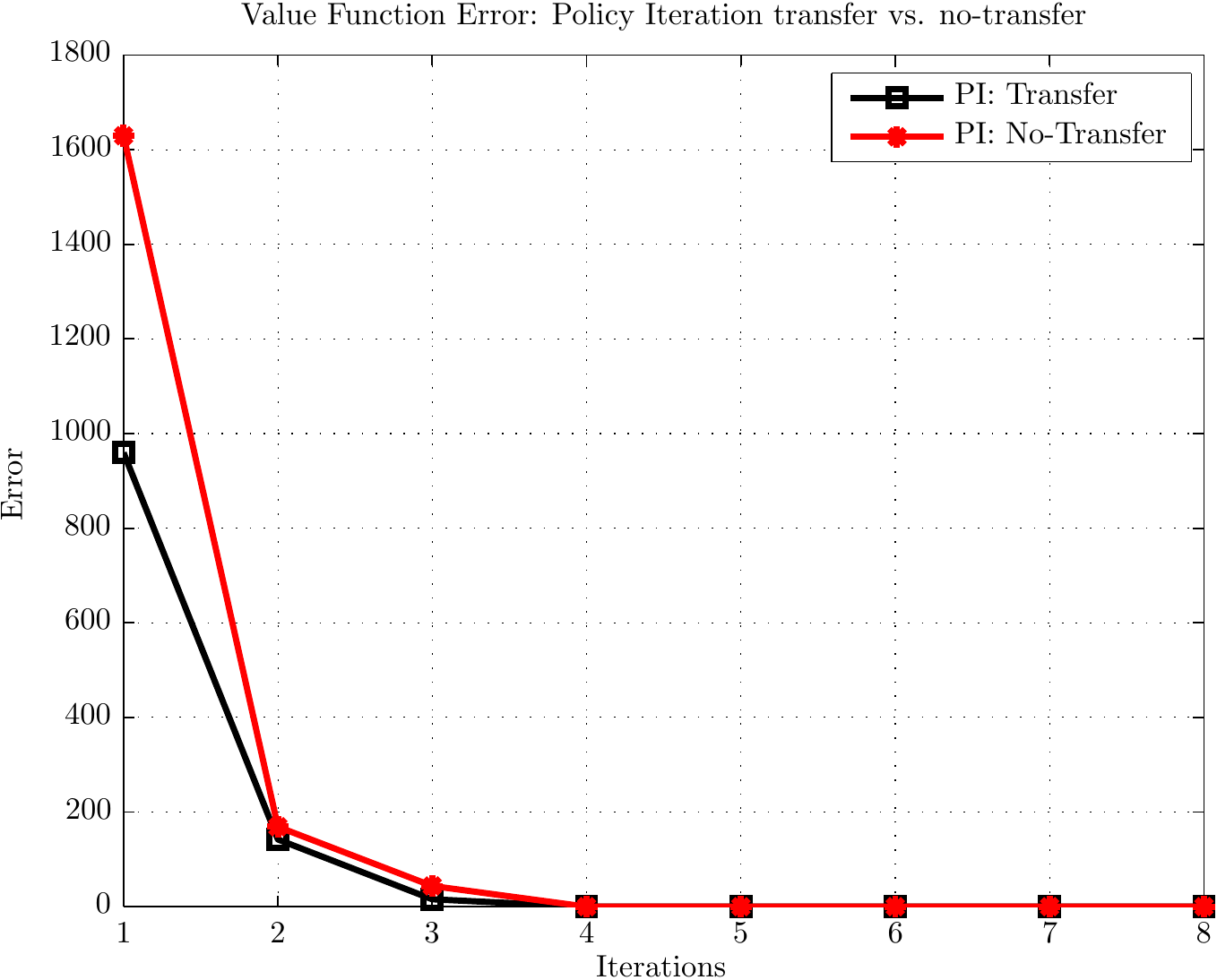}
\label{fig:pend-compare-pi}
}
\caption{{\small\em Solution of the pendulum transfer problem described in Section~\ref{sec:pend-expt}: Comparison among multiscale algorithms both with and without fine scale policy transfer, and comparison to policy iteration. See text for details.}}
\label{fig:pendulum-iters}
\end{figure}

\noindent{\em Transfer Problem Solution:}
We applied several multiscale algorithms listed in Table~\ref{tab:expt-algs} (see surrounding text for details describing these algorithms and notation) to explore the impact of the transferred policy.  Experiments with transfer and without are compared, and we compare to ordinary policy iteration. Where transfer was considered, the initial fine scale policy was the transferred policy. The default, fine-scale initial policy for experiments without transfer is an arbitrary deterministic policy that always chooses the action which applies no force to the cart. In all multiscale experiments, the initial coarse value function was the value function solving a coarse MDP resulting from compression with respect to the diffusion policy. The blending parameter appearing in policy updates was set to $\lambda=1$.

In Figure~\ref{fig:pendulum-iters} we show plots detailing the Euclidean distance ($y$-axis, ``Error'') between the value function after $t$ iterations ($x$-axis) of the given algorithm and the optimal fine value function for \problemtwo. Figures~\ref{fig:pend-transfer} and~\ref{fig:pend-notransfer}  show the performance of the  multiscale algorithms as well as policy iteration (``PI''), respectively with and without fine scale policy transfer. Here, error is plotted on a logarithmic scale. Comparing these two plots at $t=1$ iteration, transfer provides a warm start for all algorithms. Furthermore, the improvement seen in Figure~\ref{fig:pend-transfer} can be entirely attributed to transfer, since the multiscale algorithms converge as fast as or slower than policy iteration, and do not contribute any improvement in convergence rate in and of themselves for this problem. (Of course, the complexity of each iteration is still much smaller for the multiscale algorithms compared to global algorithms such as policy iteration.) 

From Figures~\ref{fig:pend-transfer} and~\ref{fig:pend-notransfer} it is also evident that \texttt{MS-oc} is the single best multiscale algorithm, and has nearly the same rate of convergence as policy iteration (we note that after $t=3$ iterations the relative error has decreased to $<1\%$, and the problem is nearly solved). Figure~\ref{fig:pend-compare-oc} gives a more detailed view of the improvement transfer confers in the context of \texttt{MS-oc}. Error is plotted on a linear scale to make the difference more visible. Figure~\ref{fig:pend-compare-pi} compares policy iteration with and without the transferred policy as the starting point. Comparing Figure~\ref{fig:pend-compare-oc} to~\ref{fig:pend-compare-pi}, the improvement due to transfer as well as the convergence rates are seen to be similar. 

Algorithms involving recompression, \texttt{MS-cr} and \texttt{MS-or}, do not converge to optimality for this problem, and \texttt{MS-oo,MS-co} converge very slowly. That the \texttt{MS-oo,MS-co} algorithms converge slowly relative to \texttt{MS-oc} confirms the importance of the bottleneck near the origin, and by extension, the updates at this bottleneck. The fact that \texttt{MS-cc} converges more slowly than \texttt{MS-oc} suggests that either of the coarse or fine initial data contain some errors. Forcing too much of the initial coarse value function or fine scale policy results in poorer performance here. For this domain, however, we found that compression with respect to a pool of policies (Section~\ref{sec:cluster_pols}, simulations not shown) does not yield an initial coarse value function giving any better performance than the coarse value function derived from compression with respect to the diffusion policy. This is likely due to the fact that there are few non-absorbing bottlenecks, and only the bottleneck near the origin evidently plays a key role; the gradients obtained from either coarse value function contain comparable information in this case.

\subsection{Playroom Domain}
\label{sec:expts-playroom}
We consider a simplified version of the playroom domain introduced in~\cite{Singh2004,Barto:ICDL:04}. In our formulation, an agent interacts with four objects in a room: a ball, a bell, a music button and a light switch. The actions available to the agent are:
\begin{enumerate}\itemsep 0pt
\item Look at a randomly selected object. (Succeeds with probability 1).
\item Place a marker on the object the agent is looking at.
\item Press the music button.
\item Kick the ball towards the marker.
\item Flip the light switch.
\end{enumerate}
All actions except the first succeed with probability 0.75. In order to take the latter three actions, the agent must first be looking at the relevant object unless otherwise noted (see modifications below). If the ball is kicked into the bell, the bell rings for exactly one time period. If the light switch is flipped to the on position, the light stays on for exactly one time period, and then switches to the off state. The state is 5-dimensional, and consists of the following variables:
\begin{enumerate}\itemsep 0pt
\item Object the agent is looking at.
\item Object the marker is currently placed on.
\item Music on/off.
\item Bell on/off.
\item Light on/off.
\end{enumerate}

We will consider two pairs of problems within the playroom domain to illustrate compression and transfer. Each pair consists of a baseline task and a variation task. For a pair of tasks, the rules governing the environment will remain fixed, however the goal of the tasks will change. The objective is thus to apply knowledge gained from solving one problem in the environment towards solving another problem in the same environment.

To build MDP models, the tasks were independently simulated for 1000 episodes of maximum length 1000 actions. Each trial was ended upon reaching either the goal state, or the maximum number of actions. Since the statespaces are small, samples were simply binned according to the underlying state variables. Transition probabilities were then estimated empirically from the samples. Rewards were set to $+10$ for transitions to the goal state, and to $-1$ for all other transitions. In all cases, we fixed the discount parameter to $\gamma=0.96$. Next, the spectral clustering procedure described in Section~\ref{sec:clustering} was applied, stopping
after a single iteration. Note that one iteration of Algorithm~\ref{alg:spectral_clustering} can potentially result in more than two parts, since a single cut can produce multiple disconnected subgraphs. We next compressed the tasks once, assuming the diffusion policy at the fine scale. For the baseline tasks from which information is transferred, the MDP hierarchies were solved to the optimal solution following Algorithm~\ref{alg:hierarchy-solve}.

\subsubsection{Coarse Transfer Example}
\label{sec:playroom-simple}
In this example we illustrate the transfer of a coarse scale potential operator from one task to another. We will refer to the problem supplying the potential operator as the {\em default} problem, and the problem into which we transfer information as the {\em transfer} problem. The default problem is assumed to be solved, in the sense that coarse MDPs have been compressed with respect to optimal  policies, and the policy used to define the potential operator is optimal. 

For the pair of playroom problems we will explore in this section, the light is turned on for one period by taking the ``flip the light switch'' action with the marker on the bell, while looking at the bell. These are the same conditions for ringing the bell, only now the agent can alternatively flip the light on. The tasks are as follows:
\begin{description}
\item[{\normalfont Default ($\problemone$):}]
The goal of the default task is to cause the bell to ring while the music is playing. Note that there is no action that will directly ring the bell. The agent must: look at the music button, press the music button, look at the bell, place the marker on the bell, look at the ball, and finally, kick the ball into the bell. Because the bell only rings for a single period, the music must be turned on before ringing the bell. In this task, {\em the light switch does not play a role}. Each episode begins with the agent looking at a random object, the marker on a random object, and all on/off objects in the off state.
\item[{\normalfont Transfer ($\problemtwo$):}] The goal is to flip the light switch to the on position while the music is playing. However, to reach the goal state the agent must still have the marker on the bell, and must be looking at the ball in order to take the ``flip light switch'' action. That is, the role of the light switch and ball kicking actions in solving the problem have been swapped. The agent must: look at the music button, turn on the music, look at the bell, place the marker on the bell, look at the ball, flip the light switch.
\end{description}
The difference between the default and transfer tasks is that the final action leading to the goal state has been switched and the underlying goal state itself has changed.

\begin{figure}[t]
\centering
\includegraphics[width=\textwidth]{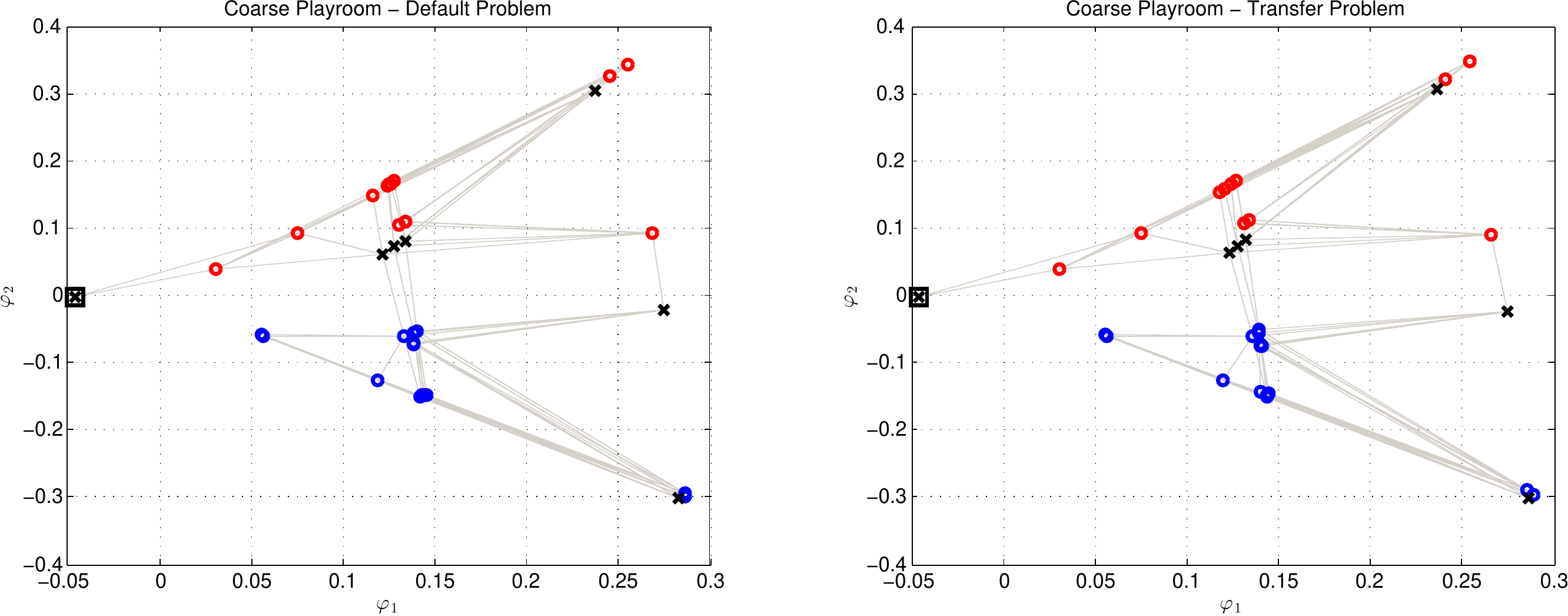}
\caption{{\small\em Diffusion map embeddings of the statespaces for the playroom coarse transfer example of Section~\ref{sec:playroom-simple}: default (left), and transfer (right) tasks. See text for details.}}
\label{fig:playroom-easy-diffmaps}
\end{figure}

Figure~\ref{fig:playroom-easy-diffmaps} shows a 2D diffusion map~\citep{Coifman:PNAS:05} visualization of the statespace graphs for each of the two tasks. Even with two coordinates, the graphs appear nearly identical so that statespace graph matching should not be difficult.  Bottlenecks are marked by '\texttt{x}',  ordinary states with '\texttt{o}', and goal (terminal) states are boxed. Non-bottleneck states are colored according to membership in one of the two possible identified \clusterstext. Edges connect pairs of states for which there is a non-zero transition probability given some action.  Although state {\em embeddings} may be similar, the underlying states can be very different. As can be seen from the plots, the goal states in particular have similar diffusion map coordinates, but of course represent different underlying states of the environment. For both tasks, spectral clustering resulted in two \clusterstext and seven bottlenecks. 

To demonstrate coarse transfer from the default task to the transfer task, we will:
\begin{enumerate}\itemsep -1pt
\item Match bottlenecks across problems by coarse scale statespace graph matching, following  Section~\ref{sec:graph-matching}.
\item Compute values for the transfer tasks's bottlenecks by transferring the default task's potential operator following Section~\ref{sec:potential-xfer}, along the coarse statespace correspondence determined in Step~(1). For this problem, the coarse action mapping determined following Section~\ref{sec:policy-xfer} is simply the canonical correspondence defined by executing the fine scale policy in a designated  \clustertext. If action $a$ in $\problemone$ means ``execute the fine policy in \clustertext\ $\clusterindex$'', then this action is mapped to an action $a'$ in $\problemtwo$ corresponding to ``execute the fine policy in \clustertext\ $\clusterindex^{\prime}$'', if \clustertext\ $\clusterindex^{\prime}$ is matched to \clustertext\ $\clusterindex$ following Section~\ref{sec:xfer-cluster-correspond}.
\item Push the coarse solution down to solve the transfer task at the fine scale following  multiscale Algorithm~\ref{alg:hierarchy-solve} and variants thereof listed in Table~\ref{tab:expt-algs}.
\end{enumerate}

\noindent{\em Bottleneck Matching:} Any graph matching algorithm may be used. We used the procedure described in Section~\ref{sec:graph-matching}, together with the matching algorithm of~\cite{Huang:AISTATS:11} (using their freely available MATLAB implementation), to confirm that the matching for this problem can be done easily. The bottlenecks and their correspondences were as follows (matched bottlenecks are listed on the same row):

\begin{center}
\begin{tabular}{l l}
Default  & Transfer \\
\texttt{(look, marker, music, bell, light)} & \\ \hline
\texttt{(ball, bell, on, on, off)} & \texttt{(ball, bell, on, off, on)} \\
\texttt{(music, ball, on, off, off)} & \texttt{(music, ball, on, off, off)} \\
\texttt{(music, music, off, off, off)} & \texttt{(music, music, off, off, off)} \\
\texttt{(music, bell, on, off, off)} & \texttt{(music, bell, on, off, off)} \\
\texttt{(music, light, on, off, off)} & \texttt{(music, light, on, off, off)} \\
\texttt{(bell, bell, off, off, off)} & \texttt{(bell, bell, off, off, off)} \\
\texttt{(bell, bell, on, off, off)} & \texttt{(bell, bell, on, off, off)}
\end{tabular}
\end{center}
The the goal states in each problem (top row) are successfully paired, while the rest of the bottlenecks are
identical across tasks.\\

\noindent{\em Transfer Problem Solution:}
We evaluated several multiscale algorithms listed in Table~\ref{tab:expt-algs} (see surrounding text for a description). For all algorithms, the blending parameter appearing in policy updates was set to $\lambda=1$. In all non-transfer experiments, the initial coarse scale value function was obtained by solving the coarse MDP given by compression with respect to the diffusion policy. The diffusion policy was chosen over the pool method in Section~\ref{sec:cluster_pols} for simplicity, so that actions across the two problems could be placed into a natural correspondence, and so that error due to the action mapping would not be conflated with other sources of error. In all experiments, the initial fine scale policy was chosen arbitrarily to be the (deterministic) policy which always takes the \texttt{look} action. 

Figure~\ref{fig:playroom-simple} compares performance among multiscale algorithms and to the canonical  policy iteration method. As before, vertical axes labeled ``Error'' show the Euclidean
distance between the value function after $t$ iterations ($x$-axis) of the given algorithm, and the optimal value function for this problem. Inset plots detail boxed regions. See Table~\ref{tab:expt-algs} for a description of the algorithms' labels. 

Figures~\ref{fig:playroom-easy-transfer} and~\ref{fig:playroom-easy-notransfer} show the performance of various multiscale algorithms with and without fine scale policy transfer, respectively.  Comparing  the two plots, transfer provides a good warm start (lower starting error), and also affords faster convergence (fewer iterations) for all multiscale algorithms. For transfer experiments (Figure~\ref{fig:playroom-easy-transfer}), the coarse initial value function comes from transfer, and we may reasonably assume it is reliable. For this reason, algorithms which leverage the initial coarse information as far as possible, and iterate inside \clusterstext to convergence before updating the boundary  (``MS-\{\texttt{cc,co,cr}\}'' traces), perform better than the algorithms that only iterate over interiors once  between boundary updates (``MS-\{\texttt{oo,oc,or}\}'' traces). In contrast, for the experiments which did not involve transfer (Figure~\ref{fig:playroom-easy-notransfer}), the MS-\{\texttt{cc,co,cr}\} algorithms exhibit slower convergence than the MS-\{\texttt{oo,oc,or}\} family. Since the initial coarse value function was derived from a fine scale diffusion policy in the no-transfer setting, we can conclude, as one would expect, that the initial coarse estimate was not entirely reliable. When there is potential operator  transfer, algorithms  MS-\{\texttt{cc,co,cr}\} are equally good for solving the problem, and MS-\texttt{cc} is the best algorithm not involving recompression. In the absence of transfer, MS-\texttt{or}  performs best and converges faster than policy iteration, while MS-\texttt{oc} is the best algorithm not involving recompression. Although in this example the recompression algorithms (``MS-\{\texttt{or,cr}\}'' traces) do not converge to the optimal value function, the sequence of policies do converge to the optimal policy. All of the multiscale algorithms reach optimal policies in fewer iterations than policy iteration in the case of transfer. When there is no transfer of information, and the initial coarse scale data is unreliable, then the multiscale algorithms not involving recompression can take more iterations to converge as compared to policy iteration. However, as mentioned previously, each iteration of the multiscale algorithms is substantially faster (involving local computations) than iterations of policy iteration, which is a global algorithm (see Section~\ref{sec:grid_expt} for a discussion regarding this point). 

In Figures~\ref{fig:playroom-easy-compare-cc} and~\ref{fig:playroom-easy-compare-oc} we compare algorithms MS-\texttt{cc} and MS-\texttt{oc}, with and without transfer, to the policy iteration algorithm on a linear scale. Policy iteration in all cases starts from the same initial fine scale policy as the multiscale algorithms. These plots more clearly demonstrate the advantage afforded by the coarse scale transfer: traces labeled ``Transfer'' confirm that there is both a warm start (lower starting error) and faster convergence (fewer iterations). The effect is also more pronounced in the case of MS-\texttt{cc}, since this algorithm maximally leverages the transferred information.

\subfiglabelskip=0pt
\begin{figure}[p]
\centering
\subfigure[][]{
\includegraphics[width=0.47\textwidth]{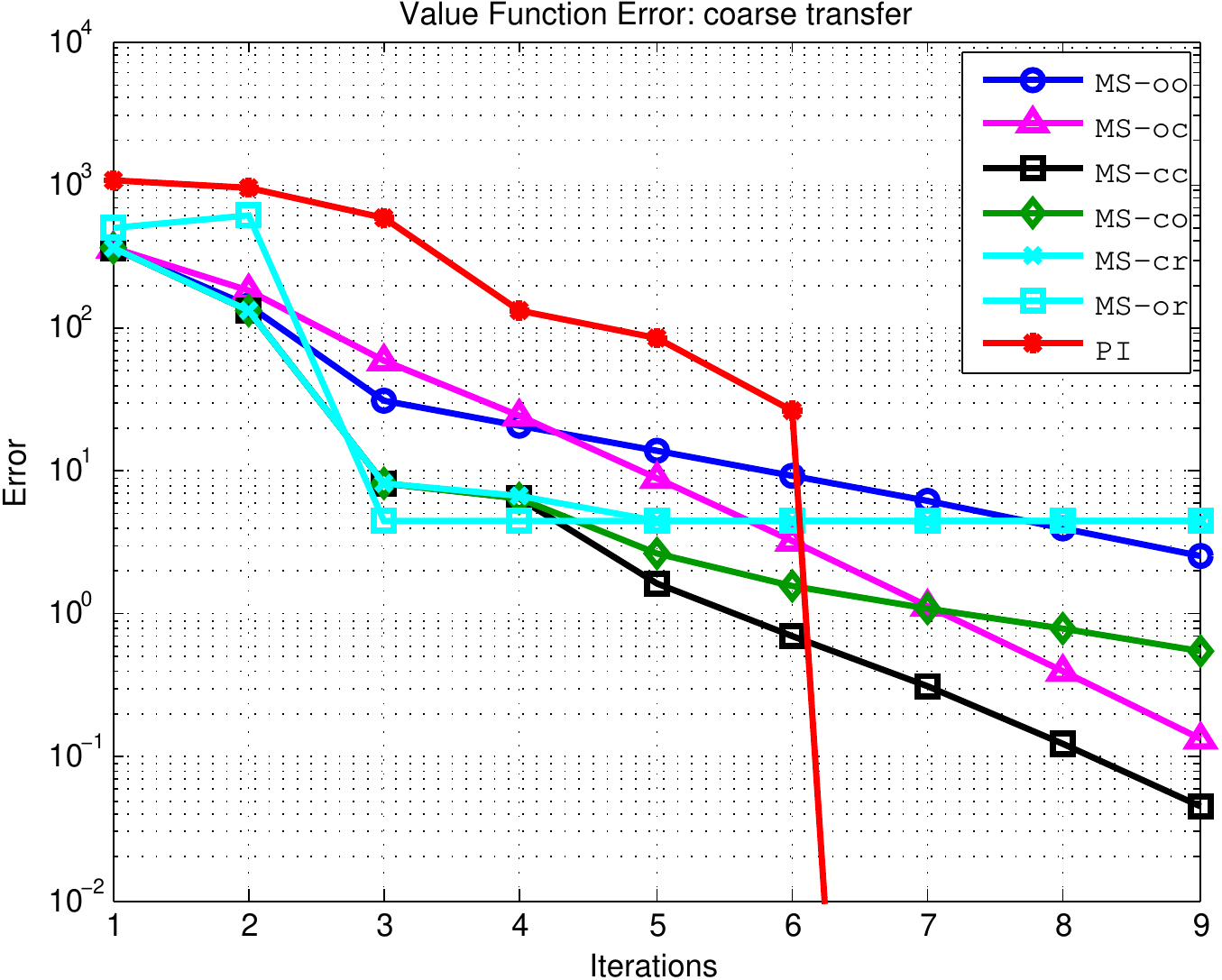}
\label{fig:playroom-easy-transfer}
}
\subfigure[][]{
\includegraphics[width=0.47\textwidth]{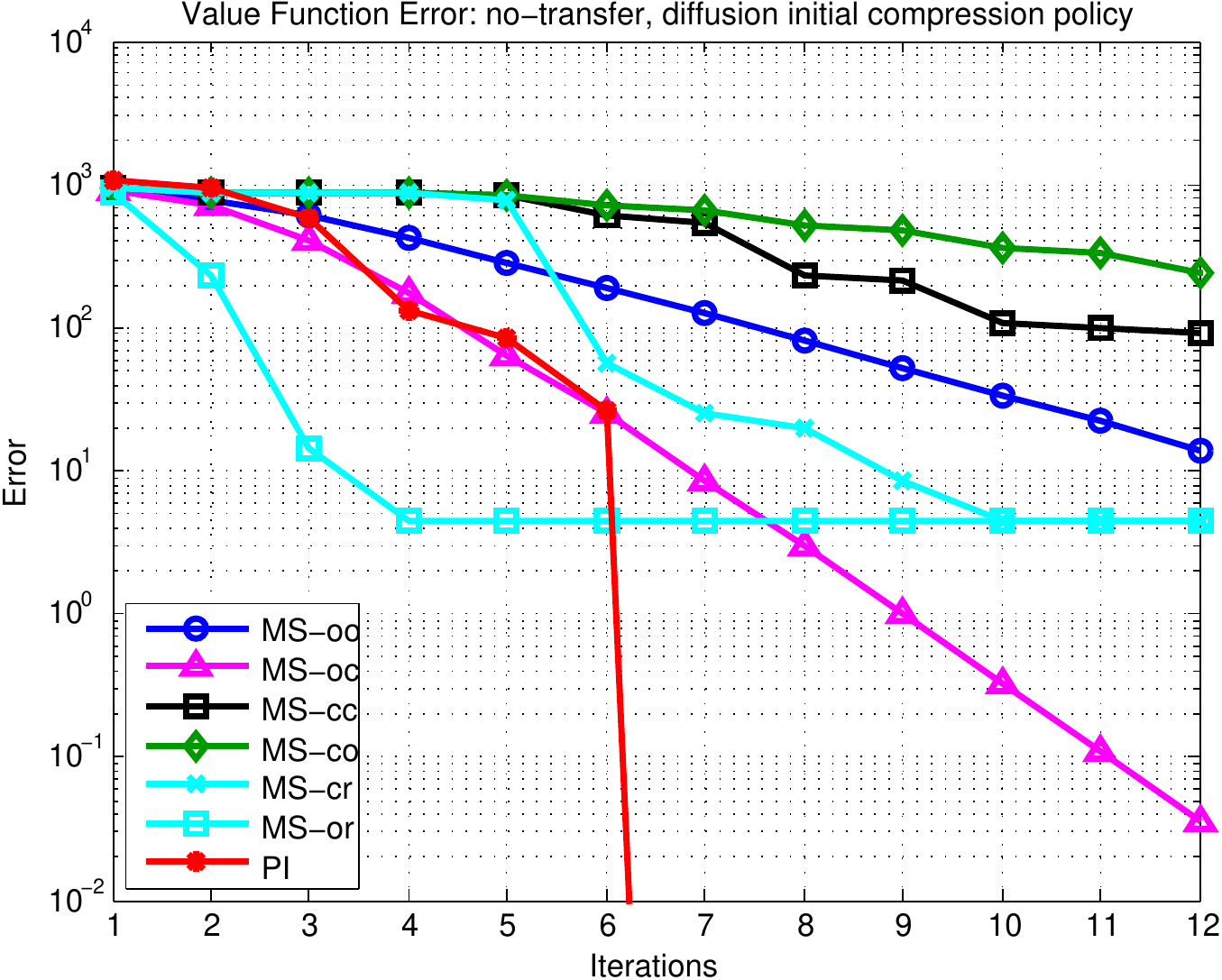}
\label{fig:playroom-easy-notransfer}
}
\\
\subfigure[][]{
\includegraphics[width=0.47\textwidth]{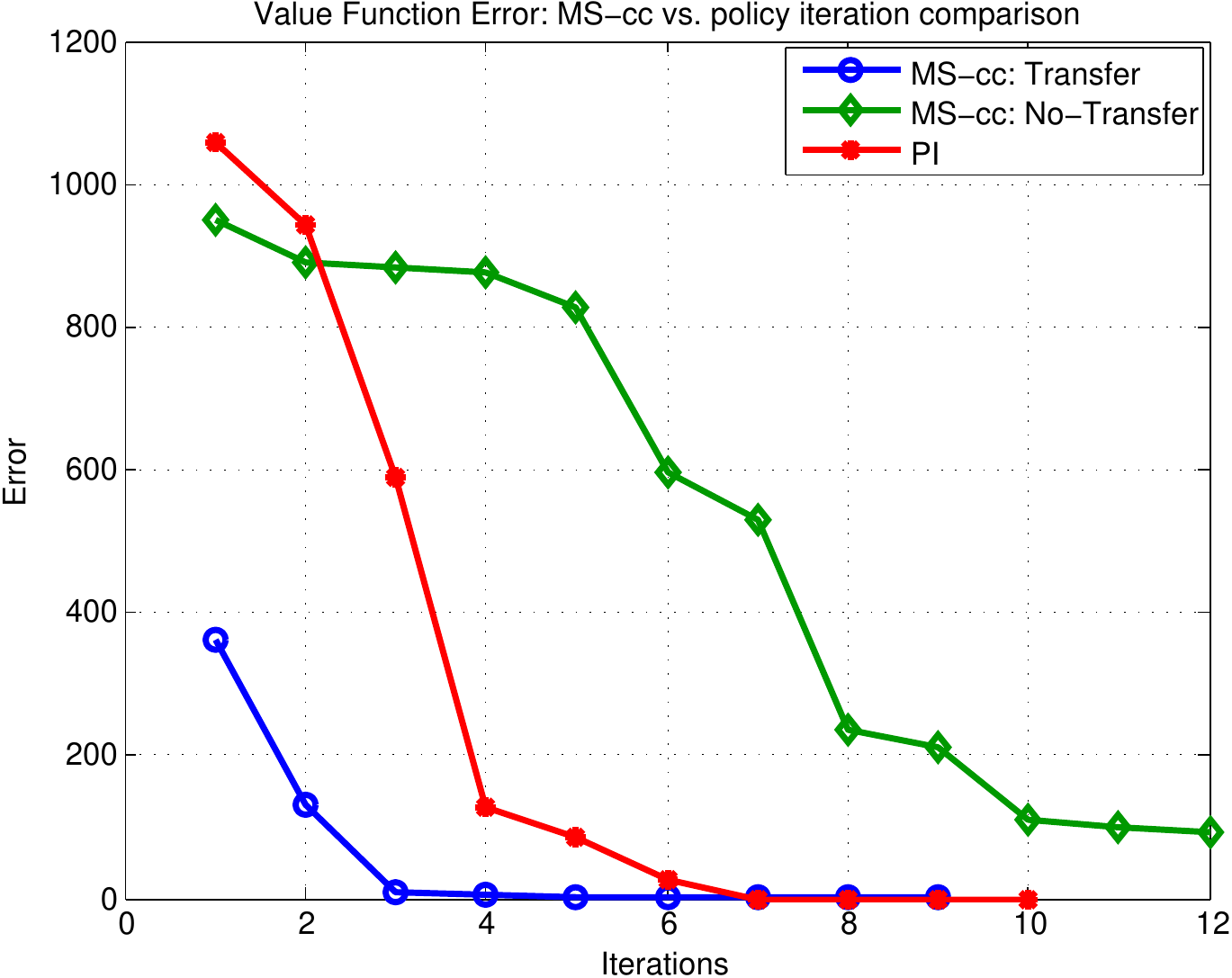}
\label{fig:playroom-easy-compare-cc}
}
\subfigure[][]{
\includegraphics[width=0.47\textwidth]{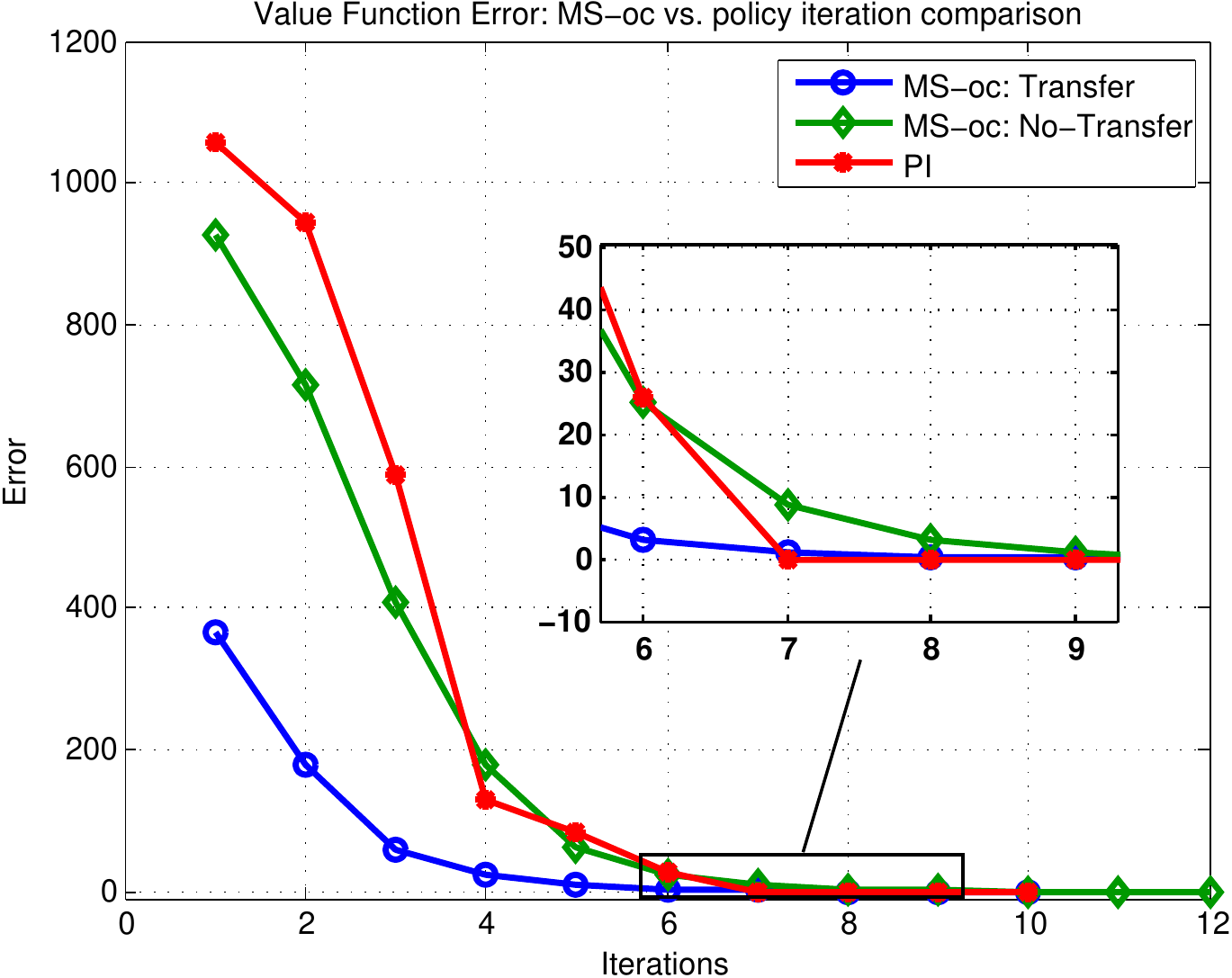}
\label{fig:playroom-easy-compare-oc}
}
\caption{{\small\em Solution of the playroom coarse transfer problem described in Section~\ref{sec:playroom-simple}: Comparison among multiscale algorithms both with and without coarse scale potential operator transfer, and comparison to policy iteration. See text for details.}}
\label{fig:playroom-simple}
\end{figure}


\subsubsection{Partial Policy Transfer Example}\label{sec:playroom-partial}
This example illustrates partial transfer of a policy at the fine scale of a two scale hierarchy.
The pair of problems in this section differ from the previous section only in that to turn on the light, the agent must be looking at the light switch and the marker must be on the bell. Now, turning on the light differs from ringing the bell by two actions. The tasks are as follows:
\begin{description}
\item[{\normalfont Default (\problemone):}]
The goal is to cause the bell to ring while the music is playing. The agent must look at the music button, press the music button, look at the bell, place the marker on the bell, look at the ball, and finally, kick the ball into the bell. The light switch does not play a role.
\item[{\normalfont Transfer (\problemtwo):}] The goal is to flip the light switch to the on position while the music is playing. The agent must: look at the music button, turn on the music, look at the bell, place the marker on the bell, {\em look at the light switch, flip the light switch}.
\end{description}
As before, each episode begins with the agent looking at a random object, the marker on a random object, and all on/off objects in the off state.

\begin{figure}[t]
\centering
\includegraphics[width=\textwidth]{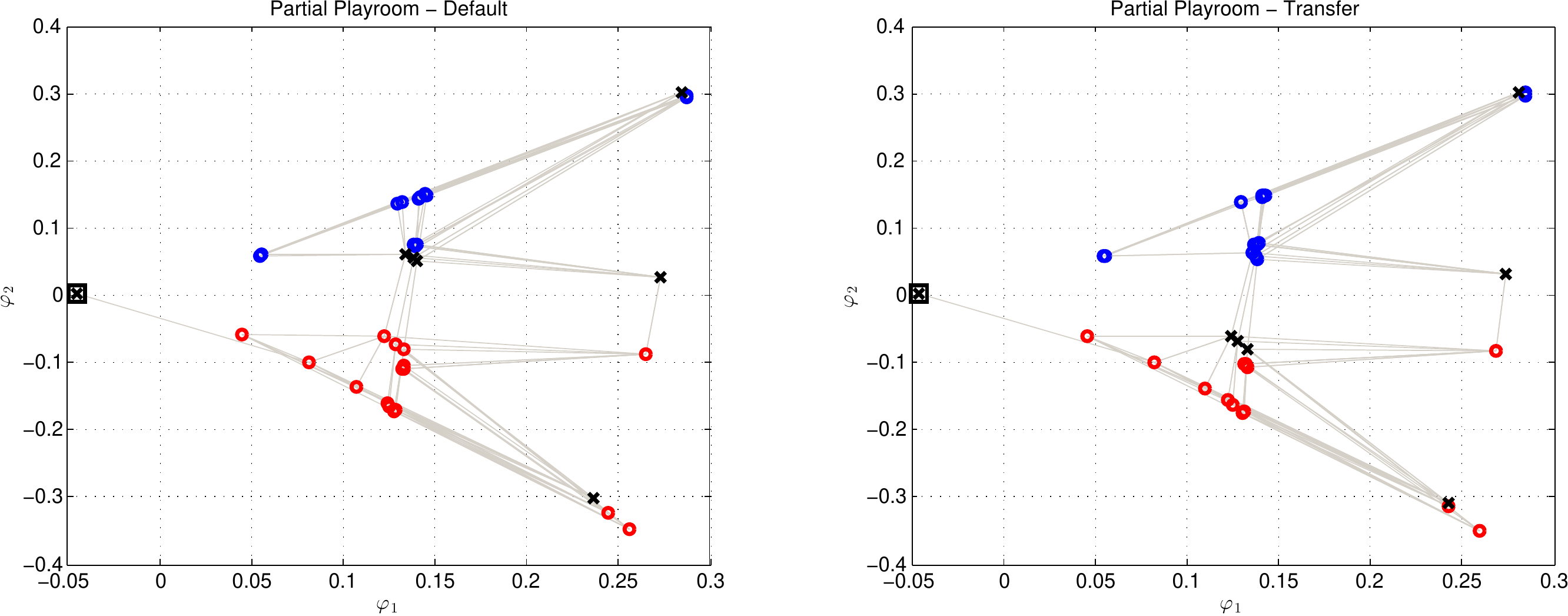}
\caption{{\small\em Diffusion map embeddings of the statespaces for the partial playroom transfer example of
Section~\ref{sec:playroom-partial}: default (left), and transfer (right) tasks. See text for details. }}
\label{fig:playroom-partial-diffmaps}
\end{figure}

Although this pair of problems involves only one additional action change in the sequence leading up to the goal as compared to the previous section's pair, the detected bottlenecks at the coarse scales cannot be easily matched. Figure~\ref{fig:playroom-partial-diffmaps} shows 2D diffusion map visualizations of the statespace graphs for the two tasks described in this section. Again, bottlenecks are marked by '\texttt{x}',  ordinary states with '\texttt{o}', and goal states are boxed. In contrast to Figure~\ref{fig:playroom-easy-diffmaps}, here it can be seen comparing the default task to the transfer task, that some bottlenecks become interior states and vice versa. Thus, direct matching of the coarse scale statespaces (assuming a two layer compression hierarchy) and subsequent coarse scale policy transfer is not an immediate possibility here.

What we might hope, however, is that we can transfer the portion of the fine scale policy dealing with states in which the music is off. It is only after the music is on that the optimal action sequences for the two tasks diverge, and when the music is off the immediate sub-goal for both tasks is to turn it on. Figure~\ref{fig:playroom-partial-diffmaps} confirms this possibility: interior states are colored according to the spectral partition as before, and the \clusterstext in this case correspond to ``music ON'' vs. ``music OFF'' states\footnote{ Of course, one does not need to know or identify what the \clusterstext mean in order to transfer something; we provide this explanation for illustrative purposes.}. As can be seen from the plots, the sign of the (Fiedler) eigenvector $\varphi_2$ gives this partitioning. The procedures described in Section~\ref{sec:transfer} were next applied to the current pair of tasks in order to detect transferability and effect policy transfer.\\

\noindent{\em Cluster Correspondence:}
The \clustertext correspondence step (Section~\ref{sec:xfer-cluster-correspond}) correctly paired together ``music OFF'' and ``music ON'' \clusterstext, respectively. The pairwise \clustertext distances
were found to be:
\begin{center}
\begin{tabular}{l|c|c}
& \problemtwo: Cluster 1 & \problemtwo: Cluster 2\\ \hline
\problemone: Cluster 1 & 0.2850 & 0.4115 \\
\problemone: Cluster 2 & 0.3583 & 0.3201
\end{tabular}.
\end{center}
Here, Cluster 1 is the ``music OFF'' \clustertext and Cluster 2 is the ``music ON''  \clustertext. \\

\begin{figure}[t]
\centering
\includegraphics[width=\textwidth]{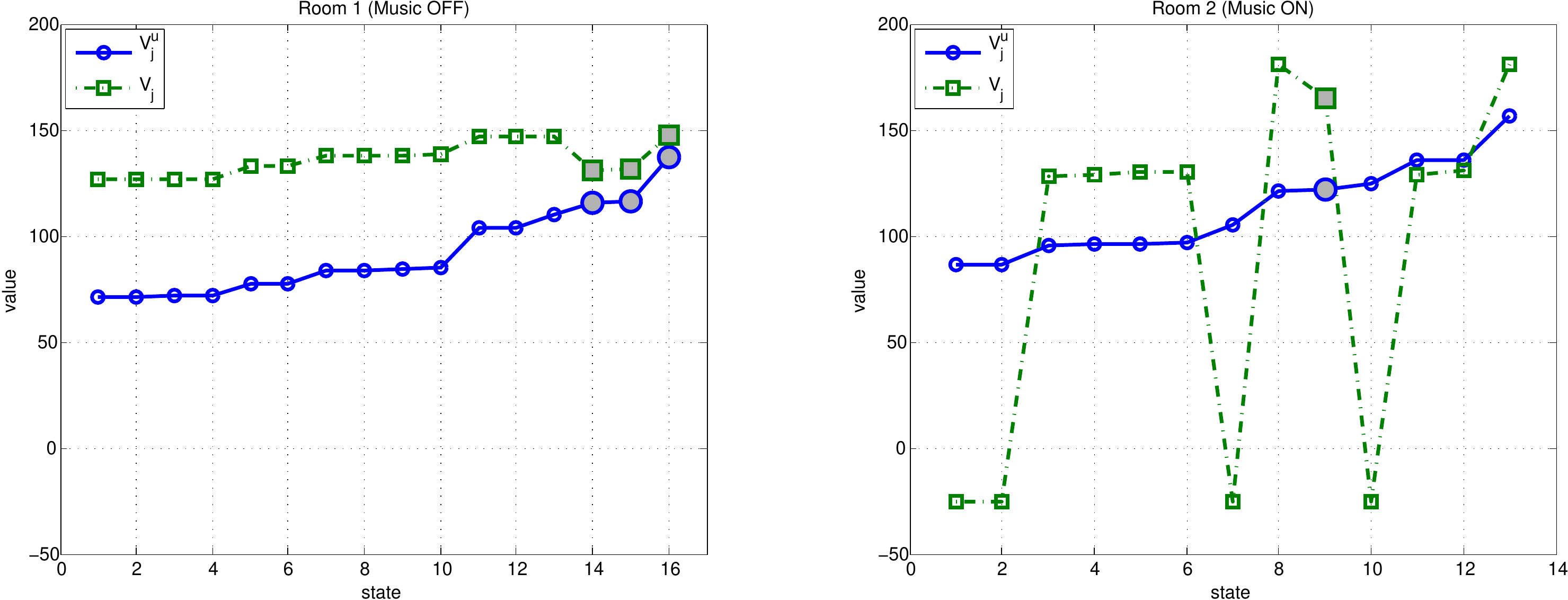}
\caption{{\small\em Transfer detection. Value functions for Cluster 1 (left) and Cluster 2 (right), shown over the entire respective \clustertext interiors. States inside intersection regions are plotted with small open points,  and large filled points identify all other states. The value functions were obtained using $\problemone$'s optimal policy following Section~\ref{sec:xfer-detect}, and are plotted in ascending sorted order according to the magnitude of $V_0$. See text for details.}}
\label{fig:playroom-partial-detect}
\end{figure}

\noindent{\em Detecting Transferability:}
Next, the transfer detection algorithm described in Section~\ref{sec:xfer-cluster-correspond} was applied separately to the pairs (Cluster 1 $\in\problemone$, Cluster 1 $\in\problemtwo$) and (Cluster 2 $\in\problemone$, Cluster 2 $\in\problemtwo$) following the \clustertext correspondence above. No statespace graph matching was done for this example, so we are effectively assuming that that the roles of the states in each problem are the same. Assessing transferability is therefore important in order to determine if and where this assumption might hold. Figure~\ref{fig:playroom-partial-detect} shows the value functions calculated using $\problemone$'s optimal policy (green box traces, labeled $V_0$ for scale $j=0$) and the diffusion policy  (blue circle traces, labeled $V_0^u$). States inside $\problemone,\problemtwo$ \clustertext intersection regions are plotted with small open points,  and large filled points identify all other states. The states (horizonal axis) are ordered according to the magnitude of $V_0^u$ for improved readability. The left-hand plot shows values for states in Cluster 1 (``music OFF''). The transferred policy clearly leads to more expected reward everywhere, and the two value functions follow a similar general trend. The right-hand plot in Figure~\ref{fig:playroom-partial-detect} shows value functions on the states in Cluster 2 (``music ON''). Here there is large disagreement on several states, suggesting that applying the  optimal policy from $\problemone$ to $\problemtwo$ in Cluster 2 could be problematic. This is not surprising considering that the goal states for both tasks are either inside or connected to the respective problem's Cluster 2. Indeed, the test given by Equation~\eqref{eqn:detect-test-matching} produces  $T=-1.31$ in Cluster 2, while $T=+6.64$ in Cluster 1. We conclude that transfer in Cluster 1 should be attempted, but transfer in Cluster 2 should not be attempted in the absence of a better statespace mapping.\\

\noindent{\em Policy Transfer:}
With confirmation that transfer within Cluster 1 could be helpful, the optimal policy from $\problemone$ was
transferred to the overlapping interior $\problemtwo$ Cluster 1 states following the process described in Section~\ref{sec:policy-xfer}. As mentioned earlier, the identity correspondence between the relevant underlying states in each task was assumed, and the actions mapped accordingly. For this particular problem the actions do not change, but we followed the general mapping process anyhow since one does not generally know {\em a priori} whether actions need to be changed, or whether the representation of the two problems is such that actions are known by the same labels or not. The initial policy, post transfer, as well as the optimal policy on the interior of $\problemtwo$ Cluster 1 were as follows:

\begin{center}
{\tt
\begin{tabular}{l|c|c|c|c|c|c|c|c|c|c|c|c|c|c|c|c}
{\rm Initial} & 1 & 1 & 1 & 1 & 1 & ? & ? & ? & 2 & 2 & 2 & 1 & 1 & 1 & 1 & 1  \\ \hline
{\rm Optimal} & 1 & 1 & 1 & 1 & 1 & 3 & 3 & 3 & 2 & 2 & 2 & 1 & 1 & 1 & 1 & 1
\end{tabular}}
\end{center}
where {\tt 1} means ``look at a random object'', {\tt 2} means ``place the marker'' and {\tt 3} means ``press music button'' (see action definitions at the top of Section~\ref{sec:expts-playroom}). Question marks in the initial policy indicate states which did not recieve a policy from $\problemone$. Cluster 1 in $\problemone$ contained 13 interior states, while Cluster 1 in $\problemtwo$ contained 16; thus the maximum of 13 policy entries were transferred to $\problemtwo$. Unknown states are given a default guess below. As can be seen in the table above, for this task the transferred policy entries correctly matched the optimal policy for $\problemtwo$. \\

\subfiglabelskip=0pt
\begin{figure}[p]
\centering
\subfigure[][]{
\includegraphics[width=0.47\textwidth]{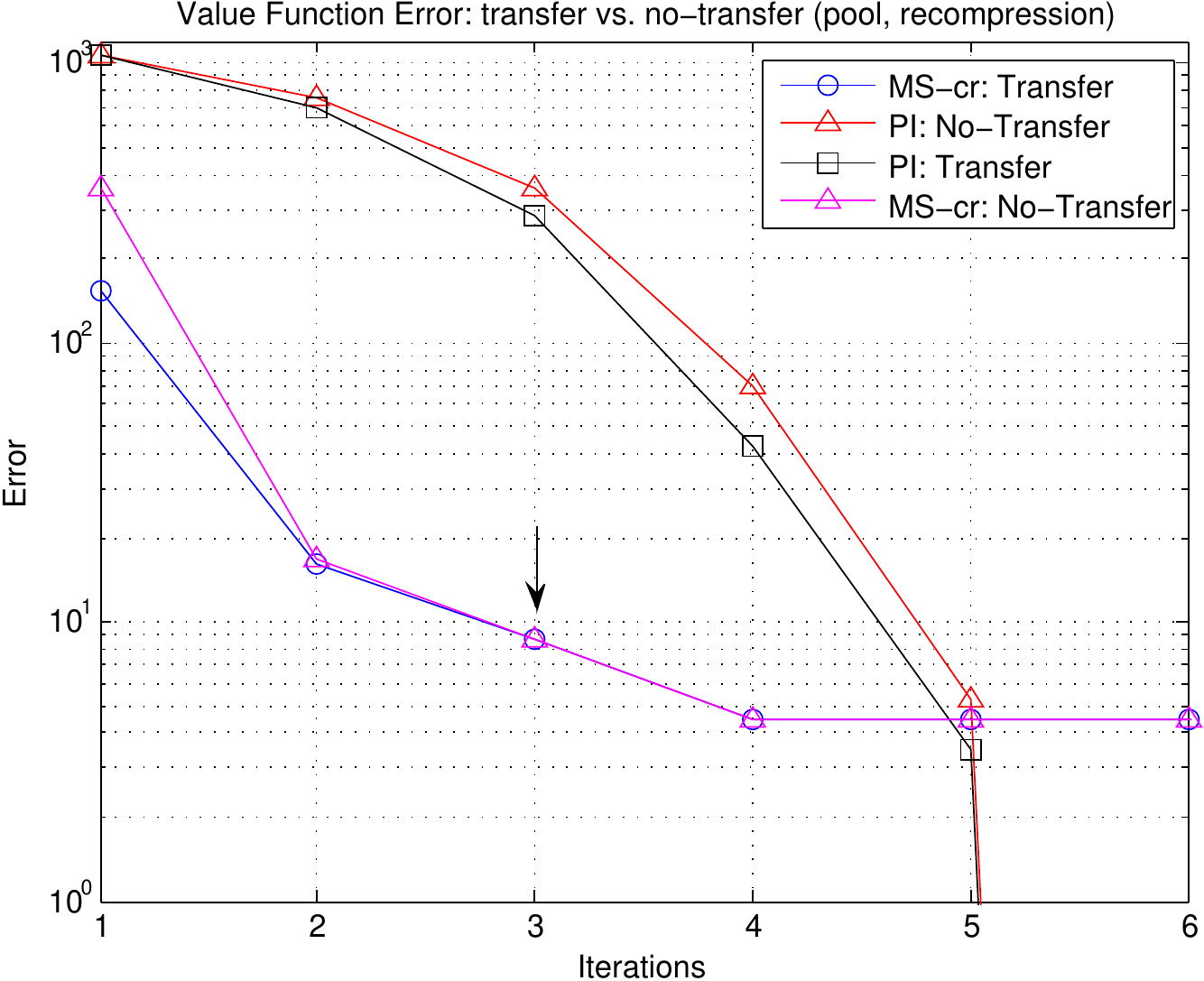}
\label{fig:play-partial-pool-comp}
}
\subfigure[][]{
\includegraphics[width=0.47\textwidth]{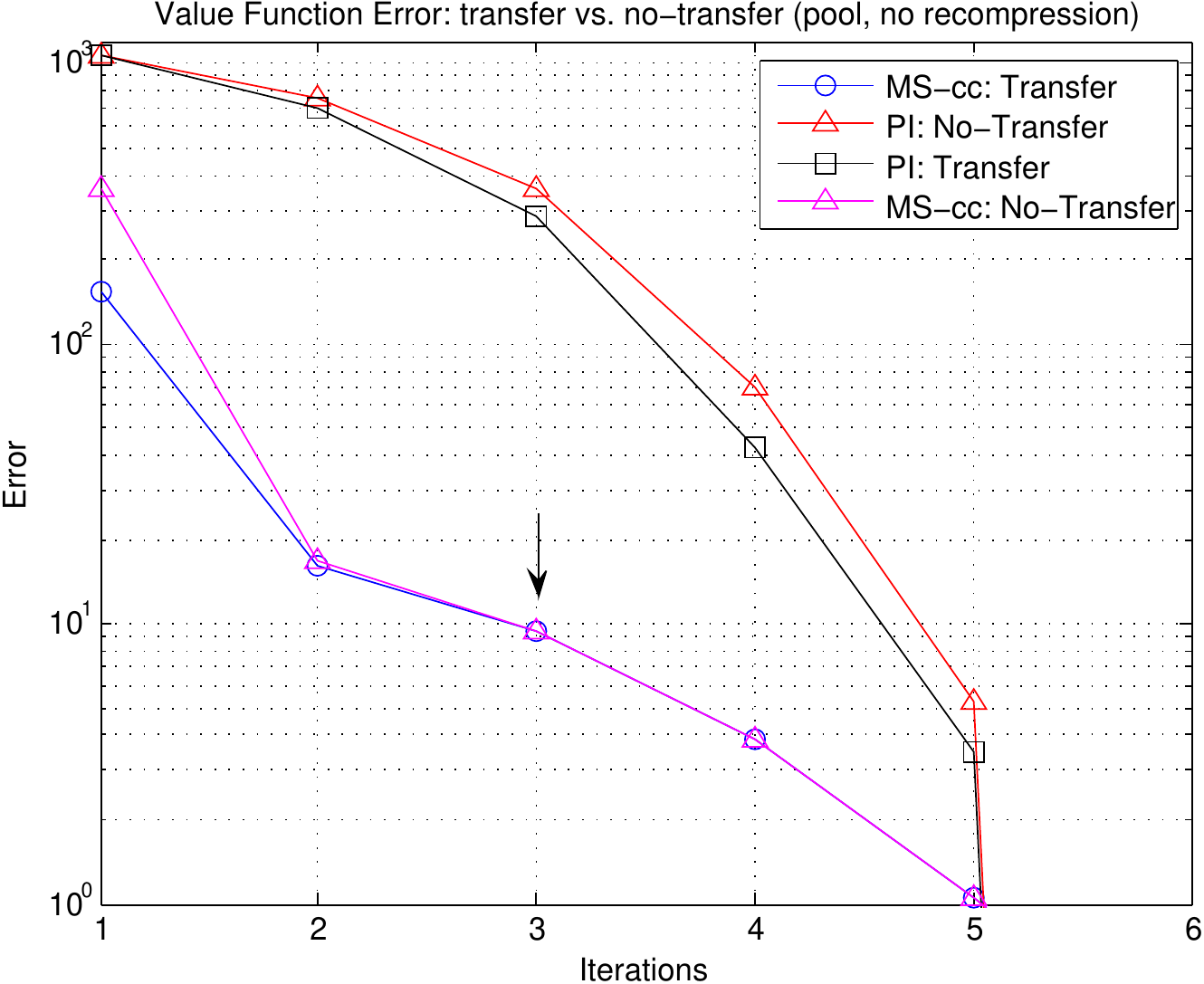}
\label{fig:play-partial-pool-nocomp}
}
\\
\subfigure[][]{
\includegraphics[width=0.47\textwidth]{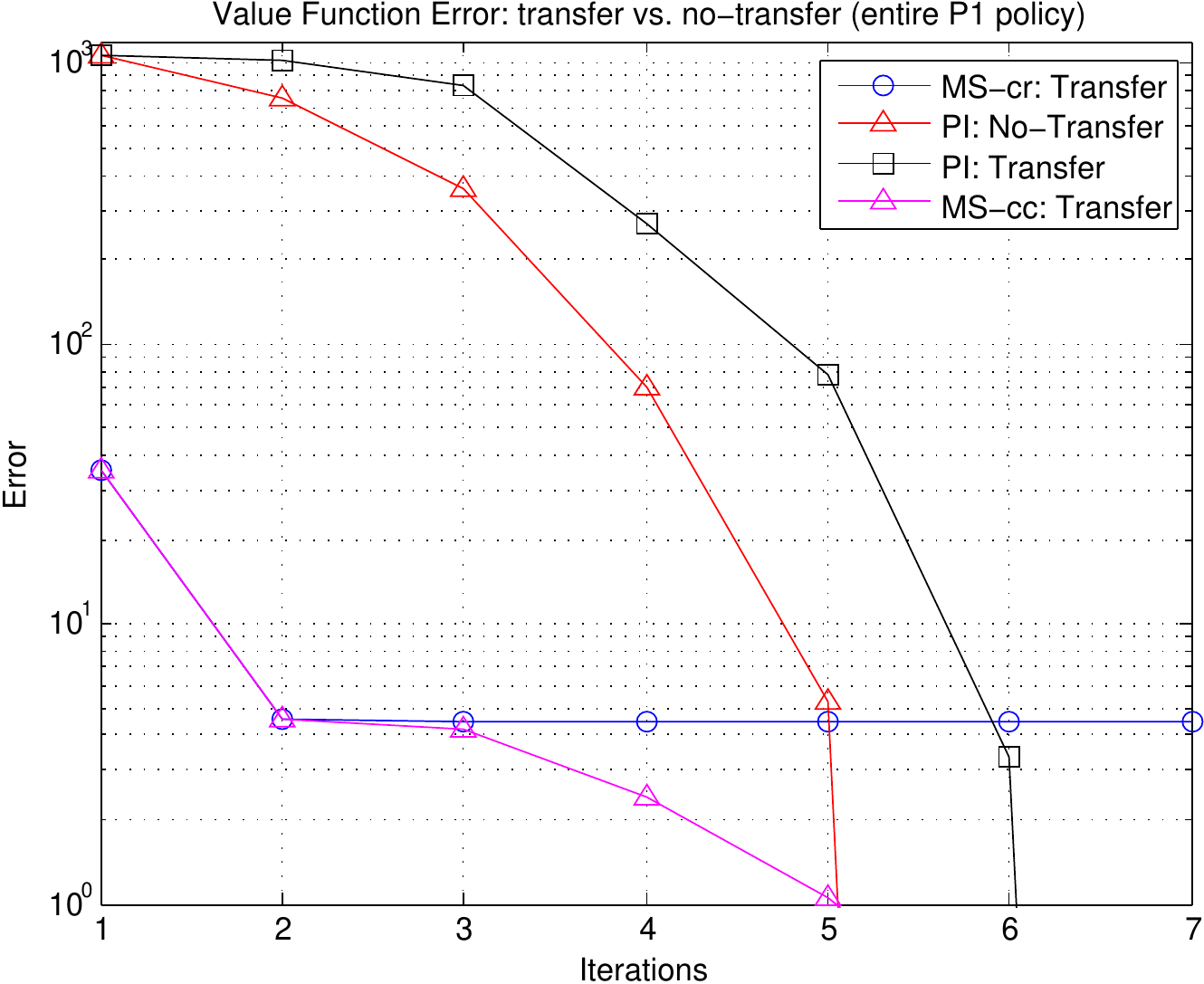}
\label{fig:play-partial-pool-full}
}
\\
\subfigure[][]{
\includegraphics[width=0.47\textwidth]{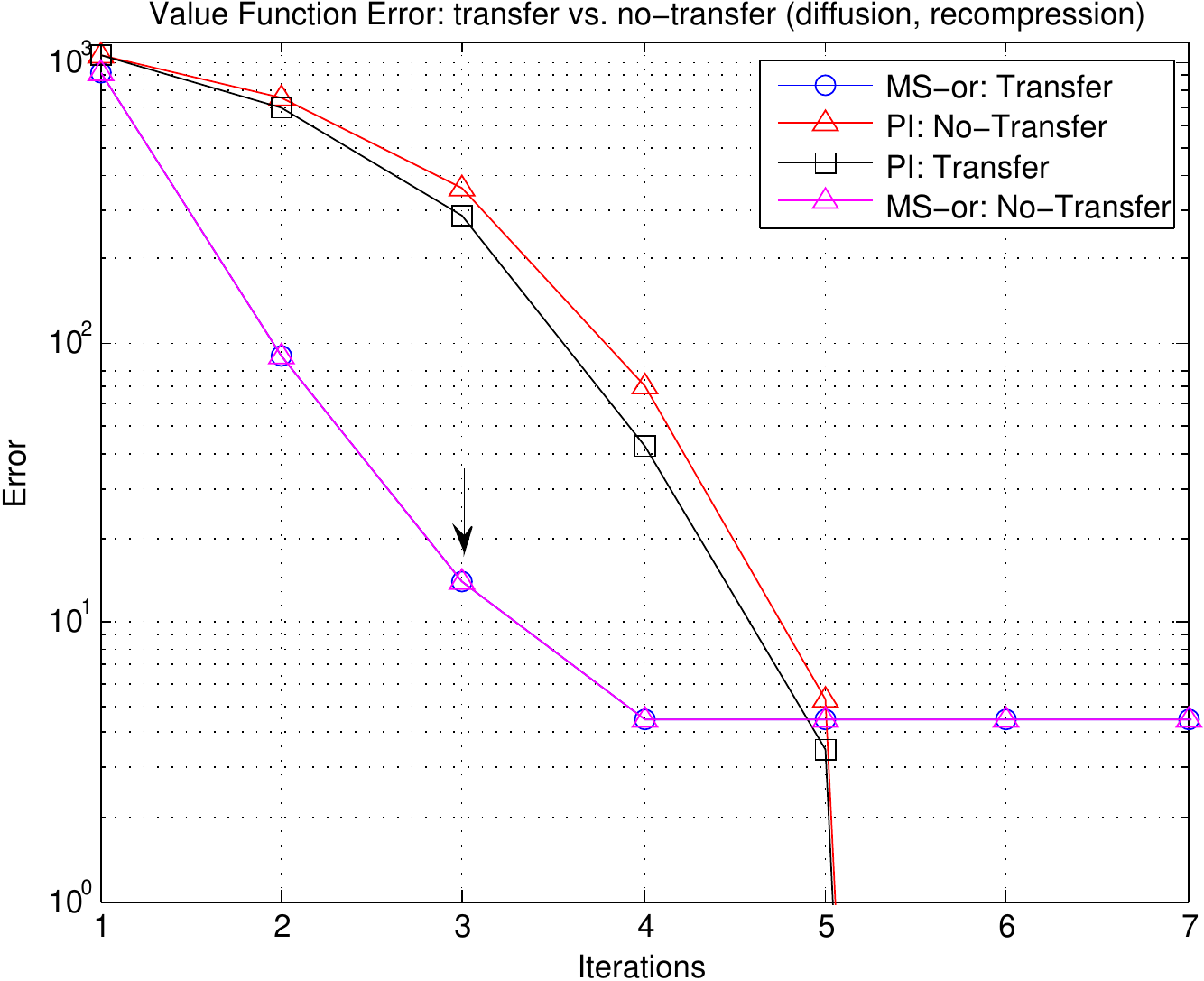}
\label{fig:play-partial-diff-comp}
}
\subfigure[][]{
\includegraphics[width=0.47\textwidth]{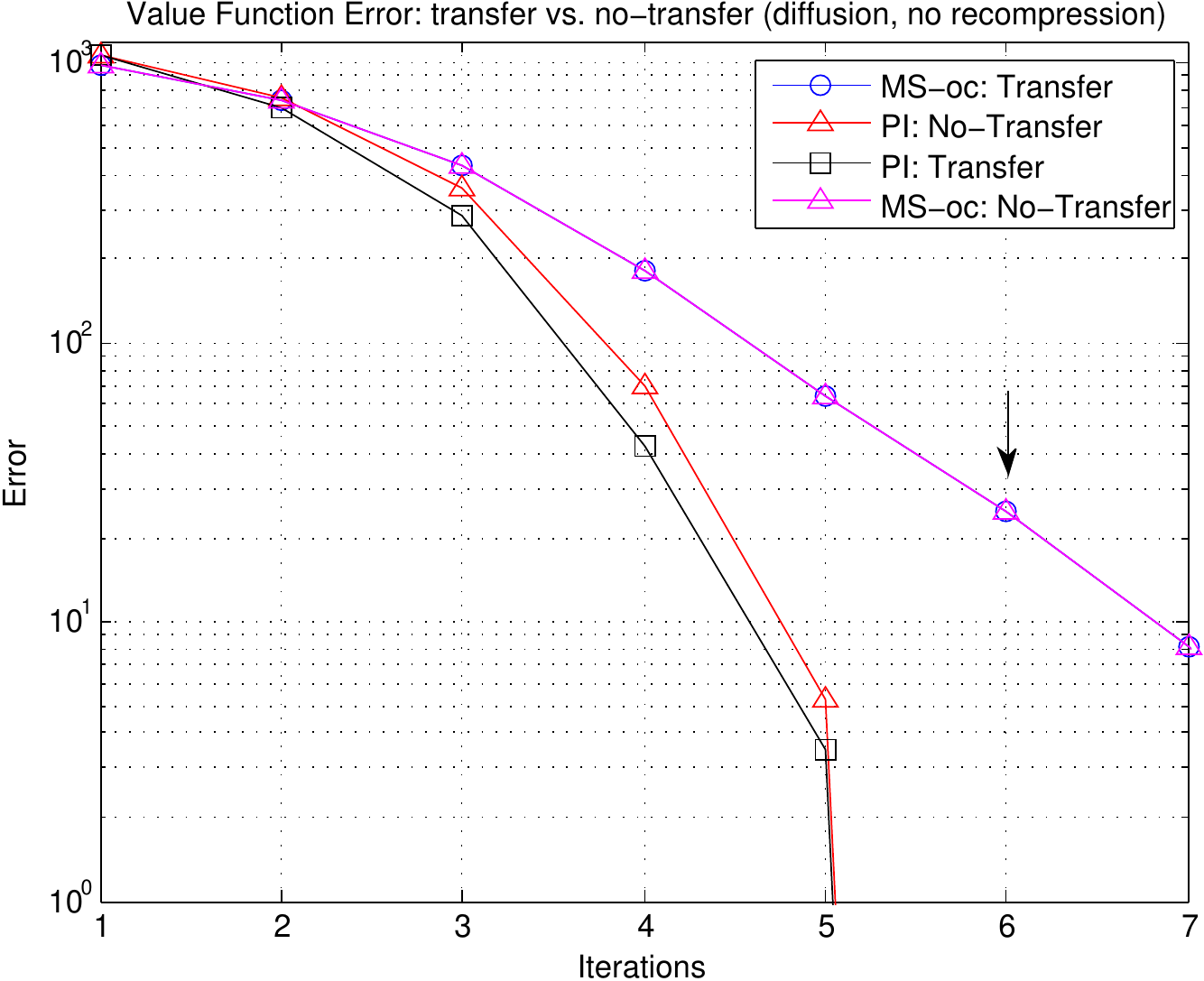}
\label{fig:play-partial-diff-nocomp}
}
\caption{{\small\em Value function errors for the playroom partial transfer example of Section~\ref{sec:playroom-partial}: comparison between policy iteration, and multiscale algorithms both with and without compression. See text for details.}}
\label{fig:playroom-partial}
\end{figure}

\noindent{\em Transfer Problem Solution:}
Figure~\ref{fig:playroom-partial} compares performance both with and without partial fine-scale policy transfer, across several different solution algorithms. The four curves in each plot correspond to different initial conditions and/or solution algorithms, and give the Euclidean distance (``Error'', vertical axis) between intermediate value functions computed after $t$ iterations (horizontal axis) of a given algorithm, and the true, optimal value function for $\problemtwo$. Traces labeled with the prefix ``PI: No-Transfer'' correspond to vanilla policy iteration on the global statespace, with no transfer information, while the traces labeled ``PI: Transfer'' correspond to policy iteration starting from the transferred fine scale policy. The ``PI: No-Transfer'' curve is the same in all plots, and the ``PI: Transfer'' curve is the same in all plots except Figure~\ref{fig:play-partial-pool-full}. Traces labeled ``MS: Transfer'' and ''MS: No-transfer'' refer to different multiscale algorithms appearing in Table~\ref{tab:expt-algs}, with and without initial coarse solutions obtained following Section~\ref{sec:cluster_pols}, and with and without fine-scale policy transfer (respectively). In the transfer experiments, the transferred policy served as the initial policy. In all cases, the blending parameter was set to $\lambda=1$ (no blending), giving a purely greedy policy update. Arrows in the plots mark the point at which the {\em policy} has converged to the optimal (fine) policy\footnote{This is clearly not detectable in practice, but is informative in the context of recompression-based algorithms which do not converge to the optimal value function in this example.}. The particular multiscale algorithms and conditions we tested in each plot are as follows:

\begin{center}
\begin{tabular}{c|c|c}
Figure & Initial Compression & Multiscale Algorithm \\\hline
Figure~\ref{fig:play-partial-pool-comp} & pool & \texttt{cr}  \\
Figure~\ref{fig:play-partial-pool-nocomp} & pool & \texttt{cc} \\ 
Figure~\ref{fig:play-partial-pool-full} & pool & \texttt{cr,cc} \\
Figure~\ref{fig:play-partial-diff-comp} & diffusion & \texttt{or} \\ 
Figure~\ref{fig:play-partial-diff-nocomp} & diffusion & \texttt{oc}
\end{tabular}
\end{center}
See Table~\ref{tab:expt-algs} and surrounding discussion for a description of the algorithms. 
The ``Initial Compression'' column above specifies whether the initial coarse value function solved a  coarse MDP compressed with respect to a collection of policies as described in Section~\ref{sec:cluster_pols} ({\em ``pool''}), or with respect to the diffusion policy ({\em ``diffusion''}). We assume that the initial coarse value function under the {\em pool} condition is trustworthy, and choose multiscale algorithms which iterate within \clustertext interiors to convergence before updating the boundary. For initial coarse value functions derived from the diffusion policy, we assume there could be errors and opt for multiscale algorithms which only update \clustertext interiors once before each boundary update. 

Several conclusions (specific to this problem domain) may be drawn from these experiments:
\begin{enumerate}
\item The impact of the transferred policy is essentially only noticeable when used in conjunction with a good initial coarse guess (Figs.~\ref{fig:play-partial-pool-comp},\ref{fig:play-partial-pool-nocomp}). Both algorithms (recompression/no-recompression) give similar performance. The recompression-based algorithm does not converge to the optimal value function, although the corresponding policy sequence does converge to the optimal policy.
\item For the canonical policy iteration algorithm, using the transferred policy as the initial condition gives only a slight advantage. However, policy iteration is far less robust to errors in the transferred policy than the family of multiscale algorithms. Figure~\ref{fig:play-partial-pool-full} shows the result of transferring the {\em entire} fine policy for $\problemone$, even in the \clustertext where transfer detection suggested transfer could be error prone. The multiscale algorithm with recompression is labeled ``\texttt{MS-c}'', and without compression ``\texttt{MS-nc}''. For this particular problem, the multiscale algorithms are tolerant of these errors, and convergence is in fact faster. (We emphasize, however, that this may not at all be true for other problems.) Policy iteration, however, suffers, and takes additional time to correct errors in the second (error-prone) \clustertext's policy. 
\item When the diffusion policy is used for compression (Figures~\ref{fig:play-partial-diff-comp},\ref{fig:play-partial-diff-nocomp}), transfer has little impact. Furthermore, recompression during the solution process is necessary to quickly correct errors imposed by a poor initial coarse  value function. The algorithm involving local bottleneck updates (Figure~\ref{fig:play-partial-diff-nocomp}) requires more iterations to converge to optimality, as compared to the other algorithms.
\item Ignoring cost per iteration, most of the improvement of the multiscale algorithms over policy iteration come from the algorithms themselves rather than the transferred information. However, in large complex domains where \clusterstext may themselves be complex tasks, even small transfer improvements may lead to substantial savings.
\end{enumerate}

\section{Related Work}\label{sec:prior_work}
Our work has many points of contact with the literature, and we do not attempt a comprehensive comparison. We highlight the main, most important similarities and differences.

There are several overarching themes which distinguish our work from much of the literature:
\begin{itemize}\itemsep 1pt
\item Multiscale structure: Multiscale is is a unifying, organizational principle in our work. Our approach enforces a strong multiscale decomposition of tasks into subtasks, such that each scale may be treated independently of the others. Hierarchies of arbitrary depth may be easily constructed. Many approaches ultimately require some form of ``flattening'' (see for instance HAMs~\citep{ParrRussell}, options~\cite{SuttonOptions}), or do not generalize well beyond a single layer of abstraction.
\item Multiscale consistency: Coarse scales are consistent with finer scales ``in the mean'', and each scale is a separate MDP. Semi-Markov decision processes (SMDPs)~\citep{Puterman}, for example, do not share this notion of consistency.
\item Computational efficiency:  The multiscale structure we impose localizes computation and improves conditioning. The computational complexity of learning and planning can be significantly reduced, both in time and in space.
\item Coupling between learning, planning, and structure discovery: Our approach combines learning of macro-actions, multiscale planning, and inference of multiscale structure in a fundamental way. Many existing approaches focus on one or the other, resulting in a disconnect that leads to inefficiency and unresolved challenges. 
\item Transfer: MMDPs support systematic, scale-independent transfer of knowledge between tasks. Knowledge may be exchanged in the form of potential operators, policies, or value functions. 
\item Generality: Different statespace partitioning and bottleneck detection algorithms may be used. Compression may be carried out with respect to any policy, or collection of locally defined policies.  Different value function representations and off the shelf algorithms for solving MDPs may be chosen. Key MMDP quantities may be computed analytically (if the model or an estimate of the model is known), or by Monte Carlo simulation. We do not assume specific choices of algorithms where possible. On the other hand, MMDPs are {\em more constrained} than SMDPs. SMDPs are very general objects, but this generality comes at the expense of conceptual and computational complexity.
\end{itemize}
A more detailed comparison to specific work in the literature follows below.

\subsection{Hierarchical Reinforcement Learning}
Empirically, ``standard'' approaches to learning within flat problem spaces are often slow, scale poorly, and do not lend themselves well to the inclusion of prior knowledge. The hierarchical reinforcement learning (HRL) literature (see \cite{BartoHRLReview} for a review,  and~\cite{SuttonOptions,DietterichHRL,ParrRussell} in particular) has long sought to address these challenges by incorporating hierarchy into the domain and into the learning process. The essential goal of the hierarchical learning literature is to divide-and-conquer, paralleling similar strategies for coping with complexity found throughout biology and neuroscience. The notion of state abstraction has been considered extensively. In most cases, coarse or ``macro'' actions are broadly defined as temporally extended sequences of primitive actions. The  pioneering ``options'' framework~\citep{SuttonOptions} proposed a means to solve reinforcement learning problems, given pre-specified collections of such macro-actions. The options framework is closely related to SMDPs, see~\citep{Puterman,DasMahadevan99}, and SMDPs have more generally become a modeling formalism of choice for HRL.  If a set of options have been pre-specified, the problem of learning an optimal policy over a set of options is an SMDP. Many of the hierarchy discovery algorithms surveyed below construct options, and then employ SMDP learning techniques. For this reason we devote special attention to the options framework, and discuss how it relates to the present work.

\subsubsection{Relation to Options and SMDPs}
In the options framework~\citep{SuttonOptions}, a hierarchical value function is used to define a flat, global policy. One level of abstraction is typically considered: Options are policies accompanied by a specification as to when an option can be invoked (``initiation set''), and when it should end, once triggered (``termination condition''). Some elements of our work can be described in the language of options, however there are some important differences distinguishing our framework from that of options/SMDPs. 

Options and SMDPs are general approaches, and some of our design decisions may be viewed as a specialization of certain aspects of options/SMDPs. Other aspects of our work may be viewed as more substantial departures from the options framework. These differences confer computational, transfer and multiscale advantages, and promote learning of the macro-actions themselves: MMDPs are constructed specifically with these objectives in mind. The options framework does not consider learning of the options, and is primarily designed for {\em planning} with pre-specified macro-actions. This disconnect between learning and planning leads to inefficiencies, in terms of both computation and exploration, when SMDPs are combined with separate schemes for learning options. How to learn macro-actions within the options framework is a major challenge that has received considerable attention in the literature, yet there is no clear consensus as to how coarse rewards and discounts should be set locally in order to efficiently learn a macro-action that can be combined with others in a consistent way. Fundamentally, we take a holistic point of view, and couple learning of macro-actions and planning with macro-actions together.  The hierarchical structure we consider is also tightly coupled with computation and conditioning. Solving for macro-actions takes advantage of improved local mixing times at finer scales and fast mixing globally at coarse scales. At any scale, learning is localized, in terms of space and computation, by way of the multiscale decomposition. The local learning is fast, and consistent globally, due to information feeding in from distant parts of the statespace through coarse scale solutions. Our framework also more easily accommodates multiscale representations of arbitrary depth.  Each scale is an MDP and may be solved using any algorithm. In contrast to options, the introduction of additional scales does not necessarily add complexity to the planning phase (indeed, it usually reduces it). As we will discuss below, that each scale is an MDP also   supports further transfer opportunities.

We point out a few other salient commonalities and differences:

%
\begin{itemize}\itemsep 1pt
\item Our coarse actions, or ``macro-actions'', are temporally extended sequences of actions at the previous (finer) scale,  and involve executing a policy within a particular \clustertext. In the language of options, the initiation set for a coarse action is any bottleneck connected to a \clustertext on which the action's policy is defined, and the policy terminates whenever a bottleneck is reached. Our coarse actions are always Markov (not semi-Markov), and the termination condition depends only on the current state. Furthermore, hierarchies of coarse actions do not lead to semi-Markov options in our framework -- they always remain Markov\footnote{This is because the homogenization we prescribe results in deterministic quantities, and Markovianity would not necessarily be preserved if coarse variables were random.}.

However, options may only direct the agent in one direction, and may terminate in the initiation set of at most one other successor option. To get around these limitations, new, separate options must be defined, increasing the problem's branching factor, and care must be taken to avoid loops (if so desired). An MMDP coarse action leaves the ``direction'' of the action undecided: the same fine policy may be executed starting in several bottleneck states, and may take the agent in one of several directions until arriving at one of multiple destinations from which different successor coarse actions may be taken. In the context of MMDPs, if one wanted to be able to transfer policies on the same cluster which guide an agent in different particular directions, separate local policies would need to be stored in the ``database'' of solved tasks. However, we would only need to transfer and plan with {\em one} of them.

\item A strength of the options framework is that multiple related queries, or tasks, may be solved essentially within the same SMDP. However, the tasks must be closely related in specific ways (e.g. tasks differing only in the goal state), and this strength comes at the expense of ignoring problem-specific information when one only wants to solve one problem. Our approach to the construction of MMDPs differs in that while we assume a particular problem when building a decomposition, we are able to consider a broader set of transfer possibilities.

\item Bottlenecks and partitioning do not explicitly enter the picture in options or SMDPs. Options may be defined on any subset of the statespace, and in applications may often take the form of a macro-action which directs the agent to an intermediate goal state starting from {\em any} state in a (possibly large) neighborhood. For example, an option may direct a robot to a hallway from any state in a room. We constrain our ``initiation'' and ``termination'' sets to be bottleneck states, however this means that learning policies at coarse scales is fast, and can be carried out completely independent of other scales. Coarse scale learning involves only the bottleneck states, giving a drastically reduced computational complexity. Provided the partitioning of a scale is well chosen, this construction allows one to capitalize on improved mixing times to accelerate convergence.

\item MMDPs are a representation for MDPs: we cannot solve problems that cannot be phrased as an MDP (i.e. problems whose solutions require non-Markov policies). A policy solving an MMDP, at any scale, is a Markov policy. SMDPs may in general have non-Markov solutions (for example, policies which depend on which option is currently being executed). 

\item Multiscaleness: The options framework is arguably, at its core, a flat method. 
In general, options may reference other options, but essentially any and all options may be made available in a given state. In order to plan with options, one needs to know which option is best to execute, and at which scale, for the given task. To choose an option at a given ``scale'', options at  other scales must be ruled out. In this sense, planning with options is ``bottom-up'', while our approach may described as ``top-down''. The bottom-up approach is potentially problematic for two reasons: (1) A domain expert has to specify coarse policies. Determining a multiscale collection of policies by hand can be difficult, if not intractable. (2) To learn a fine policy solving the problem, all options across all scales need to be considered at once potentially. This is accomplished by effectively flattening the hierarchy: a state-option $Q$-function would need to have entries for every option which might be initiated from a given state, across all scales. Scale is lost in the sense that  the value at a state may only be defined as the value at that state under a flat policy.  Without significant user guidance and tailoring of the SMDP, this cannot be avoided. Either the burden on the user is high, or the computational burden is high.

Even if one repeatedly defines SMDPs on top of each other, a number of difficulties arise: (1) The resulting transition probability kernels and reward functions are not consistent across layers. For example, the transition kernel at a coarse scale is not the transition kernel of the embedded Markov chain observed only on initiation/termination/goal states. (2) An SMDP could have a disconnected statespace at the coarse layer, and it is not clear how this problem can be resolved. (3) The lack of isolation of scales necessarily implies an increase in the number of actions, and thus an increase in the branching factor of the problem.

For these reasons, the extension of options to multiple levels may not be easily carried out, and is not often seen in the literature. By contrast, our approach is strongly multiscale. We impose a specific, stronger form of hierarchy, that abstracts each scale away from the others. At a given scale, coarse actions may refer to fine actions, {\em but only through the fixed multiscale organization}, and they cannot invoke coarse actions at or above the current scale. The multiscale structure is always enforced. This leads to significant computational savings, and a form of consistency of problems and their solutions across scales.

\item Options/SMDPs define a coarse scale transition probability law which {\em combines} transition probabilities between a macro's starting/ending states, the trajectory (sojourn) length distributions, and uniform discounting (multi-time model). This definition suffices if the options are user-specified, and one is only interested in a single layer of abstraction. The definition is problematic, however, for multi-layer hierarchies, non-uniform discounting, transfer learning, and learning of the options themselves. In our construction, coarse transition probabilities and discount factors are computed separately. One advantageous consequence of this is that path length distributions do not need to be estimated or represented explicitly. But it is also important to keep these quantities separate for the following reasons: (1) Transfer (detection, transfer of potential operators, and partial policy transfer); (2) One needs to modify the transition probability tensor in order to restrict to a local \clustertext, and solve localized sub-problems efficiently; (3) It is important to preserve multiscale consistency: a compressed MDP is an MDP that is consistent with the fine scale in the mean. This is not true of the quantities defined in the context of SMDPs. If there are non-uniform discounts, then it is the product of the discounts over trajectories that must be considered rather than a constant raised to the path length (see Section~\ref{sec:exp_path_lens} for a discussion concerning this distinction).
\item Options/SMDPs define a coarse reward function which does not depend on the termination state; aggregate rewards are pre-averaged over all possible ending states. In the context of learning and transfer, this choice can lead to serious errors. Consider the effect of averaging over paths ending at the starting state (small reward) with paths spanning a large \clustertext (large rewards). It is likely that with such a system of rewards, coarse solutions would contain little information for solving at the fine scale. In any event, this definition does not yield multiscale consistency in the sense discussed above. MMDPs keep track of the coarse reward for each possible starting and ending state, and these quantities are  approximated by analytically computed moments given a model or by Monte-carlo estimates.  Space requirements are small, however, because only bottlenecks, of which there are few, can be termination states for a macro-action. This convention also ensures multiscale consistency.
\end{itemize}

\subsubsection{Other HRL Approaches}
The MAXQ algorithm~\citep{DietterichHRL} is a method for learning a collection of policies at each layer of a programmer-specified hierarchy of subroutines, using  a form of semi-Markov Q-learning.
Two types of optimality are discussed: recursive optimality, where each sub-problem's policy is locally optimal, but the overall solution may not be optimal, and hierarchical optimality, where the global policy is optimal given the constraints the hierarchy imposes. We consider optimality with respect to the true, global optimum in the set of all stationary, Markov policies, and have discussed algorithms for solving MMDPs above which converge to optimal policies at the finest scale. More recent work~\citep{Kaelbling:CRA:11} has begun to explore hierarchy as a means for reducing the amount of search that is required to learn and/or construct a hierarchical plan. In this paper, we do not consider using a (partial) MMDP to speed up exploration and subsequent elaboration of itself, although this is an interesting avenue we leave for future work.  

Hierarchical learning in partially observable domains has also been considered~\citep{TheoKael:NIPS:04,HeBrunskill:JAIR:11,Pineau:UAI:03,Kurniawati:09}, but is a less developed topic. \cite{HeBrunskill:JAIR:11} is noteworthy in that they consider online learning with macro-actions in POMDPs, with a particular emphasis on scalability, although macro-actions must be pre-specified by the user, and are constrained to be open-loop sequences.  Our coarse actions may also be seen as open-loop controllers, but they only end when a bottleneck state is reached.

\subsubsection{Hierarchy Discovery}
Hierarchy is important in the literature referenced above,  but the meaning of the hierarchy and its geometric interpretation is often detached from the solution process. If the hierarchy must be provided by a domain expert, the solution algorithms can only make limited assumptions about what the user has provided. In this paper, {\em structure discovery} and {\em learning} are intimately connected. Hierarchies are (automatically) defined based on the geometry and goals of the problem, and this is exploited to achieve locality of the computations and scalability, and to create opportunities for transfer.  Much of the early HRL research primarily sought to define algorithms for learning given user-specified hierarchies, and only later did researchers consider automatic discovery and characterization of task hierarchies. As a result, many approaches to structure discovery appear to rest on top of generic HRL frameworks (such as options), and lack synergy with the underlying learning process. Within this collection of approaches, there are however several ideas which overlap with portions of our work. The literature concerning automatic detection of macro-actions and/or hierarchies can be roughly organized into three categories: (1) Approaches which  aggregate states based on statistics observed at individual states during simulation, (2) approaches involving graph-based clustering/analysis, and (3) approaches based on ``experimentation'' or demonstration trajectories not necessarily directly related to the task to be solved.

The work of~\cite{StollePrecup:02} recognizes a form of ``bottleneck'', there defined to be states which are visited frequently. The authors propose a heuristic algorithm which, given simulated trajectories, takes the top most frequently visited states as bottlenecks, and uses these states to define options. Others have attempted to capture similar properties. The HEXQ~\citep{Hengst02} and VISA~\citep{Jonsson06} algorithms group states based on the frequency with which their values change. \cite{Marthi07} perform a direct greedy search for hierarchical policies consistent with sample trajectories and an analysis of changes in states' values. However, a commonly occurring problem with these approaches is that they are computationally intensive and can require large trajectory samples. Exploration is global in these approaches, which could be problematic for large, complex domains. Both VISA and the HI-MAT approach of~\cite{Mehta08} (which automatically creates MAXQ hierarchies) require estimating and analyzing a dynamic Bayesian network in order to determine clusters of mutually relevant states. Approaches which assume DBN transition models allow for compact representations, but also lead to a solution cost that is exponential in the size of the representation unless specialized approximate algorithms are used. In addition, some of these approaches do not maintain a principled or consistent notion of scale.

Our intuitive notion of a bottleneck state in Algorithm~\ref{alg:spectral_clustering} is close in spirit to several other graph-theoretic definitions appearing in the literature, although we emphasize that we have not designed our approach to HRL around any single characterization of bottlenecks. Several graph clustering algorithms are proposed in~\citep{MenacheMannor:02,MannorMenache:04} for identifying subgoals (online) to accelerate Q-Learning. The latter reference employs a form of local spectral clustering (local because good bottlenecks may not always be part of a global cut)  and is related to the work of~\cite{OsentoskiMahadevan:10}, while the former proposes a clustering method that can take advantage of the current value function estimate. In these papers a weighted graph is periodically built from observed state transitions and then cut into clusters. Options are learned so that neighboring pieces of the graph can reach each other. Although policies on clusters (options) are computed separately, they are computed on the basis of an artificial reward, and can therefore be incorrect. The work of~\cite{SimsekBarto:ICML:05} is similar, and considers local spectral clustering on the basis of a limited, recent collection of trajectory samples. The statespace is successively explored and bottlenecks are identified without having to perform global computations. Options corresponding to the clusters resulting from graph cuts are again learned, but may also be incorrect, so re-learning of the options is prescribed. Another approach, distinct from spectral clustering, is the identification of bottleneck states based on ``betweenness'', proposed by~\cite{SimsekBarto:NIPS:08}. The idea follows from the observation that bottlenecks may not always be identifiable based on node connectivity.  Bottlenecks are defined to be states through which a large fraction of graph geodesics must pass. States within a small neighborhood with comparatively high betweenness are identified as bottlenecks, and options are defined on the basis of these subgoals. Betweenness is a natural alternative to diffusion based clustering techniques, but can give substantially different results depending on the geometry of the statespace and how the graph weights are chosen. Each of the clustering methods described may also be used within our framework to choose bottlenecks and  partition the statespace, and in online exploration scenarios the current statespace graph may be re-estimated (and the MDP re-compressed) as desired. 

The third category of hierarchy discovery research may be represented by the intrinsic motivation work of~\cite{Singh2004,Barto:ICDL:04}, and the skill discovery method of~\cite{Konidaris:NIPS:09}. A significant difference between these references and the work described here is that, while coarse structure is used to decompose the fine scale problem into more manageable pieces, there is no explicit independent coarse problem that is solved and pushed downwards in order to guide/accelerate the solution at the fine scale, and only one level of abstraction is considered. In~\citep{Singh2004}, an agent discovers skills by experimenting within the domain and receiving rewards for actions which lead to novel, salient, or intrinsically interesting outcomes. The learned skills then serve as options within an SMDP framework, to learn policies that can accomplish extrinsically rewarded tasks. The authors consider two layer hierarchies (one level of abstraction), and require manual specification as to which events are salient and how they are rewarded. This work differs from ours in that the sub-tasks we learn are tailored to a particular goal. In our development, nothing is required from the user and learning may be faster, but sub-task solutions may not as easily transfer to other problems. One possible way around this would be to solve multiple MMDPs corresponding to different objectives within the same domain, and store the \clustertext specific solutions in a database. 

The skill chaining method proposed by~\cite{Konidaris:NIPS:09} seeks to learn skills in continuous domains by working backwards from a goal state, and is particularly useful under a {\em query paradigm} in which multiple, closely related problems need to be solved (for instance, involving different goal states). The skill discovery process is local in the sense that only a neighborhood around the previous milestone is considered in defining a new skill. The notion of chaining together skills is similar to the constraint we have imposed stipulating that coarse actions can only be taken in bottleneck states: the initiation/terminal sets of given skill can only be initiation and/or terminal sets of other skills.  In a subsequent paper~\citep{Konidaris:IJRR:12}, the authors extend skill chains to trees of skills. However, the trees refer to the arrangement of skills within a single level of abstraction, and does not refer to a tree of scales.  Skill chains and trees  may be described as effectively imposing a localized representation for the value function, driven by the geometry of the problem. By using localized basis functions to represent the value function  in a flat model~\citep{MahadevanMaggioni:JMLR:07,OsentoskiMahadevan:10}, it is likely that similar structure could be captured. On the basis of these observations, it is possible that the recursive spectral clustering algorithm described in Section~\ref{sec:clustering} can lead to a  hierarchy of sub-tasks respecting similar geometric properties as that of the skill trees in~\citep{Konidaris:IJRR:12} but at more than one level of abstraction.

The references above propose different heuristics for defining how and when options (or other sub-goals) are learned and/or updated. However, a major challenge is to define an isolated problem whose solution can be obtained efficiently (i.e. locally), but is still consistent with other macro-actions. In several of the approaches above, the rewards used to learn the subtasks and the values fixed at bottleneck states may not be compatible, in which case policies can conflict (with respect to a designated goal) across subtasks. The multiscale framework described here provides a principled way to learn consistent local policies, given a partition of the statespace into subtasks. The reward function, boundary values, and discounts are determined automatically in our setting. Consistency is also maintained across scales so that the process may be readily repeated as necessary. If a single macro-action is itself a large problem, it is not clear how the methods above can be extended to another scale because of scale-dependent assumptions. The learning algorithms we have proposed are independent of the scale at which they are applied. Finally, it is often the case that non-standard algorithms are required to solve a given problem, once a set of subtasks is identified. Our approach provides a hierarchy of ordinary MDPs which can be solved using standard techniques.

A more integrated approach to HRL with automatic construction of the hierarchy is the recent work of~\cite{BarryKael:IJCAI:11}. The DetH* algorithm proposed in~\citep{BarryKael:IJCAI:11} shares some common themes with our work (and we like the name), although there are also pointed differences. DetH* uses a type of coarse policy (a deterministic map between coarse states) to decompose the fine scale problem into sub-problems (extending beyond two layers is not discussed), but does not determine a coarse value function that can be used to solve the independent sub-problems in a manner which maintains global consistency. A type of local optimality is, however, guaranteed. The end product at the coarse scale is not an MDP, and the complexity of the proposed coarse solution algorithm depends on a quantity similar to the worst case fine scale cluster diameter. The multiscale MDPs proposed here are hierarchies of independent, self-contained MDPs. DetH* clusters the statespace by trading off the size of the clusters against a reachability condition. Heuristics are applied to ensure that clusters do not become too large or too small. The heuristics indirectly attempt to settle on an appropriate scale, but it is not immediately clear how geometry enters the picture. ``Too-large'' appears to mainly be a computational consideration (i.e. the number of states). Our perspective is that it is perfectly fine, indeed desirable, for clusters to be large -- depending on the problem geometry or other notion of intrinsic complexity. For example, if the Markov chain associated to a policy is fast mixing within the room. More precisely, the number of states in a cluster is not a computational problem so much as {\em conditioning} is. The approach taken in our work seeks to find the right scale and partitioning based directly on local geometry, and by extension, conditioning. 

Another difference worth mentioning is that DetH* defines coarse states to be representatives of entire fine scale clusters, and goal states are also lumped into a single macro goal state. In our framework the coarse states are bottlenecks, and are elements of the fine scale statespace. Because the coarse states in DetH* are sets, the authors define a cost function on coarse states based on averages of shortest paths between clusters. The extent to which the underlying dynamics can captured is not clear, and the coarse quantities are not consistent with the fine scale in a precise, probabilistic sense invoking an underlying Markov chain. In addition, DetH* determines a hierarchy on the basis of shortest paths, and cannot consider a coarsening with respect to a particular policy. Transitions between coarse states are deterministic, whereas we construct a coarse problem which is itself an MDP, having its own transition kernel. This allows for greater generality and encompasses a richer set of problems. Finally, \cite{BarryKael:IJCAI:11} and the other references above do not consider transfer within their respective hierarchical frameworks. 


\subsection{Transfer Learning}
A good overview of the current landscape for transfer in reinforcement learning is~\citep{Taylor09}; we provide only a brief summary of some efforts related to ours. The literature discussed up to this point has been primarily concerned with discovering state abstractions for a specific problem, and transfer may be possible only to problems in the same domain, if at all. The approach taken in~\citep{Konidaris:JMLR:12,Konidaris:IJCAI:07} posits the existence of shared features across related tasks, and discusses transfer of value functions defined from those features (for example, the features may be coefficients in a basis function expansion). The shared features constitute a representation which simultaneously captures relatedness among tasks and solves any relevant correspondence problems. Transferred value functions serve as shaping functions for learning new tasks. Reward systems are the same across tasks. In some instances, options are transferred~\citep{Konidaris:IJCAI:07}, again given a suitable feature space, but the option transfer is limited to a single level of abstraction.  The authors do not discuss how to identify a suitable feature space (``agent space'') to carry out transfer, although this is a crucial element and determines what kind of transfer is possible, and the degree to which transfer helps in new tasks. The earlier work of~\cite{Guestrin:ICJAI:03} is related to that of \cite{Konidaris:JMLR:12} in that sets of similar problems are defined from the ground-up using a common, class-based formalism. In the language of~\cite{Konidaris:JMLR:12}, \cite{Guestrin:ICJAI:03} specify problems with a specific, predetermined feature space in mind, so that value functions defined on the features immediately transfer to tasks within the class, by construction. The approach is not, however, hierarchical. 

\cite{FergusonMahadevan:06} (also~\cite{Tsao:AAMAS:08}) do not consider transfer in a hierarchical setting, but take an approach that can also be incorporated into our development. In their approach, eigenfunctions of a graph Laplacian describing the domain are transferred (the graph is constructed from trajectories). These eigenfunctions (or ``proto-value functions'') serve as basis functions for defining a value function over the statespace~\citep{MahadevanMaggioni:JMLR:07}. Transfer is considered among tasks with identical domains but varied reward functions, or among tasks with fixed reward function and geometry, but scaled statespaces. Since subtasks (coarse or fine) within an MMDP are themselves MDPs, one may also select various basis functions on which to expand the local value functions specific to \clusterstext in our framework. This extension of the basic MMDP solution methodology we have described (Section~\ref{sec:mdp_solution}) can be applied at any scale, based on graph Laplacians derived from either simulations, or from a transition model $P$. The resulting proto-value functions may be stored as part of the solution to a sub-task in the library of transferrable objects, and transferred when appropriate.

Multiscale MDPs, as defined in this paper, contrast with the approaches above in that transfer may be pursued at any scale (out of many), and the procedure for carrying out transfer is the same regardless of scale. We support transfer of coarse or fine scale knowledge, or combinations thereof. In addition, we have attempted to automatically handle the problem of systematically identifying transfer opportunities and encoding the knowledge to be transferred. Although value functions (defined with basis functions or otherwise) may be transferred within our framework,  transfer can also take the form of policies or potential functions, so that transfer can occur between more dissimilar tasks. We still, however, require statespace graph matchings, which may be challenging to obtain depending on the problems and scales under consideration. Finally, refinement and improvement of a transferred quantity is straightforward when working with MMDPs. Any algorithm may be used to improve the policy, since  problem and sub-task representations are always independent MDPs, both within and across scales.


\section{Discussion}
\label{sec:discussion}
We have presented a general framework for efficiently compressing Markov decision problems, and considered multiscale knowledge transfer between related MDPs. Our treatment is multiscale, and centers on a hierarchical decomposition in which coarse scale problems are independent, deterministic MDPs, and in which local sub-problems at fine scales may be decoupled given a coarse scale solution. We then argued that such multiscale representations may be used to efficiently solve a problem, and to transfer localized and/or coarse solutions rather than global solutions.  The experiments we considered demonstrated computational speedups as well as the transfer of localized potential functions and policies at both coarse and fine scales across three example domains. As one would expect, problems receiving transferred information were shown to be solvable using less computational effort. 

In the subsections below, we address a few generalizations and suggest outstanding directions for future consideration. 

%

\subsection{Model-based vs. Model-free Learning}
The compression procedure introduced above assumes access to a pre-specified model described by $P, R$ and $\Gamma$, although it is only mildly model-dependent in the sense that compression involves averaging so that the ``model'' is not needed to high precision. The assumption that the model is known may be relaxed entirely, however. The general multiscale approach to compression we have described may be extended to include a completely model-free setting by considering a fully empirical, Monte-Carlo based compression and bottleneck detection regime. Bottlenecks may be initially detected on the basis of a local exploration (see for example Spielman's local heat flow algorithm~\citep{Spielman:LocalClustering} or Peres' evolving sets~\citep{Andersen:2009:ESP,Morris_Peres_2003}), so that the entire $P$ matrix is not needed.  The exploration may be done starting from  a goal state, for example, and would be inexpensive because only the \clustertext enclosing the goal state needs to be considered. Given the bottlenecks, Monte-Carlo based simulations can be used to compress the MDP locally in the vicinity of the starting state by directly estimating the coarse ingredients (transition probabilities, rewards, and discounts). The process may then be repeated starting from the detected bottleneck states, proceeding outwards, to build up a global picture successively adding one (or a few) \clusterstext at a time.  On-policy exploration could be accelerated in previously explored regions by using the compressed model as a fast simulator.  This approach could make  difficult problems, where long sequences of actions are necessary to reach the goal, more tractable.


\subsection{Continuous Domains, Sampling, and Dictionary Expansions}
We have assumed discrete state and action spaces, although it should be emphasized that the multiscale development here does not critically depend on the discrete assumption, and may be adapted to continuous domains. A simple approach, which was considered in Section~\ref{sec:pend-expt}, is to discretize and then build a model on a discrete set of representative states. The discretization is in general problem dependent, and need not be dense in order to apply the homogenization prescribed above. A continuous problem could be discretized with a coarse sampling determined by the problem's complexity, with the expectation of good results since complexity (and geometry) largely determines the multiscale decomposition. The translation from continuous to discrete could be model-based (e.g., eigenfunctions of $P$) or model-free (e.g., eigenvectors of the graph Laplacian built from simulated  trajectories). Moreover, in this case the coarse MDPs have discrete statespaces, so that handling continuous variables is only a concern at the finest scale.

Another approach might be to overlay a discrete coarse MDP on top of a continuous fine scale problem. In this case, the fine scale quantities (including policies) could be represented by expansions on basis function  dictionaries for the statespace, or even with factored representations and neural networks. One could then consider collapsing bottleneck regions (as sets) into single coarse states, and compressing based on fine scale trajectory statistics between these regions. Depending on the form of the model, it is possible that analytical expressions for the coarse quantities can be obtained (for example, if one  uses Gaussian basis functions to describe $P, R$, leading to Gaussian integrals). The local boundary value problems occurring at the fine scale could be amenable to solution with existing approximate DP algorithms.

We point out that adopting basis  function representations for the model could also support a broader set of transfer possibilities. Basis functions may themselves be transferred locally (representation transfer), in addition to solutions. A careful choice of basis can also impart invariance properties, and provide a means to accommodate domain changes (e.g., scaling, by way of Nystrom extensions, and goal changes) and reward function changes across problems (see~\cite{FergusonMahadevan:06} for a discussion related to representation transfer).

\subsection{Partially Observable Domains}
Partially observable MDPs (POMDPs) can be cast as fully observable MDPs on a continuous belief statespace, so that in theory POMDPs can be decomposed, solved and transferred using the framework discussed in this paper. In practice, solving belief MDPs exactly can be computationally prohibitive. Extending the multiscale MDP framework described above to POMDPs in a more fundamental way could lead to efficient approximate solution algorithms. For instance, solutions to more tractable coarse problems could be used to provide  interpolated solutions to finer problems, where accuracy vs. complexity of the interpolation can be balanced locally. 

\section*{Acknowledgments}
This work was partially supported by DARPA grants MSEE FA8650-11-1-7150; Washington State U. DOE contract SUB\#113054 G002745; NSF grants IIS-08-03293, DMS-08-47388. The authors also gratefully acknowledge helpful discussions with, and suggestions from, Sridhar Mahadevan.


\appendix

\section{Derivation of the linear system describing a value function}
Let $\pi$ be a policy on $S$.
Here we prove equation \eqref{eqn:vpi-linsys}:
\begin{equation*}
V^{\pi}(s) = \sum_{s',a}P(s,a,s')\pi(s,a)\bigl[R(s,a,s') + \Gamma(s,a,s')V^{\pi}(s')\bigr],\quad s\in S.
\end{equation*}
where we recall that $V^\pi$ is the value function defined in~\eqref{eqn:pi-discounted-rewards},
\begin{equation*}
V^{\pi}(s) = \bbE\left[R(s_0,a_1,s_1) + 
\sum_{t=1}^{\infty}\left\lbrace\prod_{\tau=0}^{t-1}\Gamma(s_{\tau},a_{\tau+1},s_{\tau+1})\right\rbrace R(s_t,a_{t+1},s_{t+1}) ~\Bigl|\Bigr.~ s_0=s\right].
\end{equation*}
Applying the Markov property to the first expectation on the right-hand side
\begin{align*}
\bbE\bigl[R(s_0,a_1,s_1) ~|~ s_0=s\bigr] &= \sum_{s',a}\bbP(s_1=s', a_1=a|s_0=s)R(s,a,s')\\
&= \sum_{s',a}P(s,a,s')\pi(s,a)R(s,a,s').
\end{align*}
For the second term, we have
\begin{multline*}
\bbE\left[\sum_{t=1}^{\infty}\prod_{\tau=0}^{t-1}\Gamma(s_{\tau},a_{\tau+1},s_{\tau+1}) R(s_t,a_{t+1},s_{t+1}) ~\Bigl|\Bigr.~ s_0=s\right] \\
\begin{aligned}
&= \bbE_{s_1,a_1}\left\{\bbE\left[\sum_{t=1}^{\infty}\prod_{\tau=0}^{t-1}\Gamma(s_{\tau},a_{\tau+1},s_{\tau+1}) R(s_t,a_{t+1},s_{t+1})
 ~\Bigl|\Bigr.~ s_0=s, s_1, a_1\right] ~\Bigl|\Bigr.~ s_0=s\right\}\\ 
& \begin{split}
 &= \bbE_{s_1,a_1}\Biggl\{\Gamma(s_0,a_1,s_1)\bbE\biggl[ R(s_1,a_2,s_2) \;+ \\
& \qquad\qquad 
\sum_{t=2}^{\infty}\prod_{\tau=1}^{t-1}\Gamma(s_{\tau},a_{\tau+1},s_{\tau+1}) R(s_t,a_{t+1},s_{t+1})
 ~\Bigl|\Bigr.~ s_0=s, s_1, a_1\biggr] ~\Bigl|\Bigr.~ s_0=s\Biggr\}
 \end{split} \\
&= \bbE_{s_1,a_1}\bigl[ \Gamma(s_0,a_1,s_1)V^{\pi}(s_1) ~|~ s_0=s\bigr]\\
&= \sum_{s'}V^{\pi}(s')\sum_a P(s,a,s')\Gamma(s,a,s')\pi(s,a).
\end{aligned}
\end{multline*}
Putting these results together, we obtain the linear system of equations
\begin{equation*}
V^{\pi}(s) = \sum_{s',a}P(s,a,s')\pi(s,a)\bigl[R(s,a,s') + \Gamma(s,a,s')V^{\pi}(s')\bigr],\quad s\in S\,,
\end{equation*}
which is \eqref{eqn:vpi-linsys}.

\section{Analytical results and computational considerations for the compression step}
\subsection{Compressed transition matrix $\widetilde P$: Proof of Proposition~\ref{prop:coarsePsas}}
Let  $a\in\widetilde{A}$ be the coarse action corresponding to executing a policy $\pi_\clusterindex\in\picluster_\clusterindex$ in \clustertext\ $\clusterindex$, so that $(X_t)_{t\geq 0}$ is the (discrete time) Markov chain on the \clustertext \clusterindex with transition matrix $\Ppirestrtocluster$. Recall that the set of bottleneck states within the \clustertext is denoted $\dcluster \subseteq \B$, and the set of non-bottleneck (interior) states in the \clustertext is denoted $\Ui:=\clusterindex\setminus\dcluster $. 

First, observe that if $s\notin\clusterindex$, then the entries $\widetilde{P}(s,a,\cdot)$ are not defined because the action is unavailable. 
Second, if $s'\notin\clusterindex$, then we know that  $\widetilde{P}(s,a,s')=0$. 
Therefore we restrict our attention to pairs $s,s'\in\clusterindex$, i.e. $s,s'$ compatible with $a$.  
The transition probabilities among pairs of states in $\dcluster $ are computed by observing the Markov chain $(X_t)_{t\ge0}$ at the hitting times of $\dcluster$: 
\[
T_m =\inf\{t>T_{m-1} ~|~ X_t\in \dcluster \},\qquad m=1,2,\ldots
\]
with $T_0 = \inf\{t\geq 0 ~|~ X_t\in \dcluster \}$.  The hitting times are a.s. finite ($\bbP(T_m<\infty)=1$, $\forall m\geq 0$) in light of the fact that,  by construction, absorbing states are bottlenecks, and the assumption that \dcluster is $\pi$-reachable from any starting point. A new chain $(Y_m)_{m\geq 0}$ taking only values in $\dcluster $ can now be defined as $Y_m=X_{T_m}$. The transition probability matrix governing $(Y_m)_m$ is computed from that of $(X_t)_t$ by solving a linear system for a few different right hand sides as follows. Let $\bbP_s(B) := \bbP(B ~|~ X_0 = s)$, for any event $B$ (measurable w.r.t. a suitable $\sigma$-algebra). Consider the hitting probabilities $\bbP_s(X_{T_0}=s')$. 
Clearly $\bbP_s(X_{T_0}=s') = \delta_{s,s'}$ for $s,s'\in\dcluster $, where $\delta$ denotes the Kronecker delta function.  
The strong Markov property (see for instance~\citet[Thm 1.4.2]{Norris97}) allows one to apply the Markov property at (finite) stopping times, so that a one-step analysis gives the hitting probabilities for $s\in\Ui,s'\in\dcluster $ as 
\begin{align*}
\bbP_s(X_{T_0}=s') &= \bbE\bigl[\bbP_s(X_{T_0}=s'~|~X_1,a_1) ~\bigl|\bigr.~ X_0=s\bigr]\\
&=\smashoperator[r]{\sum_{s''\in\clusterindex,\, a\in A}}P_{\clusterindex}(s,a,s'')\pi_{\clusterindex}(s,a)\bbP_s(X_{T_0}=s'~|~X_1=s'',a_1=a) \\
&=\sum_{a\in A}P_{\clusterindex}(s,a,s')\pi_{\clusterindex}(s,a) + 
\smashoperator[r]{\sum_{s''\in \Ui,\, a\in A}}P_{\clusterindex}(s,a,s'')\pi_{\clusterindex}(s,a)\bbP_{s''}(X_{T_0}=s').
\end{align*}
The third equality follows from the second applying the fact that  $X_{T_0}$ is independent of $a_{1}$ given $X_{1}$, and the strong Markov property.
Summarizing, these probabilities may be computed by solving the linear system
\begin{equation}\label{eqn:t0_cases}
\bbP_s(X_{T_0}=s') =
\begin{cases}
\delta_{s,s'} & s\in\dcluster \\
\Ppirestrtocluster(s,s') + \sum_{s''\in\Ui}\Ppirestrtocluster(s,s'')\bbP_{s''}(X_{T_0}=s') & s\in\Ui \,.
\end{cases}
\end{equation}
By the strong Markov property, we also have for $s,s'\in\dcluster $,
\begin{align*}
\widetilde{P}(s,a,s') & = \bbP(Y_{m+1}=s' ~|~ Y_m=s) \\
 &= \bbP(X_{T_{m+1}} = s' ~|~X_{T_m}=s)\\
&=\bbP(X_{T_1} = s' ~|~X_{T_0}=s) \\
&=\bbP(X_{T_1} = s' ~|~X_0 =s) =  \bbP_s(X_{T_1} = s').
\end{align*}
The law of total probability applied to the right-hand side of the third equality gives, for $s,s'\in\dcluster $, 
\begin{align}
\bbP(X_{T_1} = s' ~|~X_{T_0}=s) &= 
\bbE\bigl[\bbP(X_{T_1} = s' ~|~X_{T_0}=s, X_{T_0 + 1}, a_{T_0 + 1}) ~\bigl|\bigr.~ X_{T_0}=s \bigr] \nonumber\\
&= \sum_{\substack{s''\in\clusterindex\\ a\in A}}P_{\clusterindex}(s,a,s'')\pi_{\clusterindex}(s,a)\bbP(X_{T_1} = s' ~|~X_{T_0}=s, X_{T_0 + 1}=s'', a_{T_0 + 1}=a) \nonumber\\
&= \sum_{s''\in\clusterindex}\Ppirestrtocluster(s,s'')\bbP(X_{T_1} = s' ~|~X_{T_0}=s, X_{T_0 + 1}=s'') \nonumber\\
&= \Ppirestrtocluster(s,s') + \sum_{s''\in\Ui}\Ppirestrtocluster(s,s'')\bbP_{s''}(X_{T_0}=s'). \label{eqn:t1_sys}
\end{align}
The third equality follows from the second using the fact that $X_{T_1}$ is independent of $a_{T_0+1}$ given $X_{T_0+1}$.

Noticing that $(\bbP(X_{T_1} = s' ~|~X_{T_0}=s))_{s,s'\in\dcluster}$ depends on $(\bbP_s(X_{T_0}=s'))_{s,s'\in\Ui}$, but the latter do not depend on the former, we can combine Equations~\eqref{eqn:t0_cases} and~\eqref{eqn:t1_sys} into a single linear system for each $s'\in\dcluster $:
\begin{equation}\label{eqn:trans_probs_app}
H_{s,s'} =  \Ppirestrtocluster(s,s') + \sum_{s''\in\Ui}\Ppirestrtocluster(s,s'')H_{s'',s'},\quad s\in \clusterindex, s'\in\dcluster .
\end{equation}
We then have
\[
\widetilde{P}(s,a,s') = H_{s,s'},\quad  \text{ for all } s,s'\in \dcluster ,
\]
assuming $H$ is the {\em minimal non-negative solution} to~\eqref{eqn:trans_probs_app}, and $a$ is the action corresponding to executing the policy $\pi_\clusterindex$ in \clustertext\ $\clusterindex$. Consider the partitioning 
$$
\Ppirestrtocluster =
\begin{pmatrix}
Q & B\\
C & D
\end{pmatrix},
\qquad
H =
\begin{pmatrix}
h_q \\
h_b
\end{pmatrix}
$$
where the blocks $Q, D$ describe the interaction among non-bottleneck and bottleneck states within \clustertext\ $\clusterindex$ respectively. In matrix-vector form, we can solve for the compressed probabilities by computing the minimal non-negative solution to 
\begin{equation}\label{eqn:trprob_submtx}
(I-Q)h_q = B
\end{equation}
followed by $$h_b = D + Ch_q,$$ where $h_b$ is the desired transition probability matrix of the compressed MDP given the action $a$. If Equation~\eqref{eqn:trprob_submtx} has a unique solution, then the cost of these computations\footnote{Using for instance, an LU factorization ($\cO(|\Ui|^3)$) to efficiently solve for $|\dcluster |\ll |\Ui|$ right-hand sides at a cost of $\cO(|\Ui|^2)$ each.} is at most $\cO(|\Ui|^3 + |\dcluster ||\Ui|^2)$. If solving the linear system~\eqref{eqn:trprob_submtx} does not produce a non-negative solution, then algorithms for non-negative least-squares must be used.
 
 From these expressions, it is clear that the transition probabilities starting from non-bottleneck states $h_q$ do not depend on those starting from the bottleneck states or on entries of  $\Ppirestrtocluster$ outside of $Q$. In addition, by definition of the stopping times above, the transition probabilities enforce $\widetilde{P}(s,a,s')=\delta_{s,s'}$ whenever $s$ is absorbing.

\subsection{Compressed rewards $\widetilde R$: Proof of Proposition \ref{prop:coarse_rewards}}

We will first need to define a controlled Markov process conditioned on future events.  The approach taken here is similar to that of the Doob $h$-transform (see~\cite{LPWBook}) for Markov chains, but differs in that we keep track of the actions. We fix $s'\in\dcluster$. Consider the event 
$\{X_{T_0}=s'\}$ and define
\begin{equation}\label{eqn:h_function}
h_{s'}(s):=\bbP_s(X_{T_0}=s'),\qquad s\in \clusterindex,
\end{equation}
with the probabilities $\bbP_s(X_{T_0}=s')$ given by Equation~\eqref{eqn:t0_cases}. 
It can be shown that the function $h_{s'}$ is $\Ppirestrtocluster$-harmonic. Using Bayes rule and the strong Markov property,
\begin{align}
\bbP_s(X_1=s'', a_1=a~|~X_{T_0}=s') &= \frac{\bbP_s(X_{T_0}=s' ~|~ X_1=s'',a_1=a)\bbP_s(X_1=s'',a_1=a)}{\bbP_s(X_{T_0}=s')} \notag \\
 &= \frac{\bbP_{s''}(X_{T_0}=s')P_\clusterindex(s,a,s'')\pi_\clusterindex(s,a)}{\bbP_s(X_{T_0}=s')} \notag \\
 &= \frac{P_{\clusterindex}(s,a,s'')\pi_\clusterindex(s,a)h_{s'}(s'')}{h_{s'}(s)} \notag \\
 & =: P_{h_{s'}}(s,a,s'')\label{eqn:Ph_defn}
\end{align}
for $s''\in \clusterindex_{s'}:=\{s\in \clusterindex~|~h_{s'}(s)>0\}$, $s\in \Uprime\setminus\{s'\}$, and $a\in A$.

Similarly, for $(s,s')\in\supp_a(\widetilde{P})\subseteq\dcluster $, $s''\in \clusterindex_{s'}$, $a\in A$,
\begin{equation}\label{eqn:Phtilde_defn}
\begin{aligned}
\bbP_s & (X_{T_0+1}=s'',a_{T_0+1}=a~|~X_{T_1}=s')\\
&= \frac{\bbP(X_{T_1}=s' ~|~ X_{T_0}=s,X_{T_0+1}=s'',a_{T_0+1}=a)\bbP_s(X_1=s'',a_1=a)}{\bbP_s(X_{T_1}=s')}\\
 &= \frac{P_{\clusterindex}(s,a,s'')\pi_\clusterindex(s,a)h_{s'}(s'')}{\widetilde{P}(s,a,s')}\\
 & =: P_{\tilde{h}_{s'}}(s,a,s'')
\end{aligned}
\end{equation}
since 
\begin{equation*}
\bbP(X_{T_1}=s' ~|~ X_{T_0}=s,X_{T_0+1}=s'') = 
\begin{cases}
\delta_{s'',s'} & \text{if } s''\in\dcluster \\
\bbP_{s''}(X_{T_0}=s') & \text{if } s''\in\Ui
\end{cases}
\end{equation*}
is equal to $h_{s'}(s)$ defined by Equation~\eqref{eqn:h_function}, and for $s\in\dcluster $ we have $\bbP_s(X_{T_1}=s')=\widetilde{P}(s,a,s')$ as given by Equation~\eqref{eqn:t1_sys}.

We now consider the expected rewards collected along paths between bottlenecks connected to a  \clustertext. The process is similar to that of the transition probabilities, where we first defined hitting probabilities at time $T_0$, and from those quantities defined conditional hitting probabilities at time $T_1$. Here we use discounted expected rewards collected up to time $T_0$ to ultimately compute rewards collected between $T_0$ and $T_1$. Recall that we assume a reward is collected only after transitioning. 
Let $T$ and $T'$ be two arbitrary stopping times satisfying $0\leq T<T'<\infty$ (a.s.). The discounted reward accumulated over the interval $T\leq t\leq T'$ is given by the random variable
\[
R_{T}^{T'}:= R(X_T,a_{T+1}, X_{T+1}) + \sum_{t=T+1}^{T'-1}\left[
\prod_{\tau=T}^{t-1}\Gamma\bigl(X_{\tau},a_{\tau+1}, X_{\tau+1}\bigr)\right]
R\bigl(X_t,a_{t+1}, X_{t+1}\bigr) 
\]
where $a_{t+1}\sim\pi_\clusterindex(X_t)$ for $t=T,\ldots,T'-1$, and we set $R_T^T \equiv 0$ for any $T$.

Consider $\bbE_s[R_0^{T_0}~|~X_{T_0}=s']$ for some fixed $s'\in\dcluster $. We immediately have that $\bbE_s[R_0^{T_0}~|~X_{T_0}=s']=0$ if $s=s'$, and is undefined if $s\notin \Uprime:=\{ s~|~h_{s'}(s)>0\}$ (note that $(\dcluster \setminus\{s'\})\subseteq (\clusterindex\setminus \Uprime)$ from~\eqref{eqn:t0_cases}). We will need the following Lemma.
\begin{lemma}\label{lem:one-step-rews}
For $s\in\Ui\cap \Uprime, s'\in\dcluster , {s''}\in \Uprime, a\in A$,
\[
\bbE[R_1^{T_0} ~|~ X_{T_0}=s', X_1={s''}, T_0\geq 1]  =  \bbE_{s''}[R_0^{T_0}~|~ X_{T_0}=s']
\]
and therefore,
\[
\bbE_s[R_0^{T_0}~|~X_{T_0}=s', X_1={s''}, a_1=a] = R(s,a,{s''}) + \Gamma(s,a,{s''})\bbE_{s''}[R_0^{T_0}~|~X_{T_0}=s'] .
\]
\end{lemma}

\begin{proof}[Proof of Lemma \ref{lem:one-step-rews}]
We have
\begin{align*}
\bbE_s[R_0^{T_0} & ~|~X_{T_0}=s', X_1={s''}, a_1=a]\\
&= R(s,a,{s''}) + \Gamma(s,a,{s''})\bbE_s\biggl[R(X_1,a_2,X_2) \;+ \biggr. \\
&\qquad \biggl. \sum_{t=2}^{T_0-1}\left\lbrace\prod_{\tau=1}^{t-1}\Gamma\bigl(X_{\tau},a_{\tau+1}, X_{\tau+1}\bigr)\right\rbrace
R\bigl(X_t,a_{t+1}, X_{t+1}\bigr) ~\bigl|\bigr.~X_{T_0}=s', X_1={s''}, a_1=a
\biggr]\\
&= R(s,a,{s''}) + \Gamma(s,a,{s''})\bbE[R_1^{T_0}~|~X_{T_0}=s', X_1={s''}, T_0\geq 1].
\end{align*}
If ${s''}=s'$ then there is nothing more to show, since $R_1^1=0$. For ${s''}\in\Ui\cap \Uprime$, it will suffice to show that $\bbE[R_1^{T_0}~|~ X_1={s''}, T_0\geq 1] = \bbE_{s''}[R_0^{T_0}]$. Given a sequence of states $(i_n,\ldots,i_{n+p-m})$ and actions $(a_{n+1},\ldots,a_{n+p-m})$ with $0\leq n,m< p$, define the event
\[
B_{m,n}^p := \Biggl(\bigcap_{j=0}^{p-m} \{X_{m+j}=i_{n+j}\}\Biggr)\cap
\Biggl(\bigcap_{j=1}^{p-m} \{a_{m+j}=a_{n+j}\}\Biggr)
\]
and consider the conditional probability, for $n\geq 1$,
\begin{align*}
\bbP(\{T_0=n\}\cap B_{1,1}^{T_0} ~|~ X_1={s''}, T_0\geq 1) &= 
\bbP(B_{1,1}^{T_0} ~|~ T_0=n, X_1={s''})\bbP(T=n~|~ X_1={s''}, T_0\geq 1) \\
&= \bbP(B_{1,1}^{n} ~|~ X_1={s''},X_2\notin\dcluster ,\ldots,X_{n-1}\notin\dcluster ,X_n\in\dcluster )\\
 &\qquad \times \bbP(T_0=n~|~ X_1={s''}, T_0\geq 1) \\
&= \bbP(B_{0,1}^{T_0} ~|~ T_0=n-1,X_0={s''})\bbP(T_0=n-1~|~ X_0={s''}) \\
&= \bbP(\{T_0=n-1\}\cap B_{0,1}^{T_0} ~|~ X_0={s''}),
\end{align*}
where we have used the fact that $\bbP(T_0=n~|~ X_1={s''}, T_0\geq 1) =\bbP(T_0=n-1~|~ X_0={s''})$. This latter equality is true, since, by homogeneity, 
\begin{align*}
\bbP(T_0=n~|~ X_m={s''}, T_0\geq m) &= \bbP\Biggl(\bigcap_{j=m}^{n-1}\{X_j\notin\dcluster \} \cap \{X_n\in\dcluster \} ~\Bigl|\Bigr.~ X_m={s''}\Biggr) \\
&=  \bbP\Biggl(\bigcap_{j=0}^{n-m-1}\{X_j\notin\dcluster \} \cap \{X_{n-m}\in\dcluster \} ~\Bigl|\Bigr.~ X_0={s''}\Biggr)\\
&= \bbP(T_0=n-m~|~ X_0={s''})
\end{align*}
for $n\geq m$. Next, let
\[
f(i_0,\ldots,i_{n},a_1,\ldots,a_{n}):= R(i_0,a_1, i_1) + \sum_{t=1}^{n-1}
\left[\prod_{\tau=0}^{t-1}\Gamma\bigl(i_{\tau},a_{\tau+1}, i_{\tau+1}\bigr)\right]
R\bigl(i_t,a_{t+1}, i_{t+1}\bigr).
\]
Then, assuming $T_0 <\infty$ a.s., we have
\begin{align*}
\bbE[R_1^{T_0} & ~|~ X_1={s''}, T_0\geq 1] \\ 
 &= 
  \sum_{1\leq n <\infty}\sum_{\substack{i_1,\ldots,i_n\in \clusterindex\\ a_2,\ldots,a_n\in A}} \bbP(\{T_0=n\}\cap B_{1,1}^{T_0} ~|~ X_1={s''}, T_0\geq 1) f(i_1,\ldots,i_{n},a_2,\ldots,a_{n})\\
  &= \sum_{1\leq n <\infty}\sum_{\substack{i_1,\ldots,i_n\in \clusterindex\\ a_2,\ldots,a_n\in A}}
  \bbP(\{T_0=n-1\}\cap B_{0,1}^{T_0} ~|~ X_0={s''}) f(i_1,\ldots,i_{n},a_2,\ldots,a_{n}) \\
&= \sum_{0\leq n <\infty}\sum_{\substack{i_0,\ldots,i_n\in \clusterindex\\ a_1,\ldots,a_n\in A}}
  \bbP(\{T_0=n\}\cap B_{0,0}^{T_0} ~|~ X_0={s''}) f(i_0,\ldots,i_{n},a_1,\ldots,a_{n})\\
&= \bbE[R_0^{T_0}~|~ X_0=s''] \,.
\end{align*}
\end{proof}

With the above definitions, we turn to proving the proposition.
\begin{proof}[Proof of Proposition \ref{prop:coarse_rewards}]
With the same choice of $s'$ used in $\bbE_s[R_0^{T_0}~|~X_{T_0}=s']$, define $h_{s'}$ as in Equation~\eqref{eqn:h_function} and let as above $\Uprime:=\{s\in \clusterindex ~|~h_{s'}(s)>0\}$.
For $s\in \Ui\cap \Uprime = \Uprime\setminus\{s'\}$, a one-step analysis gives
\begin{align*}
\bbE_s[R_0^{T_0} & ~|~X_{T_0}=s'] \\   
 &= \bbE_s\bigl\{\bbE_s[R_0^{T_0}~|~X_{T_0}=s', X_1, a_1] ~\bigl|\bigr.~ X_{T_0}=s'\bigr\} \\
&= \sum_{s''\in \clusterindex, a\in A}\bbP_s(X_1=s'',a_1=a ~|~X_{T_0}=s')\bbE_s[R_0^{T_0}~|~X_{T_0}=s', X_1=s'', a_1=a] \\
&= \sum_{s''\in \Uprime, a\in A}P_{h_{s'}}(s,a,s'')\bbE_s[R_0^{T_0}~|~X_{T_0}=s', X_1=s'', a_1=a] \\
&= \sum_{s''\in \Uprime, a\in A}P_{h_{s'}}(s,a,s'')\bigl(R(s,a,s'') + \Gamma(s,a,s'')\bbE_{s''}[R_0^{T_0}~|~X_{T_0}=s'] \bigr) \\ 
&= \smashoperator[r]{\sum_{s''\in \Ui\cap \Uprime, a\in A}} P_{h_{s'}}(s,a,s'')\Gamma(s,a,s'')\bbE_{s''}[R_0^{T_0}~|~X_{T_0}=s']
 + \sum_{s''\in \Uprime, a\in A}P_{h_{s'}}(s,a,s'')R(s,a,s'')
\end{align*}
where the third equality follows from the second using Equation~\eqref{eqn:Ph_defn} and the fact that $P_{h_{s'}}(s,a,s'')>0$ only for $s''\in \Uprime$, and the fourth follows from the third applying Lemma~\ref{lem:one-step-rews}. The last equality follows after rearranging terms.

With these expectations in hand, we can compute the discounted rewards between bottlenecks. Note that $\bbE[R_{T_0}^{T_1} ~|~ X_{T_0}=s, X_{T_1}=s'] = \bbE_s[R_{T_0}^{T_1} ~|~ X_{T_1}=s']$, for  $s,s'\in\dcluster $. By convention, we set $\bbE_s[R_{T_0}^{T_1} ~|~ X_{T_1}=s'] = 0$ if $\widetilde{P}(s,a,s')=0$. For $s\in\dcluster $ such that $(s,s')\in\supp_a(\widetilde{P})$, we have $T_0=0$, and a one-step analysis similar to the above gives
\begin{align*}
\bbE_s & [R_{T_0}^{T_1} ~|~ X_{T_1}=s']\\  
 &= \bbE_s\bigl\{\bbE_s[R_{T_0}^{T_1} ~|~ X_{T_1}=s', X_{T_0+1}, a_{T_0+1}] ~\bigl|\bigr.~ X_{T_1}=s'\bigr\} \\
&=\smashoperator[r]{\sum_{s''\in \clusterindex, a\in A}}\bbP_s(X_{T_0+1}=s'', a_{T_0+1}=a ~|~ X_{T_1}=s')
\bbE_s[R_{T_0}^{T_1} ~|~ X_{T_1}=s', X_{T_0+1}=s'', a_{T_0+1}=a] \\
&= \smashoperator[r]{\sum_{s''\in \Ui\cap \Uprime, a\in A}} P_{\tilde{h}_{s'}}(s,a,s'')\Gamma(s,a,s'')\bbE_{s''}[R_0^{T_0}~|~X_{T_0}=s'] + \sum_{s''\in \Uprime, a\in A}P_{\tilde{h}_{s'}}(s,a,s'')R(s,a,s'')
\end{align*}
where we have used the fact that
\begin{multline*}
\bbE_s[R_{T_0}^{T_1} ~|~ X_{T_1}=s', X_{T_0+1}=s'', a_{T_0+1}=a]\\ 
= 
\begin{cases}
R(s,a,s'') + \Gamma(s,a,s'')\bbE_{s''}[R_0^{T_0}~|~X_{T_0}=s'] & \text{if } s''\in\Ui \\
R(s,a,s') & \text{if } s''=s'
\end{cases}.
\end{multline*}
As mentioned in Section~\ref{sec:trans_probs}, the boundary is reachable from any state by assumption, so $\bbP_s(T_m < \infty)=1$ for all $s\in \clusterindex, m <\infty$. Hence, the solution to the linear system above is unique and bounded~\citep[Thm. 4.2.3]{Norris97}.
\end{proof}

We express the linear systems describing the coarse rewards above in matrix-vector form for convenience.
Fix a destination bottleneck $s'\in\dcluster $. Consider $P_{h_{s'}},P_{\tilde{h}_{s'}}$ as tensors, and partition $P_{h_{s'}}$ into the pieces $(P_{h_{s'}}^{\circ})_{s,a,k} = P_{h_{s'}}(s,a,k)$ for $s,k\in\Ui\cap \Uprime, a\in A$ and $(P_{h_{s'}}^{\partial})_{s,a,k} =  P_{h_{s'}}(s,a,k)$ for $s\in\Ui\cap \Uprime, k=s',a\in A$. Next, partition $P_{\tilde{h}_{s'}}$ into the pieces $(P_{\tilde{h}_{s'}}^{\circ})_{s,a,k} = P_{\tilde{h}_{s'}}(s,a,k)$ for $(s,s')\in\supp_a(\widetilde{P}), k\in\Ui\cap \Uprime, a\in A$ and $(P_{\tilde{h}_{s'}}^{\partial})_{s,a,k} =  P_{\tilde{h}_{s'}}(s,a,k)$ for $(s,s')\in\supp_a(\widetilde{P}), k=s',a\in A$.
Similarly, partition $\Gamma_i,R_i$ into pieces $\Gamma_{h_{s'}}^{\circ},\Gamma_{h_{s'}}^{\partial},\Gamma_{\tilde{h}_{s'}}^{\circ},\Gamma_{\tilde{h}_{s'}}^{\partial}$ and 
$R_{h_{s'}}^{\circ},R_{h_{s'}}^{\partial},R_{\tilde{h}_{s'}}^{\circ},R_{\tilde{h}_{s'}}^{\partial}$
corresponding to the respective pieces of $P_{h_{s'}},P_{\tilde{h}_{s'}}$ mentioning the same state/action triples $(s,a,s')$. Equations~\eqref{eqn:coarse_rews_linsys} and~\eqref{eqn:coarse_rews_sums} may be written, respectively, as
\begin{align*}
\bigl(I - (P_{h_{s'}}^{\circ}\circ\Gamma_{h_{s'}}^{\circ})^{\pi_\clusterindex}\bigr)h_q &=
\bigl[(P_{h_{s'}}^{\circ}\circ R_{h_{s'}}^{\circ})^{\pi_\clusterindex} ~~ (P_{h_{s'}}^{\partial}\circ R_{h_{s'}}^{\partial})^{\pi_\clusterindex}\bigr]\bbone \\
h_b &= (P_{\tilde{h}_{s'}}^{\circ}\circ\Gamma_{\tilde{h}}^{\circ})^{\pi_\clusterindex}h_q +
\bigl[(P_{\tilde{h}_{s'}}^{\circ}\circ R_{\tilde{h}_{s'}}^{\circ})^{\pi_\clusterindex} ~~ (P_{\tilde{h}_{s'}}^{\partial}\circ R_{\tilde{h}_{s'}}^{\partial})^{\pi_\clusterindex}\bigr]\bbone.
\end{align*}

Two final observations of practical interest are in order. 
The solution of the linear systems defined above (one for each destination bottleneck $s'\in\dcluster $) can be potentially carried out efficiently by preconditioning some systems on the basis of solutions to the others. In particular, if a particular bottleneck $s'$ defining a system is close to another in the statespace graph in terms of the diffusion distance induced by $\Ppirestrtocluster$, then there is good reason to believe that the solutions will be close. Second, calculation of the rewards above is  closely related computationally to calculation of the coarse discount factors. The next section derives the discount factors, and discusses when one set of quantities can be obtained from the other at essentially no cost.

\subsection{Compressed discount factors $\widetilde{\Gamma}$: Proof of Proposition \ref{prop:coarse_discount_factors}}

\begin{proof}[Proof of Proposition \ref{prop:coarse_discount_factors}]
The approach is similar to that of the rewards. First note that if we set $R(s,a,s')\equiv 1$ uniformly for all $s,s'\in \clusterindex, a\in A$, then $\Delta_T^{T'} = R_T^{T'+1} - R_T^{T'}$ (and $T'+1$ is clearly a stopping time, so this quantity is well-defined). If $s\in\Ui\cap \Uprime, s'\in\dcluster , s''\in \Uprime, a\in A$, then invoking Lemma~\ref{lem:one-step-rews}, 
\begin{equation*}\label{eqn:delta_t0_onestep}
\begin{aligned}
\bbE_s[\Delta_{0}^{T_0} ~|~ X_{T_0}=s', X_1=s'', a_1=a] &=
\Gamma(s,a,s'')\bbE_s[\Delta_{1}^{T_0} ~|~ X_{T_0}=s', X_1=s'']\\
&= \Gamma(s,a,s'')\bbE_s[R_{1}^{T_0+1} - R_{1}^{T_0} ~|~ X_{T_0}=s', X_1=s'']\\
&= \Gamma(s,a,s'')\bbE_{s''}[R_{0}^{T_0+1} - R_{0}^{T_0} ~|~ X_{T_0}=s']\\
&= \Gamma(s,a,{s''})\bbE_{s''}[\Delta_{0}^{T_0} ~|~ X_{T_0}=s'], 
\end{aligned}
\end{equation*}
and $\bbE_s[\Delta_{0}^{T_0} ~|~ X_{T_0}=s', X_1={s''}, a_1=a] = \Gamma(s,a,s')$ if ${s''}=s'$.
Thus to obtain the expectations $\bbE_s[\Delta_{0}^{T_0} ~|~ X_{T_0}=s'], s\in\Ui\cap \Uprime$, we may solve the linear system defined by
\begin{align*}
\bbE_s[\Delta_{0}^{T_0} & ~|~ X_{T_0}=s'] \\
& = \sum_{{s''},a}P_{h_{s'}}(s,a,{s''})\bbE_s[\Delta_{0}^{T_0} ~|~ X_{T_0}=s', X_1={s''}, a_1=a] \\
&= \sum_{{s''}\in\Ui\cap \Uprime, a\in A}P_{h_{s'}}(s,a,{s''})\Gamma(s,a,{s''})\bbE_{s''}[\Delta_{0}^{T_0} ~|~ X_{T_0}=s'] + \sum_{a\in A}P_{h_{s'}}(s,a,s')\Gamma(s,a,s') 
\end{align*}
where $P_{h_{s'}}$ is given by Equation~\eqref{eqn:Ph_defn} and depends on the choice of $s'$.
Next, if $s,s'\in\dcluster , {s''}\in\Ui\cap \Uprime$, by the strong Markov property,
\begin{align*}
\bbE_s[\Delta_{T_0}^{T_1} ~|~ X_{T_1}=s', X_{T_0+1}={s''}, a_{T_0+1}=a]
 &= \Gamma(s,a,{s''})\bbE[\Delta_{1}^{T_1} ~|~ X_{T_1}=s', X_{1}={s''}, T_0=0]\\
&=\Gamma(s,a,{s''})\bbE[\Delta_{1}^{T_0} ~|~ X_{T_0}=s', X_{1}={s''}, T_0\geq 1 ] \\
&=\Gamma(s,a,{s''})\bbE_{s''}[\Delta_{0}^{T_0} ~|~ X_{T_0}=s'] 
\end{align*} 
where the third equality follows from the second applying Lemma~\ref{lem:one-step-rews} with 
$\Delta_1^{T_0} = R_1^{T_0+1} - R_1^{T_0}$. If ${s''}=s'$, then we simply have $\bbE_s[\Delta_{T_0}^{T_1} ~|~ X_{T_0}=s', X_1={s''}, a_1=a]=\Gamma(s,a,s')$.

With these facts in hand, the compressed discount factors may be found from the expectations  $\{\bbE_s[\Delta_{0}^{T_0} ~|~ X_{T_0}=s'], s\in\Ui\cap \Uprime\}$ computed above, as
\begin{align*}
\bbE_s[\Delta_{T_0}^{T_1} & ~|~ X_{T_1}=s'] \\
&= \sum_{{s''}\in \clusterindex,a\in A}P_{\tilde{h}_{s'}}(s,a,{s''})\bbE_s[\Delta_{T_0}^{T_1} ~|~ X_{T_1}=s', X_{T_0+1}={s''}, a_{T_0+1}=a] \\
&= \sum_{{s''}\in\Ui\cap \Uprime, a\in A}P_{\tilde{h}_{s'}}(s,a,{s''})\Gamma(s,a,{s''})\bbE_{s''}[\Delta_{0}^{T_0} ~|~ X_{T_0}=s'] + \sum_{a\in A}P_{\tilde{h}_{s'}}(s,a,{s'})\Gamma(s,a,s')
\end{align*}
for $(s,s')\in\supp_a(\widetilde{P})$, where $P_{\tilde{h}_{s'}}$ is defined by Equation~\eqref{eqn:Phtilde_defn} and depends on the choice of $s'$.
\end{proof}
We note that for each destination bottleneck $s'\in\dcluster $, the linear system appearing in Proposition~\ref{prop:coarse_discount_factors} has the same left-hand side as the corresponding linear system for $s'$ in Section~\ref{sec:coarse_rewards}: we can compute the compressed discounts essentially for free if we have previously computed the compressed rewards (or vice versa), provided the resulting  solutions to the discount factor equations are non-negative. If they are not non-negative, separate non-negative least squares solutions must be computed, although there may still be helpful preconditioning possibilities.

\bibliography{mdp_journal}

\end{document}